\documentclass[english,10pt,journal]{IEEEtran}
\usepackage{color,graphicx,amsmath,amssymb,amsthm,epsfig,mathrsfs,cite,bm,enumerate}
\usepackage{array}
\usepackage{verbatim}
\usepackage{mathtools}

\usepackage{color,soul}
\usepackage{multirow}
\usepackage{multicol}

\makeatletter
\newcommand*{\rom}[1]{\expandafter\@slowromancap\romannumeral #1@}
\makeatother
\newcommand{\squeezeup}{\vspace{-2.5mm}}
\def\P{\mathbb{P}}

\newcommand{\upperRomannumeral}[1]{\uppercase\expandafter{\romannumeral#1}}

\newtheorem{lemma}{Lemma}{}
  \newtheorem{thm}{Theorem}

  \newtheorem{cor}[thm]{Corollary}

\newtheorem{remark}{Remark}
\usepackage{amssymb}
\usepackage{float}
\usepackage{tikz}
\usetikzlibrary{trees}
\usepackage[lofdepth,lotdepth]{subfig}
  \newtheorem{theorem}{Theorem}
\usepackage{stfloats}
\title{Noise  Statistics Oblivious GARD For Robust Regression With Sparse Outliers} 

\author{Sreejith Kallummil,  \hspace{0cm} Sheetal Kalyani  \\
 Department of Electrical Engineering \\  Indian Institute of Technology Madras\\
  Chennai, India 600036 \\
  \{ee12d032,skalyani\}@ee.iitm.ac.in
  }

\begin{document}
 \maketitle
\begin{abstract}
Linear regression models contaminated by Gaussian noise (inlier) and possibly unbounded  sparse outliers are common in many signal processing applications. Sparse recovery inspired  robust regression (SRIRR) techniques   are shown to deliver high quality estimation performance in such regression models. Unfortunately, most SRIRR  techniques  assume \textit{a priori} knowledge of  noise statistics like inlier noise variance or  outlier statistics like number of outliers. Both inlier and outlier noise statistics are rarely known \textit{a priori} and this limits the efficient operation of many SRIRR algorithms. This article proposes a novel noise statistics oblivious   algorithm called residual ratio thresholding GARD (RRT-GARD) for robust regression in the presence of sparse outliers. RRT-GARD is developed  by modifying  the recently proposed  noise statistics dependent greedy algorithm for robust de-noising (GARD).   Both finite sample and asymptotic analytical results indicate that RRT-GARD performs nearly similar to GARD with \textit{a priori} knowledge of noise statistics.  Numerical simulations in real and synthetic data sets also  point to the  highly competitive performance of RRT-GARD.
\end{abstract}
\section{Introduction}
Linear regression models with additive Gaussian noise  is one of the most widely used  statistical model in  signal processing and machine learning. However, it is widely known that  this model is extremely sensitive to the presence of gross errors or outliers in the data set. Hence, identifying outliers in linear regression models and making regression estimates robust to the presence of outliers are of fundamental interest in all the aforementioned areas of study.    Among the various outlier infested regression models considered in literature,  linear regression models contaminated by  sparse and arbitrarily  large outliers is particularly important in  signal processing. For example, sparse outlier models are used to model occlusions in image processing/computer vision tasks like face recognition\cite{self} and fundamental matrix estimation in computer vision applications\cite{armangue2003overall}. Similarly, interferences are modelled using sparse outliers\cite{nbi} in many wireless applications. This article discusses this  practically and theoretically important problem of robust regression in the presence of sparse outliers.  After presenting the necessary notations, we mathematically explain the robust regression problem considered in this article. 
\subsection{Notations used in this article}
  $\mathbb{P}(\mathcal{A})$ represents the probability of event $\mathcal{A}$ and  $\mathbb{P}(\mathcal{A}|\mathcal{B})$ represents the conditional probability of event $\mathcal{A}$ given event $\mathcal{B}$. Bold upper case letters represent matrices and bold lower case letters represent vectors.  $span({\bf X})$ is the column space of ${\bf X}$. ${\bf X}^T$ is the transpose and ${\bf X}^{\dagger}=({\bf X}^T{\bf X})^{-1}{\bf X}^T$ is the   pseudo inverse of ${\bf X}$. ${\bf P}_{\bf X}={\bf X}{\bf X}^{\dagger}$ is the projection matrix onto  $span({\bf X})$.  ${\bf X}_{\mathcal{J}}$ denotes the sub-matrix of ${\bf X}$ formed using  the columns indexed by $\mathcal{J}$. ${\bf X}_{\mathcal{J},:}$ represents the rows of ${\bf X}$ indexed by $\mathcal{J}$.  Both ${\bf a}_{\mathcal{J}}$ and ${\bf a}({\mathcal{J}})$  denote the  entries of vector ${\bf a}$ indexed by $\mathcal{J}$. $\sigma_{min}({\bf X})$ represents the minimum  singular value of ${\bf X}$.  ${\bf 0}_m$ is the $m\times 1$ zero vector and ${\bf I}^m$ is the $m \times m$ identity matrix.  $\|{\bf a}\|_q=(\sum\limits_{j=1}^m|{\bf a}_j|^q)^{1/q} $ is the $l_q$ norm of ${\bf a}\in \mathbb{R}^m$.  $supp({\bf a})=\{k:{\bf a}_k\neq 0\}$ is the support of ${\bf a}$. $l_0$-norm  of ${\bf a}$ denoted by $\|{\bf a}\|_0=card(supp({\bf a}))$ is the cardinality of the set $supp({\bf a})$.  $\phi$ represents the null set. For any two index sets $\mathcal{J}_1$ and $\mathcal{J}_2$, the set difference  $\mathcal{J}_1/\mathcal{J}_2=\{j:j \in \mathcal{J}_1\& j\notin  \mathcal{J}_2\}$.  $f(m)=O(g(m))$ iff $\underset{m \rightarrow \infty}{\lim}\frac{f(m)}{g(m)}<\infty$. ${\bf a}\sim \mathcal{N}({\bf u},{\bf C})$ implies that ${\bf a}$ is a Gaussian  random vector/variable (R.V) with mean ${\bf u}$ and covariance ${\bf C}$.  $\mathbb{B}(a,b)$ is a beta R.V with parameters $a$ and $b$. $B(a,b)=\int_{t=0}^1t^{a-1}(1-t)^{b-1}dt$ is the beta function with parameters $a$ and $b$. $[m]$ represents the set $\{1,\dotsc,m\}$. ${\bf a}\sim {\bf b}$ implies that ${\bf a}$ and ${\bf b}$ are identically distributed. ${\bf a} \overset{P}{\rightarrow }{\bf b}$ denotes the convergence  of R.V ${\bf a}$ to ${\bf b}$ in probability.
\subsection{Linear regression models with sparse outliers}
We consider an outlier contaminated linear regression model 
\begin{equation}\label{model}
{\bf y}={\bf X}\boldsymbol{\beta}+{\bf w}+{\bf g}_{out},
\end{equation} 
where ${\bf X} \in \mathbb{R}^{n\times p}$ is a full rank design matrix with $n>p$ or $n\gg p$. $\boldsymbol{\beta}$ is the unknown regression vector to be estimated. Inlier  noise  ${\bf w}$ is assumed to be Gaussian distributed with mean zero and variance $\sigma^2$, i.e., ${\bf w}\sim \mathcal{N}({\bf 0}_n,\sigma^2{\bf I}^n)$. Outlier  ${\bf g}_{out}$ represents the large   errors in the regression equation that are not modelled by the inlier noise distribution.  As aforementioned,  ${\bf g}_{out}$ is modelled as sparse in  practical applications, i.e., the support of ${\bf g}_{out}$ given by $\mathcal{S}_{g}=supp({\bf g}_{out})=\{k:{\bf g}_{out}(k)\neq 0\}$ has cardinality  $k_g=\|{\bf g}_{out}\|_0=card(\mathcal{S}_g)\ll n$. However, $\|{\bf g}_{out}\|_2$ can take arbitrarily large  values. Please note that no sparsity assumption is made on the regression vector $\boldsymbol{\beta}$. The least squares (LS) estimate of $\boldsymbol{\beta}$ given by
\begin{equation}
\boldsymbol{\beta}_{LS}=\underset{{\bf b}\in \mathbb{R}^p}{\arg\min}\|{\bf y}-{\bf X}{\bf b}\|_2^2={\bf X}^{\dagger}{\bf y}
\end{equation}
is the natural choice for estimating $\boldsymbol{\beta}$ when outlier ${\bf g}_{out}={\bf 0}_n$. However, the error in the LS estimate $\boldsymbol{\beta}_{LS}$ becomes unbounded even when a single non zero entry in ${\bf g}_{out}$ becomes unbounded.  This motivated the development of the robust linear regression models discussed next.
 
\subsection{Prior art on robust regression with sparse outliers} 
Classical techniques proposed to estimate $\boldsymbol{\beta}$ in the presence of sparse outliers can be broadly divided into two categories. First category  includes algorithms like  least absolute deviation (LAD),  Hubers' M-estimate\cite{maronna2006wiley} and their derivatives which  replace the $l_2$ loss function in LS with more robust loss functions. Typically, these estimates have low break down points\footnote{BDP is defined as the fraction of outliers $k_g/n$ upto which a robust regression algorithm can deliver satisfactory performance.} (BDP). Second category includes algorithms like random sample consensus (RANSAC)\cite{ransac}, least median of squares (LMedS), least trimmed squares\cite{rousseeuw2005robust} etc. These algorithms try to identify outlier free observations by repeatedly sampling $O(p)$ observations from the total $n>p$ observations $\{({\bf y}_i,{\bf X}_{i,:})\}_{i=1}^n$. RANSAC, LMedS etc. have better BDP compared to M-estimation, LAD etc. However, the computational complexity of RANSAC, LMedS etc. increases exponentially with $p$.  This makes LMedS, RANSAC etc.  impractical for regression models with large  $p$ and $n$.  
  
 A significant breakthrough in robust regression with sparse outliers is the introduction of sparse recovery  principles  inspired  robust regression (SRIRR) techniques that explicitly utilize the sparsity of outliers\cite{fuchs1999inverse}. SRIRR schemes have high BDPs, (many have) explicit finite sample guarantees and are computationally very efficient in comparison to LMedS, RANSAC etc.  SRIRR algorithms  can also be classified into two categories. Category 1 includes algorithms like basis pursuit robust regression (BPRR)\cite{koushik_conf,koushik},  linear programming (LP) and second order conic programming (SOCP) formulations in \cite{error_correction_candes}, Bayesian sparse robust regression (BSRR)\cite{koushik} etc. These algorithms first project ${\bf y}$ orthogonal to $span({\bf X})$  resulting in the following  sparse regression model 
 \begin{equation}\label{transformed}
 {\bf z}=({\bf I}^n-{\bf P}_{{\bf X}}){\bf y}=({\bf I}^n-{\bf P}_{{\bf X}}){\bf g}_{out}+\tilde{\bf w},
 \end{equation}
 where $\tilde{\bf w}=({\bf I}^n-{\bf P}_{\bf X}){\bf w}$. The sparse vector ${\bf g}_{out}$ is then estimated using ordinary sparse estimation algorithms. For example,  BPRR algorithm involves applying Basis pursuit de-noising  \cite{tropp2006just}
 \begin{equation}\label{BPRR}
 \hat{\bf g}_{out}=\underset{{\bf g} \in \mathbb{R}^n}{\min}\|{\bf g}\|_1\ \ s.t\  \|{\bf z}-({\bf I}^n-{\bf P}_{\bf X}){\bf g}\|_2\leq \lambda_{bprr}
 \end{equation}
 to the transformed model  (\ref{transformed}). The outliers are  then identified as $\hat{\mathcal{S}}_g=supp(\hat{{\bf g}_{out}})$ and removed. Finally, an LS estimate is computed using the outlier free data as follows. 
 \begin{equation}\label{removal}
\hat{\boldsymbol{\beta}}={\bf X}_{[n]/\hat{\mathcal{S}}_g,:}^{\dagger}{\bf y}_{[n]/\hat{\mathcal{S}}_g}
\end{equation}   
Likewise, BSRR applies relevance vector machine\cite{tipping2001sparse} to estimate ${\bf g}_{out}$ from (\ref{transformed}).

The second category of SRIRR algorithms  include techniques such as robust maximum a posteriori (RMAP)[Eqn.5,\cite{bdrao_robust}], self scaled regularized robust regression ($S^2R^3$) \cite{self}, robust sparse Bayesian learning (RSBL)\cite{bdrao_robust}, greedy algorithm for robust de-noising (GARD)\cite{gard}, algorithm for robust outlier support identification (AROSI)\cite{arosi}, iterative procedure for outlier detection (IPOD)\cite{she2011outlier} etc.  try to jointly estimate the regression vector $\boldsymbol{\beta}$ and the sparse outlier ${\bf g}_{out}$.
For example, RMAP solves the  optimization problem,
\begin{equation}\label{RMAP}
\hat{\boldsymbol{\beta}}, \hat{{\bf g}_{out}}=\underset{{\bf b} \in \mathbb{R}^p, {\bf g} \in \mathbb{R}^n}{\min}\|{\bf y}-{\bf X} {\bf b}-{\bf g}\|_2^2+\lambda_{rmap} \|{\bf g}\|_1.
\end{equation}
whereas, AROSI solves the optimization problem
\begin{equation}\label{arosi}
\hat{\boldsymbol{\beta}}, \hat{{\bf g}_{out}}=\underset{{\bf b} \in \mathbb{R}^p, {\bf g} \in \mathbb{R}^n}{\min}\|{\bf y}-{\bf X} {\bf b}-{\bf g}\|_1+\lambda_{arosi} \|{\bf g}\|_0.
\end{equation}
Likewise, GARD  is a greedy iterative algorithm  to solve the sparsity constrained joint estimation problem
\begin{equation}
\hat{\boldsymbol{\beta}}, \hat{{\bf g}_{out}}=\underset{{\bf b} \in \mathbb{R}^p, {\bf g} \in \mathbb{R}^n}{\min} \|{\bf g}\|_0 \ \ s.t\ \ \|{\bf y}-{\bf X}{\bf b}-{\bf g}\|_2\leq \lambda_{gard}
\end{equation} 
Note that the sparsity inducing $l_0$ and $l_1$ penalties in RMAP, AROSI and GARD are applied only to the outlier  ${\bf g}_{out}$. Similarly, when the sparsity level $k_g$ is known \textit{ a priori}, GARD can also be used to solve the joint estimation problem 
\begin{equation}
\hat{\boldsymbol{\beta}}, \hat{{\bf g}_{out}}=\underset{{\bf b} \in \mathbb{R}^p, {\bf g} \in \mathbb{R}^n}{\min} \|{\bf y}-{\bf X}{\bf b}-{\bf g}\|_2 \ \ s.t\ \ \|{\bf g}\|_0\leq k_g.
\end{equation}
\subsection{Availability of noise statistics}
SRIRR techniques with explicit performance guarantees\footnote{ Theoretically, Bayesian algorithms like BSRR, RSBL etc. can be operated with or without the explicit \textit{a priori} knowledege of $\sigma^2$. However, the performance of these iterative algorithms depend crucially on the initialization values of $\sigma^2$,  the choice of which is not discussed well in literature. Further, unlike algorithms like RMAP, BPRR etc.,  these algorithms does not have any  performance guarantees to the best of our knowledge.  } like RMAP, BPRR, $S^2R^3$ etc.  require \textit{a priori} knowledge of  inlier  statistics like $\{\|{\bf w}\|_2,\sigma^2\}$ for efficient operation, whereas,  GARD requires a priori knowledge of either $\{\|{\bf w}\|_2,\sigma^2\}$ or outlier statistics like $k_g$ for efficient operation. In particular, authors suggested to set $\lambda_{bprr}=\sqrt{\frac{n-p}{n}}\|{\bf w}\|_2$, $\lambda_{rmap}=\sigma\sqrt{\frac{2\log(n)}{3}}$, $\lambda_{arosi}=5\sigma$ and $\lambda_{gard}=\|{\bf w}\|_2$  for BPRR, RMAP and AROSI respectively. However,   inlier  statistics like $\{\|{\bf w}\|_2,\sigma^2\}$ and outlier statistics like $k_g$ are unknown \textit{a priori} in most practical applications.  Indeed, it is possible to separately estimate $\sigma^2$ using M-estimation, LAD etc. \cite{dielman2006variance}. For example, a widely popular estimate of $\sigma^2$ is 
\begin{equation}\label{l1noise}
\hat{\sigma}=\frac{1}{0.675}median\{|{\bf r}_{LAD}(k)|:{\bf r}_{LAD}(k)\neq 0\},
\end{equation}
where ${\bf r}_{LAD}={\bf y}-{\bf X}\hat{\boldsymbol{\beta}}_{LAD}$ is the residual corresponding to the LAD estimate of $\boldsymbol{\beta}$ given by $\hat{\boldsymbol{\beta}}_{LAD}=\underset{{\bf b} \in \mathbb{R}^p}{\arg\min}\|{\bf y}-{\bf X}{\bf b}\|_1$\cite{arosi,bdrao_robust}. Another popular estimate is
\begin{equation}\label{mad}
\hat{\sigma}=1.4826\ MAD({\bf r}),
\end{equation}
where ${\bf r}$ is the residual corresponding to the LAD or M-estimate of $\boldsymbol{\beta}$. Median absolute deviation (MAD) of ${\bf r}\in \mathbb{R}^n$ is given by $MAD({\bf r})=\underset{k \in [n]}{median}(|{\bf r}(k)-\underset{j \in [n]}{median}({\bf r}(j))|)$. However, these separate noise variance estimation schemes  will  increase the computational burden of SRIRR algorithms. Further, the analytical characterization   of   SRIRR algorithms with estimated noise statistics is not  discussed in literature to the best of our knowledge. Numerical simulations presented in section \rom{6} indicate that the performance of  SRIRR algorithms like RMAP, BPRR, AROSI etc. deteriorates significantly when true $\sigma^2$ is replaced  with estimated $\sigma^2$. This degradation of performance can be directly attributed to the low BDP of LAD, M-estimation etc.  which are typically used to estimate $\sigma^2$.  No scheme to estimate the outlier sparsity $k_g$ is discussed in open literature to the best of our knowledge. 

\subsection{Contribution of this article}
 This article proposes a novel SRIRR technique called  residual ratio thresholding based GARD (RRT-GARD) to perform robust regression without the knowledge of  noise statistics like $\{\|{\bf w}\|_2,\sigma^2,k_g\}$. RRT-GARD involves a single hyper parameter $\alpha$ which can be set without the knowledge  of $\{\|{\bf w}\|_2,\sigma^2,k_g\}$.   We provide both finite sample and asymptotic  analytical guarantees  for  RRT-GARD. Finite sample guarantees indicate that RRT-GARD can correctly identify all the outliers under the same assumptions on design matrix ${\bf X}$ required by GARD with \textit{a priori} knowledge of $\{\|{\bf w}\|_2,\sigma^2,k_g\}$. However, to achieve support recovery, the outlier magnitudes have to be slightly higher than that required by GARD with \textit{a priori} knowledge of $\{\|{\bf w}\|_2,\sigma^2,k_g\}$.  Asymptotic analysis indicates that RRT-GARD and GARD with \textit{a priori} knowledge of $\{\|{\bf w}\|_2,\sigma^2,k_g\}$ are  identical as $n\rightarrow \infty$. Further, RRT-GARD is asymptotically tuning free in the sense that values of $\alpha$ over a very wide  range deliver similar results as $n\rightarrow \infty$. {  When the sample size $n$ is finite, we show through extensive numerical simulations  that a value of $\alpha=0.1$ delivers a performance very close to the best  performance achievable  using RRT-GARD. Such a fixed value of $\alpha$ is also analytically shown to result in  the accurate recovery of outlier support with a probability exceeding $1-\alpha$ when the outlier components  are sufficiently  stronger than the inlier noise.}  Further, RRT-GARD is numerically  shown to deliver a highly competitive estimation performance when compared with popular SRIRR techniques like GARD, RMAP, BPRR, AROSI, IPOD etc.  The competitive performance  of RRT-GARD is also demonstrated in the context of outlier detection in  real data sets. The numerical results in this article also provide certian heuristics to improve the performance of algorithms like AROSI when used  with estimated noise statistics.  
  
 \subsection{Organization of this article} 
 This article is organized as follows. Section \rom{2}   presents the GARD algorithm.  Section \rom{3} presents  the behaviour of residual ratio statistic. Section \rom{4} presents RRT-GARD algorithm.  Section \rom{5} provides analytical guarantees for RRT-GARD.   Section \rom{6}  presents numerical simulations.  
  
 \section{Greedy Algorithm For Robust De-noising(GARD)} 
The GARD algorithm described in TABLE \ref{tab:gard} is a recently proposed robust regression technique that tries to jointly estimate $\boldsymbol{\beta}$ and ${\bf g}_{out}$ and it   operates as follows.  Starting with an outlier support estimate  ${\mathcal{S}}^0_{GARD}=\phi$, the GARD algorithm in each step identifies  a possible outlier  based on the maximum residual in the previous estimate, i.e.,  $\hat{i}_k=\underset{i=1,\dotsc,n}{\arg\max}|{\bf r}^{k-1}_{GARD}(i)|$ and aggregate this newly found support index to the existing support estimate, i.e., ${\mathcal{S}}^k_{GARD}={\mathcal{S}}^{k-1}_{GARD}\cup \hat{i}_k$.  Later, $\boldsymbol{\beta}$ and ${\bf g}_{out}({{\mathcal{S}}^k_{GARD}})$ are jointly estimated using the LS estimate and the residual is updated using this updated estimate of $\boldsymbol{\beta}$ and ${\bf g}_{out}({{\mathcal{S}}_{GARD}^k})$. Please note that the matrix inverses and residual computations in each iteration of GARD can be iteratively computed\cite{gard}. This makes GARD a very computationally  efficient tool for  robust regression.
\begin{table}\centering

\begin{tabular}{|l|}
\hline
{\bf Input:-} Observed vector ${\bf y}$, Design Matrix ${\bf X}$ \\ Inlier statistics $\{\|{\bf w}\|_2, \sigma^2\}$ or user specified sparsity level $k_{user}$. \\
{\bf Initialization:-} ${\bf A}^{0}={\bf X}$, ${\bf r}_{GARD}^{0}=({\bf I}^n-{\bf P}_{{\bf A}^{0}}){\bf y}$. $k=1$. ${\mathcal{S}}_{GARD}^0=\phi$.\\
Repeat Steps 1-4 until\ $\|{\bf r}^{k}_{GARD}\|_2\leq \|{\bf w}\|_2$, $\|{\bf r}^{k}_{GARD}\|_2 \leq  \epsilon^{\sigma}$\\ or $card(\mathcal{S}_{GARD}^k)=k_{user}$ if given $\|{\bf w}\|_2$, $\sigma^2$ and $k_{user}$ respectively. \\
{\bf Step 1:-} Identify the strongest residual in ${\bf r}^{k-1}_{GARD}$, i.e., \\   \ \ \ \ \ \ \ \ \ \ $\hat{i}_k=\underset{i=1,\dotsc,n}{\arg\max}|{\bf r}^{k-1}_{GARD}(i)|$. ${\mathcal{S}}^k_{GARD}={\mathcal{S}}^{k-1}_{GARD}\cup \hat{i}_k$. \\
{\bf Step 2:-} Update the matrix ${\bf A}^{k}=[{\bf X}\ \ \ {\bf I}_{{\mathcal{S}}^k_{GARD}}^n]$.\\
{\bf Step 3:-} Estimate $\boldsymbol{\beta}$ and ${\bf g}_{out}({{\mathcal{S}}^k_{GARD}})$ as $
[{\hat{\boldsymbol{\beta} }}^T  \hat{\bf g}_{out}({{\mathcal{S}}^k_{GARD}})^T ]^T= {{\bf A}^{k}}^{\dagger}{\bf y}.
$\\
{\bf Step 4:-} Update the residual ${\bf r}^{k}_{GARD}={\bf y}-{\bf A}^k[{\hat{\boldsymbol{\beta} }}^T  \hat{\bf g}_{out}({{\mathcal{S}}^k_{GARD}})^T ]^T$=\\
\ \ \ \ \ \ \ \ \ \ $({\bf I}^n-{\bf P}_{{\bf A}^{k}}){\bf y}$.           $k \leftarrow k+1$. \\
 {\bf Output:-} Signal estimate $\hat{\boldsymbol{\beta}}$. Outlier support estimate ${\mathcal{S}}_{GARD}^k$.\\
\hline
\end{tabular}
\caption{GARD algorithm.  $\epsilon^{\sigma}=\sigma\sqrt{n+2\sqrt{n\log(n)}}$  }
\label{tab:gard} 
\end{table}
\subsection{Stopping rules for GARD}
 An important practical aspect regarding GARD is its' stopping rule, i.e., how many iterations of GARD are required?  When the inlier noise ${\bf w}={\bf 0}_n$, the residual ${\bf r}^{k}_{GARD}$ will be equal to ${\bf 0}_n$ once all the non zero outliers ${\bf g}_{out}(\mathcal{S}_g)$ are identified. However, this is not possible when the  inlier noise ${\bf w}\neq {\bf 0}_n$.  When ${\bf w}\neq {\bf 0}_n$, \cite{gard} proposes to run  GARD  iterations  until  $\|{\bf r}^{k}_{GARD}\|_2\leq \|{\bf w}\|_2$. GARD with this stopping rule is denoted by GARD($\|{\bf w}\|_2$). However, access to a particular realisation of ${\bf w}$ is nearly impossible and in comparison, assuming \textit{a priori} knowledge of inlier noise variance $\sigma^2$ is a much more realisable assumption. Note that  ${\bf w}\sim \mathcal{N}({\bf 0}_n,\sigma^2{\bf I}^n)$ with $\epsilon^{\sigma}=\sigma\sqrt{n+2\sqrt{n\log(n)}}$ satisfies 
  \begin{equation}
  \mathbb{P}\left(\|{\bf w}\|_2<\epsilon^{\sigma}\right)\geq 1-1/n
  \end{equation}
 \cite{cai2011orthogonal}. Hence, $\epsilon^{\sigma}$ is a high probability  upper bound on $\|{\bf w}\|_2$ and  one can stop GARD iterations for Gaussian noise once $\|{\bf r}^{k}_{GARD}\|_2\leq \epsilon^{\sigma}$. GARD with this stopping rule is denoted by GARD($\sigma^2$). When the sparsity level of the outlier, i.e., $k_g$ is known \textit{a priori}, then one can stop GARD  after $k_g$ iterations, i.e., set $k_{user}=k_g$. This stopping rule is denoted by GARD($k_g$). 
  \subsection{Exact outlier support recovery using GARD}
  The performance  of GARD depends very much on the relationship between regressor subspace, i.e., $span({\bf X})$ and   the $k_g$ dimensional outlier subspace, i.e., $span({\bf I}_{\mathcal{S}_{ g}}^n)$. This relationship is captured using the quantity $\delta_{k_g}$ defined next. Let the QR decomposition of ${\bf X}$ be given by ${\bf X}={\bf Q}{\bf R}$, where ${\bf Q} \in \mathbb{R}^{n \times p}$ is an orthonormal projection matrix onto the column subspace of ${\bf X}$ and ${\bf R}$ is a $p\times p$ upper triangular matrix. Clearly, $span({\bf X})=span({\bf Q})$.  \\
{\bf Definition 1:-}  Let $\tilde{S}$ be any subset of $\{1,2,\dotsc,n\}$  with $card(\tilde{S})=k_g$ and $\delta_{\tilde{S}}$ be the smallest value of $\delta$ such that $|{\bf v}^T{\bf u}|\leq \delta\|{\bf u}\|_2\|{\bf v}\|_2$,  $\forall {\bf v}\in span({\bf Q})$ and $\forall {\bf u}\in span({\bf I}_{\tilde{S}}^n)$. Then $\delta_{k_g}=\min\{\delta_{\tilde{S}}:\tilde{S}\subset\{1,2,\dotsc,n\},card(\tilde{S})=k_g\}$ \cite{gard}. 

 In words, $\delta_{k_g}$ is the smallest angle between the regressor subspace $span({\bf X})=span({\bf Q})$ and any  $k_g$ dimensional  subspace of the form $span({\bf I}_{\tilde{\mathcal{S}}}^n)$. In particular, the angle between regressor subspace $span({\bf Q})$ and the  outlier subspace $span({\bf I}_{{\mathcal{S}_{ g}}}^n)$ must be greater than or equal to  $\delta_{k_g}$.
\begin{remark} Computing $\delta_{k_g}$ requires the computation of $\delta_{\tilde{\mathcal{S}}}$ in Definition 1  for all the $n\choose {k_g}$ $k_g$ dimensional outlier subspaces. Clearly, the computational complexity of this increases with $k_g$ as $O(n^{k_g})$. Hence,  computing $\delta_{k_g}$ is computationally infeasible. Analysis of  popular   robust regression techniques like BPRR, RMAP, AROSI etc. are also carried out in terms of  matrix properties such as smallest principal angles\cite{koushik}, leverage constants \cite{arosi} etc.   that are impractical to compute.  
\end{remark} 
The performance guarantee for GARD  in terms of $\delta_{k_g}$ and ${\bf g}_{min}=\underset{j \in \mathcal{S}_g}{\min}|{\bf g}_{out}(j)|$ \cite{gard} is summarized below.
\begin{lemma}\label{lemma_gard}
 Suppose that $\delta_{k_g}$ satisfies $\delta_{k_g}<\sqrt{\dfrac{{\bf g}_{min}}{2\|{\bf g}_{out}\|_2}}$. Then, GARD($k_g$) and GARD$(\|{\bf w}\|_2)$ identify the outlier support $\mathcal{S}_{ g}$  provided that $\|{\bf w}\|_2\leq \epsilon_{GARD}=({{\bf g}_{min}-2\delta_{k_g}^2\|{\bf g}_{out}\|_2})/({2+\sqrt{6}})$.   
\end{lemma}
\begin{cor} When ${\bf w}\sim \mathcal{N}({\bf 0}_n,\sigma^2{\bf I}^n)$, $\|{\bf w}\|_2\leq \epsilon^{\sigma}$ with a probability greater than $1-1/n$. Hence, if $\delta_{k_g}<\sqrt{\dfrac{{\bf g}_{min}}{2\|{\bf g}_{out}\|_2}}$ and  $\epsilon^{\sigma}\leq \epsilon_{GARD}$, then  GARD($k_g$) and GARD($\sigma^2$) identify $\mathcal{S}_g$ with probability greater than $1-1/n$. 
 \end{cor}
Lemma \ref{lemma_gard} and Corollary 1 state that GARD can identify the outliers correctly once  the outlier magnitudes are sufficiently higher than the inlier magnitudes  and the angle between outlier and regressor subspaces is sufficiently small (i.e., $\delta_{k_g}^2< {\bf g}_{min}/2\|{\bf g}_{out}\|_2$). 
\section{Properties of residual ratios}
{ As discussed in section \rom{2}, stopping rules for  GARD based on the  behaviour of residual norm $\|{\bf r}^{k}_{GARD}\|_2$ or outlier sparsity level are highly intuitive. However, these stopping rules require \textit{a priori} knowledge of inlier statistics $\{\sigma^2,\|{\bf w}\|_2\}$  or outlier sparsity $k_g$ which are rarely available. In this section, we analyse the properties of the  residual ratio statistic $RR(k)=\frac{\|{\bf r}^{k}_{GARD}\|_2}{\|{\bf r}^{k-1}_{GARD}\|_2}$ and establish its' usefulness in identifying the  outlier support $\mathcal{S}_g$ from the support sequence generated by GARD  without having any \textit{a priori} knowledge of noise   statistics $\{\sigma^2,\|{\bf w}\|_2,k_g\}$ or their estimates. Statistics based on residual ratios are not widely used in sparse recovery or robust regression literature yet.    In a recent related contribution, we successfully applied  residual ratio techniques   operationally similar to the  one discussed in this article  for sparse recovery in underdetermined linear regression models \cite{kallummil18a}. This residual ratio technique  \cite{kallummil18a} can be used to estimate sparse vectors $\boldsymbol{\beta}$  in an outlier free  regression model ${\bf y}={\bf X}\boldsymbol{\beta}+{\bf w}$ with finite sample guarantees even when ${\bf X}$ is not full rank and the statistics of ${\bf w}$  are unknown \textit{a priori}. This finite sample guarantees are applicable only when the noise ${\bf w}\sim \mathcal{N}({\bf 0}_n,\sigma^2{\bf I}^n)$.  This technique can be used instead of BPDN or relevance vector machine (used in BPRR and BSRR)  to estimate outlier support $\mathcal{S}_g$ from ${\bf z}=({\bf I}^n-{\bf P}_{{\bf X}}){\bf y}=({\bf I}^n-{\bf P}_{{\bf X}}){\bf g}_{out}+\tilde{\bf w}$ in (\ref{Projection}) as a part of projection based robust regression. However, it is impossible to derive any finite sample or asymptotic guarantees for \cite{kallummil18a} in this situation since the noise $\tilde{\bf w}$ in  (\ref{Projection}) is correlated with a rank deficient correlation matrix $\sigma^2({\bf I}^n-{\bf P}_{\bf X})$, whereas,  \cite{kallummil18a} expects the noise to be uncorrelated. Further, empirical evidences \cite{arosi} suggest that the  joint vector and outlier estimation approach used in RMAP, AROSI, GARD etc. are superior in performance compared to the projection based approaches like BPRR. The main contribution of this article is to transplant the operational philosophy in \cite{kallummil18a} developed for sparse vector estimation to  the different problem of joint regression vector and outlier estimation (the strategy employed in GARD) and develop finite and large sample guarantees using the results available for GARD.  }

We begin in our analysis of residual ratios by stating some of it's fundamental properties which are based on the properties of support sequences generated by GARD algorithm.
\begin{lemma}\label{lemma:gard}
\label{gard} The support estimate and residual sequences produced by GARD satisfy the following properties\cite{gard}.\\
A1). Support estimate ${\mathcal{S}}_{GARD}^k$ is monotonically increasing  in the sense that ${\mathcal{S}}_{GARD}^{k_1} \subset {\mathcal{S}}_{GARD}^{k_2}$ whenever $k_1<k_2$. \\
A2). The residual norm $\|{\bf r}^{k}_{GARD}\|_2$ decreases monotonically, i.e., $\|{\bf r}^{k_2}_{GARD}\|_2 \leq \|{\bf r}^{k_1}_{GARD}\|_2$ whenever $k_1< k_2$. 
\end{lemma}
As a consequence of A2) of Lemma \ref{lemma:gard}, residual ratios are upper bounded by one, i.e.,  $RR(k)\leq 1$. Also given the non negativity of residual norms, one has $RR(k)\geq 0$. Consequently, residual ratio statistic is a bounded random variable taking values in $[0,1]$. Even though residual norms are non increasing, please note that residual ratio statistic does not exhibit any monotonic behaviour.

\subsection{Concept of minimal superset}
Consider operating the GARD algorithm  with $k_{user}= k_{max}$, where $k_{max}$ is a user defined value satisfying $k_{max}\gg k_g$. Let $\mathcal{S}_{GARD}^k$ and ${\bf r}_{GARD}^k$ for $k=1,\dotsc,k_{max}$ be the support estimate and residual  after the $k^{th}$ GARD iteration in TABLE \ref{tab:gard}. 
The concept of minimal superset is  important in the analysis of GARD support estimate sequence $\{{\mathcal{S}}_{GARD}^k\}_{k=1}^{k_{max}}$.\\
{\bf Definition 2:-} The minimal superset in the GARD support estimate sequence $\{{\mathcal{S}}^k_{GARD}\}_{k=1}^{k_{max}}$ is given by ${\mathcal{S}}^{k_{min}}_{GARD}$, where $k_{min}=\min\{k:\mathcal{S}_g\subseteq {\mathcal{S}}^k_{GARD}\}$. When the set $\{k:\mathcal{S}_g\subseteq {\mathcal{S}}^k_{GARD}\}=\phi$, it is assumed that $k_{min}=\infty$ and $\mathcal{S}^{k_{min}}_{GARD}=\phi$. 

In words, $k_{min}$ is the first time GARD support estimate ${\mathcal{S}}^k_{GARD}$ covers the outlier support $\mathcal{S}_g$.  Please note that $k_{min}$ is an unobservable R.V that depends on the data $\{{\bf y},{\bf X},{\bf w}\}$. Since, $\mathcal{S}_g \not\subseteq \mathcal{S}^k_{GARD}$ for $k<k_g$, the random variable $k_{min}$ satisfies $k_{min}\geq k_g$. Further, when $k_g\leq k_{min}<k_{max}$, the monotonicity of support estimate $\mathcal{S}^k_{GARD}$ implies that $\mathcal{S}_g\subset \mathcal{S}_{GARD}^k$ for $k_{min}<k\leq k_{max}$. Based on the value of $k_{min}$, the following three situations can happen.  In the following  running example suppose that $\mathcal{S}_g=\{1,2\}$ (i.e., $k_g=2$), $n=10$ and $k_{max}=4$.  

{\bf Case 1:-} $k_{min}=k_g$. The  outlier support $\mathcal{S}_g$ is present in the  sequence $\{\mathcal{S}^k_{GARD}\}_{k=1}^{k_{max}}$. For example, let $\mathcal{S}_{GARD}^1=\{1\},\mathcal{S}_{GARD}^2=\{1,2\},\mathcal{S}_{GARD}^3=\{1,2,3\}$ and $\mathcal{S}_{GARD}^4=\{1,3,2,7\}$. Here $k_{min}=k_g$ and $\mathcal{S}_{GARD}^{k_{min}}=\mathcal{S}_g$.   Lemma \ref{lemma_gard} implies that $k_{min}=k_g$ and $\mathcal{S}^{k_g}_{GARD}=\mathcal{S}_g$ if $\|{\bf w}\|_2\leq \epsilon_{GARD}$.

{\bf Case 2:-} $k_g<k_{min}\leq k_{max}$. In this case,  outlier support $\mathcal{S}_g$ is not present in $\{\mathcal{S}^k_{GARD}\}_{k=1}^{k_{max}}$. However, a superset of the  outlier support $\mathcal{S}_g$ is present in $\{\mathcal{S}^k_{GARD}\}_{k=1}^{k_{max}}$. For example, let  $\mathcal{S}_{GARD}^1=\{1\},\mathcal{S}_{GARD}^2=\{1,3\},\mathcal{S}_{GARD}^3=\{1,3,2\}$ and $\mathcal{S}_{GARD}^4=\{1,3,2,7\}$.   Here $k_{min}=3>k_g=2$ and $\mathcal{S}_{GARD}^{k_{min}}\supset \mathcal{S}_g$. 

{\bf Case 3:-} $k_{min}=\infty$. Neither the outlier support $\mathcal{S}_g$ nor a superset of $\mathcal{S}_g$ is present in the GARD solution path.  For example, let $\mathcal{S}_{GARD}^1=\{1\},\mathcal{S}_{GARD}^2=\{1,3\},\mathcal{S}_{GARD}^3=\{1,3,5\}$ and $\mathcal{S}_{GARD}^4=\{1,3,5,7\}$. Since no support estimate satisfies $\mathcal{S}_g \subseteq \mathcal{S}_{GARD}^k$, $k_{min}=\infty$.

\subsection{ Implications for estimation performance}
Minimal superset has the following impact on the GARD estimation performance. 
Since $\mathcal{S}_g\subseteq  \mathcal{S}_{GARD}^{k_{min}}$, ${\bf g}_{out}={\bf I}^n_{\mathcal{S}_{GARD}^{k_{min}}}{\bf g}_{out}(\mathcal{S}_{GARD}^{k_{min}})$. Hence  ${\bf y}$ can be written as  \begin{equation}
{\bf y}={\bf A}^{k_{min}}[\boldsymbol{\beta}^T \ {\bf g}_{out}(\mathcal{S}_{GARD}^{k_{min}})^T]^T+{\bf w}.
\end{equation}
 Consequently, the joint estimate  $[\hat{\boldsymbol{\beta}}^T \ \hat{{\bf g}_{out}}(\mathcal{S}_{GARD}^{k_{min}})^T]^T=({\bf A}^{k_{min}})^{\dagger}{\bf y}=[\boldsymbol{\beta}^T \ {\bf g}_{out}(\mathcal{S}^k_{GARD})^T]^T+({\bf A}^{k_{min}})^{\dagger}{\bf w}$
has error $\|\boldsymbol{\beta}-\hat{\boldsymbol{\beta}}\|_2$ independent of the outlier magnitudes.
Since, $\mathcal{S}_g\subset \mathcal{S}_{GARD}^k$, similar outlier free estimation performance can be delivered by support estimates $\mathcal{S}_{GARD}^k$ for $k\geq k_{min}$. 
 However, the estimation error due to the inlier noise, i.e., $\|({\bf A}^{k})^{\dagger}{\bf w}\|_2$ increases with increase in $k$. Similarly for $k<k_{min}$, the observation ${\bf y}$ can be written as  
\begin{equation}
 {\bf y}={\bf A}^{k}[\boldsymbol{\beta}^T \ {\bf g}_{out}(\mathcal{S}_{GARD}^{k})^T]^T+{\bf w}+{\bf I}^n_{\mathcal{S}_g/ \mathcal{S}_{GARD}^k}{\bf g}_{out}(\mathcal{S}_g/ \mathcal{S}_{GARD}^k).
 \end{equation} Hence the joint estimate $[\hat{\boldsymbol{\beta}}^T \ \hat{{\bf g}_{out}}(\mathcal{S}_{GARD}^{k})^T]^T=({\bf A}^{k})^{\dagger}{\bf y}=[\boldsymbol{\beta}^T \ {\bf g}_{out}(\mathcal{S}^k_{GARD})^T]^T+({\bf A}^{k})^{\dagger}{\bf w}+({\bf A}^{k})^{\dagger}{\bf I}^n_{\mathcal{S}_g/ \mathcal{S}_{GARD}^k}{\bf g}_{out}(\mathcal{S}_g/\mathcal{S}_{GARD}^k)$ 
 has error $\|\boldsymbol{\beta}-\hat{\boldsymbol{\beta}}\|_2$ influenced by outliers. 
 Hence, when the outliers are strong, among all the support estimates $\{\mathcal{S}^k_{GARD}\}_{k=1}^{k_{max}}$ produced by GARD, the joint estimate corresponding to $\mathcal{S}_{GARD}^{k_{min}}$ delivers the best estimation performance. Consequently, identifying $k_{min}$ from the support estimate sequence $\{\mathcal{S}^k_{GARD}\}_{k=1}^{k_{max}}$ can leads to high quality estimation performance. The behaviour of residual ratio statistic $RR(k)$ described next provides a noise statistics oblivious way to identify $k_{min}$.

\begin{figure*}[htb]
\begin{multicols}{2}

    \includegraphics[width=1\linewidth]{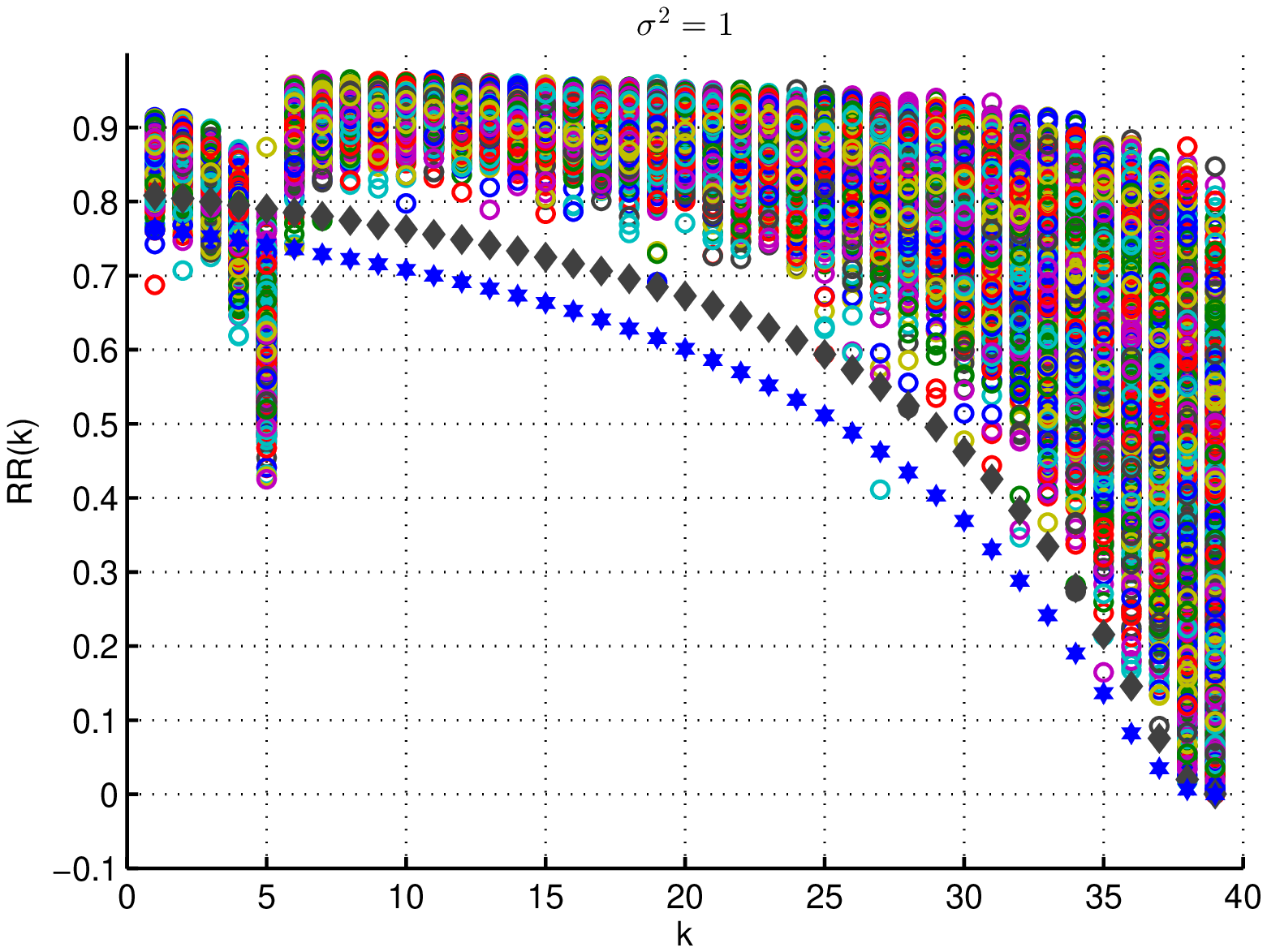} 
    \caption*{a). $\{RR(k)<\Gamma_{RRT}^{\alpha}(k),\forall\ k>k_{min}\}$ $0.5\%$ for ($\alpha=0.1$), $0.1\%$ for ($\alpha=0.01$) }
    
    \includegraphics[width=1\linewidth]{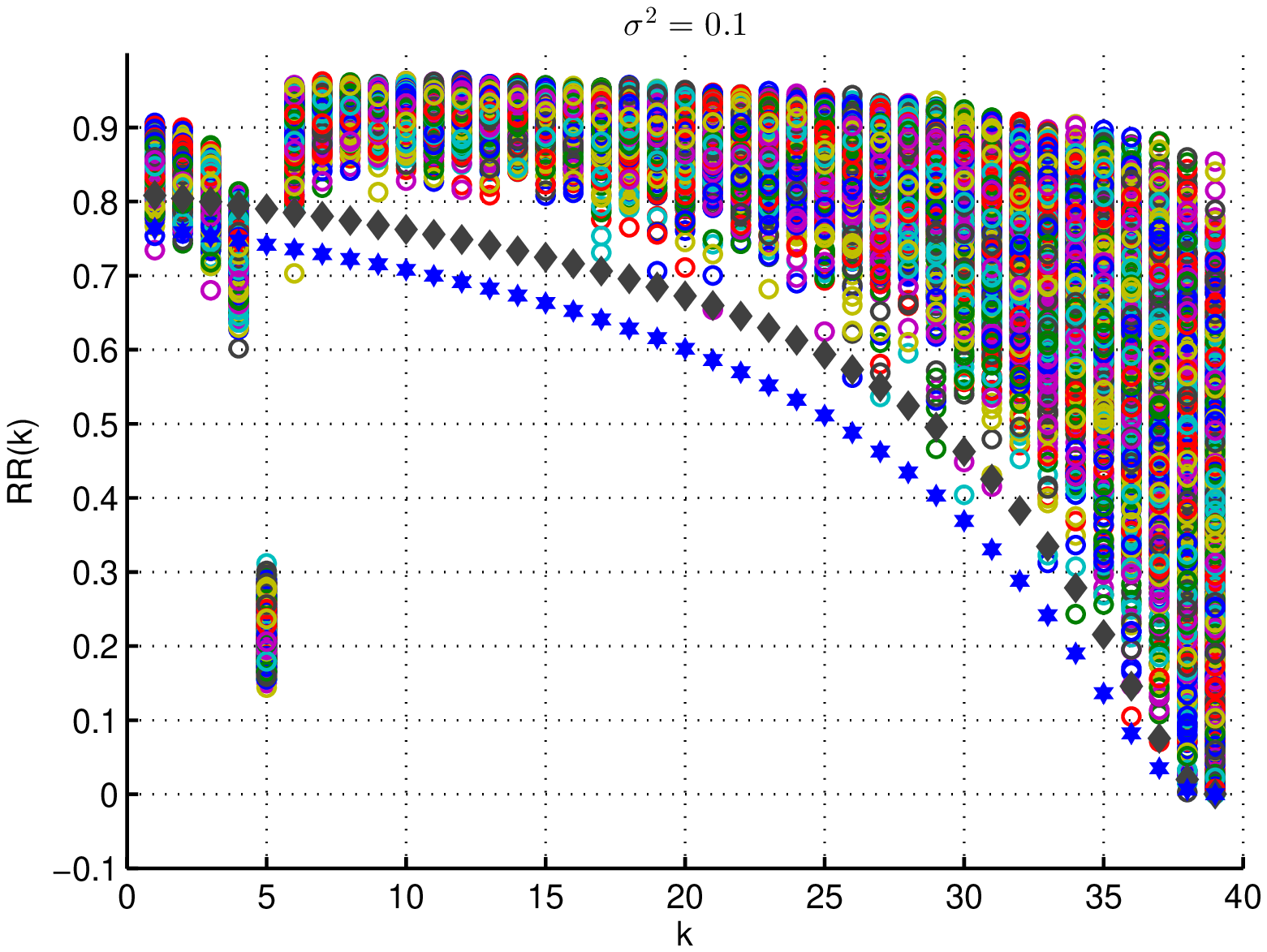}
    \caption*{a). $\{RR(k)<\Gamma_{RRT}^{\alpha}(k),\forall\ k>k_{min}\}$ $0.6\%$ for ($\alpha=0.1$), $0\%$ for ($\alpha=0.01$) }
   \end{multicols} 
   \caption{Behaviour of $RR(k)$ for the model described in Section \rom{3}.D. $\sigma^2=1$ (left) and $\sigma^2=0.1$ (right). Circles in Fig. 1 represents the values of $RR(k)$, diamond represents $\Gamma_{RRT}^{\alpha}$ with $\alpha=0.1$ and hexagon represents $\Gamma_{RRT}^{\alpha}$ with $\alpha=0.01$.  }
   \label{fig:evolution}
\end{figure*}
\subsection{Behaviour of residual ratio statistic $RR(k)$} 
We next analyse the behaviour of the  residual ratio statistic $RR(k)=\frac{\|{\bf r}^{k}_{GARD}\|_2}{\|{\bf r}^{k-1}_{GARD}\|_2}$ as $k$ increases from $k=1$ to  $k=k_{max}$.  Since the residual norms are decreasing according to Lemma \ref{lemma:gard}, $RR(k)$ satisfies  $0\leq   RR(k)\leq  1$. Theorem \ref{lemma:tf}  states the behaviour of $RR(k_{min})$ once the regularity conditions in Lemma \ref{lemma_gard} are satisfied.

\begin{theorem}\label{lemma:tf}
 Suppose that the matrix conditions in Lemma \ref{lemma_gard} are satisfied (i.e., $\delta_{k_g}<\sqrt{\dfrac{{\bf g}_{min}}{2\|{\bf g}_{out}\|_2}}$), then \\
 a). $\underset{\sigma^2\rightarrow 0}{\lim}\mathbb{P}(k_{min}=k_g)=1$. \\ 
 b). $RR(k_{min})\overset{P}{\rightarrow }0$ as $\sigma^2\rightarrow  0$. 
\end{theorem}
\begin{proof}
 Please see Appendix A for proof.
\end{proof}
Theorem \ref{lemma:tf} states  that when the matrix regularity conditions in Lemma \ref{lemma_gard} are satisfied, then with decreasing inlier variance $\sigma^2$ or equivalently with increasing difference between outlier and inlier powers, the residual ratio statistic  $RR(k_{min})$ takes progressively smaller and smaller values.    The following theorem characterizes  the behaviour of $RR(k)$ for $k\geq k_{min}$. 
\begin{theorem} \label{thm:Beta}
Let $F_{a,b}(x)$ be the cumulative distribution function (CDF) of $\mathbb{B}(a,b)$ R.V and $F^{-1}_{a,b}(x)$ be its' inverse CDF. Then, for all $0\leq \alpha\leq 1$  and for all $\sigma^2>0$, $\Gamma_{RRT}^{\alpha}(k)=\sqrt{F_{\frac{n-p-k}{2},0.5}^{-1}\left(\dfrac{\alpha}{k_{max}(n-k+1)}\right)}>0$ satisfies $\mathbb{P}\left(RR(k)>\Gamma_{RRT}^{\alpha}(k),\forall k\in\{k_{min}+1,k_{max}\}\right)\geq 1-\alpha$ . 
\end{theorem} 
\begin{proof} Please see Appendix B for proof. 
\end{proof}
{Theorem \ref{thm:Beta} can be understood as follows. Consider two sequences, \textit{viz.} the random sequence $\{RR(k)\}_{k=1}^{k_{max}}$  and the deterministic sequence $\{\Gamma_{RRT}^{\alpha}(k)\}_{k=1}^{k_{max}}$  which is dependent only on the matrix dimensions $(n,p)$. Then Theorem \ref{thm:Beta} states that the portion of the random sequence $\{RR(k)\}_{k=1}^{k_{max}}$  for $k>k_{min}$  will be   lower bounded  by the corresponding portion of the  deterministic sequence $\{\Gamma_{RRT}^{\alpha}(k)\}_{k=1}^{k_{max}}$ with a probability greater than $1-\alpha$. Please note that $k_{min}$ is itself a random variable. Also please note that Theorem \ref{thm:Beta} hold true for all values of $\sigma^2>0$. In contrast, Theorem \ref{lemma:tf} is true only when  $\sigma^2\rightarrow 0$. Also unlike Theorem \ref{lemma:tf}, Theorem \ref{thm:Beta} is valid even when the regularity conditions in Lemma \ref{lemma_gard} are not satisfied. }
\begin{lemma}\label{lemma:properties}The following properties of the function $\Gamma_{RRT}^{\alpha}(k)$ follow directly from the properties of inverse CDF and the definition of Beta distribution.\\
1). $\Gamma_{RRT}^{\alpha}(k)$ is a monotonically increasing function of $\alpha$ for $0\leq \alpha\leq k_{max}(n-k+1)$. In particular, $\Gamma_{RRT}^{\alpha}(k)=0$ for $\alpha=0$ and $\Gamma_{RRT}^{\alpha}(k)=1$ for $\alpha=k_{max}(n-k+1)$.\\
2). Since $\mathbb{B}(a,b)$ distribution is defined only for $a>0$ and $b>0$, Theorem \ref{thm:Beta} is valid only if $k_{max}\leq n-p-1$.
\end{lemma} 
\subsection{Numerical validation of Theorem \ref{lemma:tf} and Theorem \ref{thm:Beta}} 
We consider a design matrix  ${\bf X} \in \mathbb{R}^{50 \times 10}$ such that ${\bf X}_{i,j}\sim \mathcal{N}(0,1/n)$ and inlier noise  ${\bf w}\sim \mathcal{N}({\bf 0}_n,\sigma^2{\bf I}^n)$.  Outlier ${\bf g}_{out}$ has $k_g=5$ non zero entries. All the $k_g$ non-zero entries of ${\bf g}_{out}$ are fixed at 10. We  fix $k_{max}={n-p-1}$ which is the maximum value of $k$ upto which Theorem \ref{thm:Beta} hold true.  Fig. \ref{fig:evolution} presents $1000$ realisations of the sequence $\{RR(k)\}_{k=1}^{k_{max}}$ for two different values of $\sigma^2$. When $\sigma^2=1$, we have observed that $k_{min}=k_g=5$ in $999$ realizations out of the $1000$ realizations, whereas, $k_{min}=k_g=5$ in  all the $1000$ realizations when $\sigma^2=0.1$. As one can see from Fig. \ref{fig:evolution}, $RR(k_{min})$ decreases with decreasing $\sigma^2$ as claimed in Theorem \ref{lemma:tf}. Further, it is evident in  Fig. \ref{fig:evolution} that $RR(k)>\Gamma_{RRT}^{\alpha}(k)$ for all $ k>k_{min}$ in most of the realizations.  In both cases, the empirical evaluations of the probability of $\{RR(k)\geq \Gamma_{RRT}^{\alpha}(k),\forall k\geq k_{min}\}$ also agree with the $1-\alpha$ bound derived in Theorem \ref{thm:Beta}.  
\section{Residual ratio threshold based GARD}
The proposed RRT-GARD algorithm is based on the following observation. From Theorem  \ref{lemma:tf} and Fig. \ref{fig:evolution}, one can see that with decreasing   $\sigma^2$, $RR(k_{min})$ decreases to zero. This implies that with decreasing   $\sigma^2$, $RR(k_{min})$ is more likely to be smaller than $\Gamma_{RRT}^{\alpha}(k_{min})$.  At the same time, by Theorem \ref{thm:Beta}, $RR(k)$ for $k>k_{min}$ is lower bounded by $\Gamma_{RRT}^{\alpha}(k)$ which is independent of  $\sigma^2$. Hence, with decreasing $\sigma^2$,  the last index $k$ such that $RR(k)<\Gamma_{RRT}^{\alpha}(k)$ would correspond to $k_{min}$ with a high probability (for smaller values of $\alpha$).  Hence,  finding the last index $k$ such that $RR(k)$ is lower than $\Gamma_{RRT}^{\alpha}(k)$  can provide a very reliable and noise statistics oblivious way of identifying $k_{min}$. This observation motivates the RRT-GARD algorithm presented in TABLE \ref{tab:rrt-gard} which tries to identify $k_{min}$ using the last index $k$ such that $RR(k)$ is smaller than $\Gamma_{RRT}^{\alpha}(k)$. The efficacy of the RRT-GARD is visible in Fig. \ref{fig:evolution} itself.  When $\sigma^2=1$, the last index where $RR(k)<\Gamma_{RRT}^{\alpha}(k)$ corresponded to $k_{min}$ $99\%$ of time for $\alpha=0.1$ and $90\%$ of time for $\alpha=0.01$. For $\sigma^2=0.1$, the corresponding numbers are $99.4\%$ of the time for $\alpha=0.1$ and $100\%$ of time for $\alpha=0.01$. 
\begin{table}\centering
\begin{tabular}{|l|}
\hline 
{\bf Input:-}  Observed vector ${\bf y}$, design matrix ${\bf X}$, RRT parameter $\alpha$.  \\
{\bf Step 1:-} Run GARD with $k_{user}=k_{max}$. \\
{\bf Step 2:-} Estimate $k_{min}$ as $k_{RRT}=\max\{k:RR(k)\leq \Gamma_{RRT}^{\alpha}(k)\}$.\\ 
{\bf Step 3:-} Estimate $\boldsymbol{\beta}$ and ${\bf g}_{out}(\mathcal{S}^{k_{RRT}}_{GARD})$:\\
\ \ \ \ \ \ \ \ \ \ \ \ \ \ \ \  $[{\hat{\boldsymbol{\beta}} }^T  \hat{{\bf g}_{out}(\mathcal{S}^{k_{RRT}}_{GARD})}^T ]^T= {{\bf A}^{{k}_{RRT}}}^{\dagger}{\bf y}.$\\
{\bf Output:-} Signal estimate $\hat{\boldsymbol{\beta}}$. Outlier support estimate $\mathcal{S}_{RRT}=\mathcal{S}_{GARD}^{k_{RRT}}$\\ 
\hline 
\end{tabular}
\caption{Residual Ratio Threshold GARD: RRT-GARD}

\label{tab:rrt-gard}
\end{table}
\begin{table*}[htb]
\centering
\begin{tabular}{|l|l|l|l|l|l|}
\hline
Algorithm& Complexity order & \multicolumn{2}{c|}{ Noise Variance Estimation} & \multicolumn{2}{c|}{Overall Complexity}   \\ \hline
\multicolumn{2}{|c|}{}& LAD & M-est &LAD & M-est \\ \hline 
\hline
GARD($\sigma^2$) (when $k_g\ll n$) & $O\left(k_g^3+np^2\right)$& $O(n^3)$ & $O(np^2)$ &  $O(n^3)$ & $O(np^2+k_g^3)$ \\
\hline
GARD($\sigma^2$) (when $k_g=O(n)$) & $O\left(n^3\right)$ &$O(n^3)$ & $O(np^2)$ &$O(n^3)$ &$O(n^3)$\\ \hline
RRT-GARD& $O(n^3)$ &-&-&\multicolumn{2}{c|}{$O(n^3)$} \\ \hline
RMAP \cite{bdrao_robust,gard}& $O(n^3)$  &$O(n^3)$& $O\left(np^2\right)$&$O\left(n^3\right)$&$O\left(n^3\right)$\\ \hline
BPRR\cite{koushik_conf}& $O(n^3)$ &$O(n^3)$& $O(np^2)$&$O\left(n^3\right)$&$O\left(n^3\right)$  \\ \hline
M-est\cite{gard}& $O(np^2)$ &-&-& \multicolumn{2}{c|}{ $O(np^2)$}\\  \hline
AROSI\cite{arosi}& $O(n^3)$ & - & $O(np^2)$ &$O(n^3)$ & $O(n^3)$ \\ \hline
\end{tabular}
\caption{Complexity order of robust regression techniques. $p \ll n$. LAD based $\sigma^2$ estimation can be incorporated into AROSI. Hence no additional complexity is involved in AROSI with   LAD based $\sigma^2$ estimation. }
\label{tab:complexity}
\end{table*}
\begin{remark} When the set $\{k:RR(k)<\Gamma_{RRT}^{\alpha}(k)\}$ in Step 2 of RRT-GARD is empty, it is an indicator of the fact that $\sigma^2$ is high which in turn implies that the inlier and outlier powers are comparable.  In such situations,  we increase the value of $\alpha$ such that the  set $\{k:RR(k)<\Gamma_{RRT}^{\alpha}(k)\}$ is non empty. Mathematically, we set $\alpha$ to $\alpha_{new}$ where
\begin{equation}
\alpha_{new}=\underset{a \geq \alpha}{\min}\{\{k:RR(k)\leq \Gamma_{RRT}^{a}(k)\}\neq \phi\}
\end{equation}
Since $a=k_{max}n$ gives $\Gamma_{RRT}^{a}(1)=1$ (by Lemma \ref{lemma:properties}) and $RR(1)\leq 1$ always, it is true that $\alpha\leq \alpha_{new}\leq k_{max}n$ exists.  
\end{remark}
  
\begin{remark}{\bf Choice of $k_{max}$:- }
For the successful operation of RRT-GARD, i.e., to estimate $k_{min}$ and hence ${\mathcal{S}}^{k_{min}}_{GARD} $ accurately, it is required that $k_{max}\geq k_{min}$. However, $k_{min}$ being a R.V is difficult to be known \textit{a priori}. Indeed, when $\sigma^2$ is small, it is true that $k_{min}=k_g$  when $\delta_{k_g}<\sqrt{\frac{{\bf g}_{min}}{\|{\bf g}_{out}\|_2}}$. However,  nothing is assumed to be known about $k_g$ too. Hence, we set $k_{max}=n-p-1$, the maximum value of $k$ upto which $\Gamma_{RRT}^{\alpha}(k)$ is defined. Since matrices involved in GARD will become rank deficient at the $n-p+1$th iteration, $n-p$ is the maximum  number of iterations possible for GARD. Hence $k_{max}=n-p-1$ practically involves running GARD upto its' maximum possible sparsity level. Please note that this choice of $k_{max}$ is independent of the outlier and inlier statistics. 
\end{remark}
 \begin{remark} $k_{max}$ is a predefined data independent quantity. However, situations may arise such that the GARD iterations in  TABLE \ref{tab:rrt-gard} be stopped at an intermediate iteration $\tilde{k}_{max}<k_{max}$ due to the rank deficiency of ${\bf A}^{k}=[{\bf X},\ {\bf I}^n_{\mathcal{S}^k_{GARD}}]$. In those situations, we set $RR(k)$ for $\tilde{k}_{max}<k\leq k_{max}$ to one. Since $\Gamma_{RRT}^{\alpha}(k)<1$, substituting $RR(k)=1$ for $k>\tilde{k}_{max}$ will not alter the outcome of RRT-GARD as long as $k_{min}\leq \tilde{k}_{max}$.  All the theoretical guarantees derived for RRT-GARD will also remain true as long as $\tilde{k}_{max}\geq k_{min}$. Note that when $\mathcal{S}_g\not\subseteq \mathcal{S}_{GARD}^{\tilde{k}_{max}}$, all support estimates produced by GARD will be adversely  affected by outliers.
\end{remark}
\subsection{Computational Complexity of the RRT-GARD}
The computational  complexity order of RRT-GARD and some popular  robust regression methods are given in TABLE \ref{tab:complexity}. For algorithms requiring \textit{a priori} knowledge of $\{\sigma^2$, $\|{\bf w}\|_2\}$ etc., we compute the overall complexity order after including the complexity of estimating $\sigma^2$  using (\ref{l1noise}) or (\ref{mad}).   GARD with $k_g$ iterations has complexity $O\left(p^3+k_g^3/3+(n+3k_g)p^2+3k_gnp\right)$\cite{gard}. RRT-GARD involves $n-p-1$ iterations of GARD. Hence, the complexity of RRT-GARD is of the order $O(n^3+p^3)$. Thus, when the number of outliers is very small, i.e., $k_g\ll n$, then the complexity of RRT-GARD is  higher than the complexity of GARD itself. However, when the number of outliers $k_g=O(n)$,  both RRT-GARD and GARD have similar complexity order. Further, once we include the $O(n^3)$ complexity of LAD based $\sigma^2$ estimation,  GARD and RRT-GARD have same overall  complexity order. When $k_g$ is low and M-estimation based $\sigma^2$ estimate is used, GARD has significantly lower complexity than RRT-GARD. However,  the performance of GARD with M-estimation based $\sigma^2$ estimate is very poor.  Also note that the complexity order of RRT-GARD  is comparable to popular SRIRR techniques like BPRR, RMAP, AROSI etc.   M-estimation is also oblivious to inlier statistics. However, the performance of M-estimation is much inferior  compared to RRT-GARD. Hence, in spite of its'  lower complexity \textit{viz. a viz.} RRT-GARD, M-estimation has limited utility. 
\section{Theoretical analysis of RRT-GARD}
In this section, we analytically compare the proposed  RRT-GARD algorithm and GARD($\sigma^2$) in terms of  exact outlier support recovery. The sufficient condition for outlier support recovery using RRT-GARD is given in Theorem \ref{thm:rrt-gard}.
\begin{theorem}\label{thm:rrt-gard}  Suppose that $\delta_{k_g}$ satisfies $\delta_{k_g}<\sqrt{\dfrac{{\bf g}_{min}}{2\|{\bf g}_{out}\|_2}}$ and inlier noise ${\bf w}\sim \mathcal{N}({\bf 0}_n,\sigma^2{\bf I}^n)$.  Then RRT-GARD identifies the  outlier support $\mathcal{S}_{g}$ with probability at least $(1-\alpha-1/n)$  if $\epsilon^{\sigma}<\min(\epsilon_{GARD},\epsilon_{RRT})$. Here $\epsilon_{RRT}=({{\bf g}_{min}-\delta_{k_g}^2\|{\bf g}_{out}\|_2})/({\frac{1}{\Gamma_{RRT}^{\alpha}(k_g)}+1+\sqrt{\frac{3}{2}}})$.  
\end{theorem}
\begin{proof}
Please see Appendix C.
\end{proof}

 \begin{figure*}[htb]
\begin{multicols}{3}

    \includegraphics[width=\linewidth]{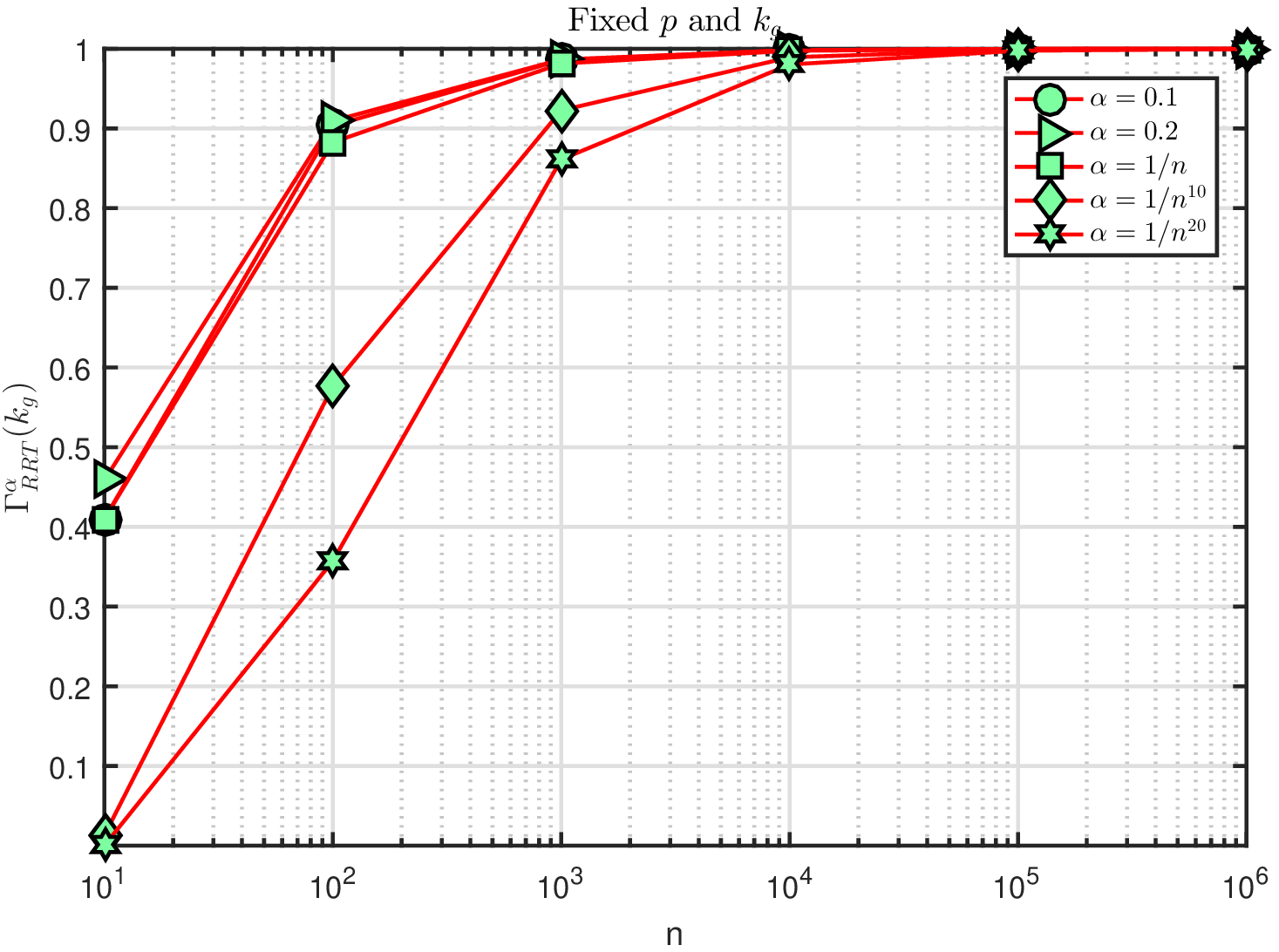} 
    \caption*{a). $k_g=2$, $p=2$.   }
    
    \includegraphics[width=\linewidth]{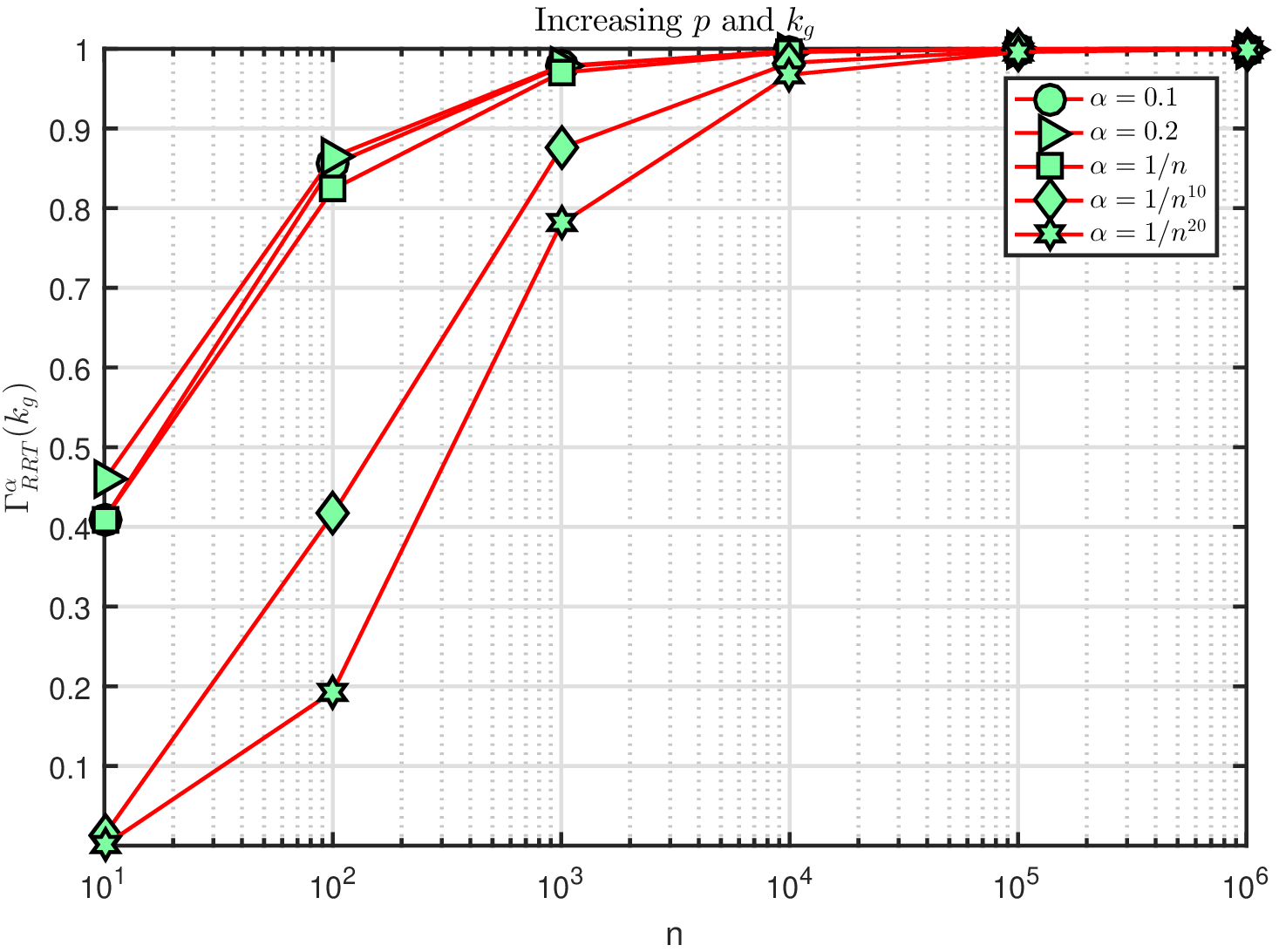}
    \caption*{b). $k_g=0.2n$, $p=0.2n$.   }
    \includegraphics[width=\linewidth]{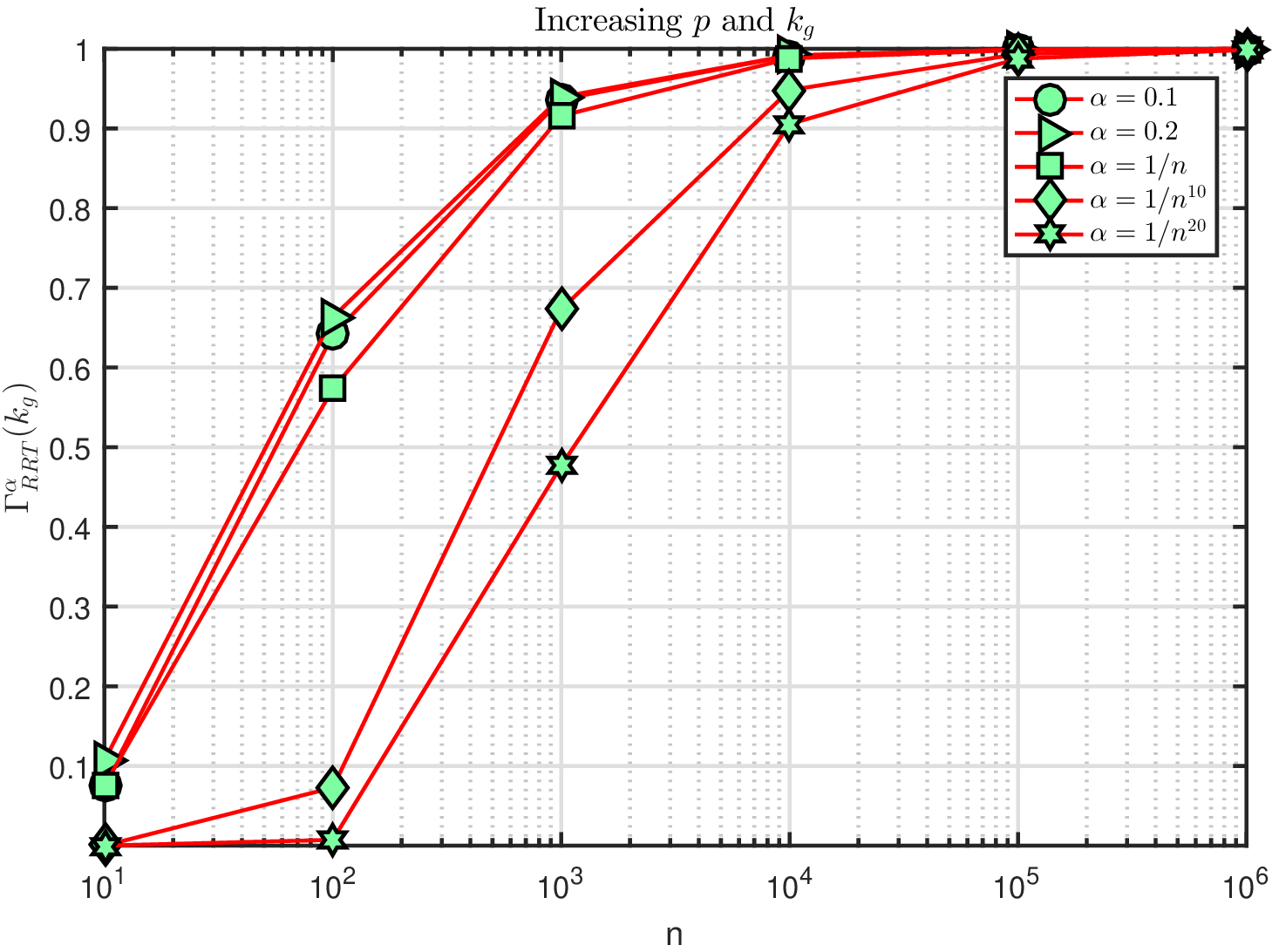}
    \caption*{c). $k_g=0.4n$, $p=0.4n$.   }
   \end{multicols} 
   \caption{Verifying Theorem \ref{thm:asymptotic}.  All choices  of $\alpha$ in a), b) and c) satisfy $\alpha_{lim}=0.$ a) has $d_{lim}=0$, b) has $d_{lim}=0.4$ and c) has $d_{lim}=0.8$.}
   \label{fig:asymptotic}
\end{figure*}
The performance guarantees for RRT-GARD in Theorem \ref{thm:rrt-gard} and GARD($\sigma^2$) in Corollary 1 can be compared in terms of three properties, \textit{viz.} matrix conditions, success probability and  outlier to inlier norm ratio (OINR) which is defined as the minimum value of ${\bf g}_{min}/\epsilon^{\sigma}$ required for the successful outlier detection. Smaller the value of OINR, the more capable an algorithm is in terms of outlier support recovery. 
Theorem \ref{thm:rrt-gard} implies that RRT-GARD can identify all the outliers under the same conditions on the design matrix ${\bf X}$ required by  GARD($\sigma^2$). The success probability of RRT-GARD is smaller than GARD($\sigma^2$) by a factor $\alpha$.  Further, the OINR  of  GARD($\sigma^2$) given by $OINR_{GARD}={\bf g}_{min}/\epsilon_{GARD}$ is smaller than the OINR of RRT-GARD given by $OINR_{RRT}={\bf g}_{min}/\min(\epsilon_{RRT},\epsilon_{GARD})$, i.e.., GARD($\sigma^2$) can  correctly identify outliers of smaller magnitude than RRT-GARD.  Reiterating,  RRT-GARD  unlike GARD($\sigma^2$) is oblivious to $\sigma^2$ and this slight performance loss is the price paid for not knowing $\sigma^2$ \textit{a priori}.  Note that $\epsilon_{RRT}$ can be bounded as follows. 
\begin{equation}
\epsilon_{RRT}\geq \dfrac{{\bf g}_{min}-2\delta_{k_g}^2\|{\bf g}\|_2}{\dfrac{1}{\Gamma_{RRT}^{\alpha}(k_g)}+1+\sqrt{\dfrac{3}{2}}}=\dfrac{(4+2\sqrt{6})\epsilon_{GARD}}{2+\sqrt{6}+\dfrac{2}{\Gamma_{RRT}^{\alpha}(k_g)}}
\end{equation}  
Hence the extra OINR required by RRT-GARD quantified by $OINR_{extra}={OINR_{RRT}}/{OINR_{GARD}}$ satisfies
 \begin{equation}\label{OINR_extra}
  1\leq OINR_{extra}\leq \max\left(1,\frac{2+\sqrt{6}+\frac{2}{\Gamma_{RRT}^{\alpha}(k_g)}}{4+2\sqrt{6}}\right).
 \end{equation}
 By Lemma \ref{lemma:properties},   $\Gamma_{RRT}^{\alpha}(k_g)$  monotonically increases from zero to one as $\alpha$ increases from $0$ to $k_{max}(n-k_g+1)$.  Hence,  $OINR_{extra}$ in (\ref{OINR_extra}) monotonically decreases from infinity for $\Gamma_{RRT}^{\alpha}(k_g)=0$ (i.e., $\alpha=0$) to one for $\Gamma_{RRT}^{\alpha}(k_g)=1$ (i.e., $\alpha=k_{max}(n-k_g+1)$).   Hence, a value of $\Gamma_{RRT}^{\alpha}(k_g)$ close to one  is favourable in terms of $OINR_{extra}$. This requires setting the value of $\alpha$ to a high value which will reduce the probability of outlier support recovery given by $1-\alpha-1/n$.  However, when the sample size $n$ increases to $\infty$, it is possible to achieve both $\alpha\rightarrow 0$ and $\Gamma_{RRT}^{\alpha}(k_g)\rightarrow 1$ simultaneously. This behaviour of RRT-GARD is discussed next. 
\subsection{Asymptotic behaviour of RRT-GARD} 
 In this section, we discuss the behaviour of RRT-GARD and $OINR_{extra}$ as sample size $n\rightarrow \infty$. The asymptotic behaviour of RRT-GARD depends crucially on the behaviour of $\Gamma_{RRT}^{\alpha}(k_g)$ as $n \rightarrow \infty$ which is discussed in the following theorem.
 \begin{theorem}
\label{thm:asymptotic} Let $d_{lim}=\underset{n \rightarrow \infty}{\lim}\dfrac{p+k_g}{n}$ and $\alpha_{lim}=\underset{n \rightarrow \infty}{\lim}\dfrac{\log(\alpha)}{n}$. $\Gamma_{RRT}^{\alpha}(k_g)=\sqrt{F_{\frac{n-p-k_g}{2},0.5}^{-1}\left(\frac{\alpha}{k_{max}(n-k_g+1)}\right)}$  with $k_{max}=n-p-1$ satisfies the following  limits. \\
a). $\underset{n \rightarrow \infty}{\lim}\Gamma^{\alpha}_{RRT}(k_g)=1$ if  $0\leq d_{lim}<1$ and $\alpha_{lim}=0$. \\
b). $0<\underset{n \rightarrow \infty}{\lim}\Gamma^{\alpha}_{RRT}(k_g)=e^{\frac{\alpha_{lim}}{(1-d_{lim})}}<1$  if  $0\leq d_{lim}<1$ and $-\infty<\alpha_{lim}<0$. \\
c). $\underset{n \rightarrow \infty}{\lim}\Gamma^{\alpha}_{RRT}(k_g)=0$ if  $0\leq d_{lim}<1$ and $\alpha_{lim}=-\infty$.
\end{theorem}
\begin{proof}
Please see Appendix D.
\end{proof}
Please note that the maximum number of outliers any algorithm can tolerate is $n-p$, i.e., $k_g$ should satisfy $\frac{p+k_g}{n}<1$ for all $n$. Hence, the condition $0\leq d_{lim}<1$  will be trivially met in all practical scenarios. Theorem \ref{thm:asymptotic} implies that when $\alpha$ is a constant or a function of $n$ that decreases  to zero with increasing $n$  at a rate slower than $a^{-n}$ for some $a>1$, (i.e., $\underset{n \rightarrow \infty}{\lim}\log(\alpha)/n=0$), then it is possible to achieve a value of $\Gamma_{RRT}^{\alpha}(k_g)$ arbitrarily close  to one as $n\rightarrow \infty$.  Choices of $\alpha$ that satisfy $\underset{n \rightarrow \infty}{\lim}\log(\alpha)/n=0$ other than  $\alpha=\text{constant}$ include $\alpha=1/\log(n)$, $\alpha=1/n^c$ for some $c>0$ etc. However, if one decreases $\alpha$ to zero at a rate $a^{-n}$ for some $a>1$ (i.e., $-\infty<\underset{n \rightarrow \infty}{\lim}\log(\alpha)/n<0$) , then it is impossible to achieve a value of $\Gamma_{RRT}^{\alpha}(k_g)$ closer to one. When $\alpha$ is reduced to zero at a rate faster than $a^{-n}$ for some $a>1$ (say $a^{-n^2}$), then $\Gamma_{RRT}^{\alpha}(k_g)$  converges to zero as $n\rightarrow \infty$.  Theorem \ref{thm:asymptotic} is numerically validated in Fig. \ref{fig:asymptotic} where  it is clear that with increasing $n$, $\Gamma_{RRT}^{\alpha}(k_g)$  converges to one when  $d_{lim}=0$, $d_{lim}=0.4$ and $d_{lim}=0.8$.
Theorem \ref{thm:consistency} presented next is a direct consequence of Theorem \ref{thm:asymptotic}.

\begin{theorem} \label{thm:consistency}Consider a  situation where problem dimensions $(n,p,k_g)$ increase to $\infty$ satisfying the conditions in Lemma \ref{lemma_gard}, i.e., $\delta_{k_g}<\sqrt{\dfrac{{\bf g}_{min}}{2\|{\bf g}_{out}\|_2}}$ and $\epsilon^{\sigma}\leq \epsilon_{GARD}$. Then the following statements are true. \\
1). GARD($\sigma^2$) correctly identifies the outlier support as $n\rightarrow \infty$, i.e., $\underset{n\rightarrow \infty}{\lim}\mathbb{P}({\mathcal{S}}_{GARD}=\mathcal{S}_g)=1$. \\
2). RRT-GARD with  $\alpha$ satisfying $\alpha\rightarrow 0$ and $\log(\alpha)/n\rightarrow 0$ as $n\rightarrow \infty$ also correctly identifies the outlier support as $n\rightarrow \infty$, i.e., $\underset{n\rightarrow \infty}{\lim}\mathbb{P}({\mathcal{S}}_{RRT}=\mathcal{S}_g)=1$. 
\end{theorem}
\begin{proof}  Statement 1) follows from  Corollary  1 which states that   $\mathbb{P}({\mathcal{S}}_{GARD}=\mathcal{S}_g)\geq 1-1/n$ when $\delta_{k_g}<\sqrt{\dfrac{{\bf g}_{min}}{2\|{\bf g}_{out}\|_2}}$ and $\epsilon^{\sigma}\leq \epsilon_{GARD}$. By Theorem \ref{thm:asymptotic},  $\log(\alpha)/n\rightarrow 0$ implies that $\Gamma_{RRT}^{\alpha}(k_g)\rightarrow 1$ as $n\rightarrow \infty$ which in turn implies that  $OINR_{extra}\rightarrow  1$ and $\min(\epsilon_{RRT},\epsilon_{GARD})\rightarrow \epsilon_{GARD}$.  This along with $\alpha\rightarrow 0$ as  $n\rightarrow \infty$ implies that the probability bound $\mathbb{P}({\mathcal{S}}_{RRT}=\mathcal{S}_g)\geq 1-1/n-\alpha$ in Theorem \ref{thm:rrt-gard} converges to one.  This proves  statement 2).
\end{proof}
\begin{cor} From the proof of Theorem \ref{thm:consistency}, one can see that  as $n\rightarrow \infty$, the success probability of RRT-GARD given by $\mathbb{P}({\mathcal{S}}_{RRT}=\mathcal{S}_g)\geq 1-1/n-\alpha$ approximately equals the success probability of $GARD(\sigma^2)$ given by $\mathbb{P}({\mathcal{S}}_{GARD}=\mathcal{S}_g)\geq 1-1/n$. Further, the OINR of GARD($\sigma^2$) and RRT-GARD are also approximately same, i.e,  $OINR_{extra}\approx 1$. Hence, both GARD($\sigma^2$)  and RRT-GARD behave similarly in terms of outlier support recovery as $n\rightarrow \infty$, i.e., they are asymptotically equivalent.     
\end{cor}
\begin{cor} Theorem \ref{thm:consistency} implies that all choices of $\alpha$ satisfying $\underset{n\rightarrow \infty}{\lim}\alpha=0$ and   $\underset{n\rightarrow \infty}{\lim}\log(\alpha)/n=0$ deliver $\mathbb{P}({\mathcal{S}}_{RRT}=\mathcal{S}_g)\approx 1$ as $n\rightarrow \infty$.  These constraints are satisfied by a very wide range of adaptations like $\alpha=1/n$ and $\alpha=1/n^{10}$.  Hence,    RRT-GARD is asymptotically tuning free as long as $\alpha$ belongs to  this very broad class of functions.
\end{cor}
\begin{cor}{  Please note that Theorem \ref{thm:consistency} does not implies that RRT-GARD can recover the outlier support with probability  tending towards one asymptotically for all sampling regimes satisfying $d_{lim}<1$.  This is because of the fact that GARD itself can recover outlier support with such accuracy only  when the sampling regime  satisfies  the regularity conditions in Lemma \ref{lemma_gard}. However, the $(n,p,k_g)$ regime where  these  regularity conditions are satisfied is not explicitly charecterized  in open literature to the best of our knowledge. Since no algorithm can correct outliers when $k_g>n-p$, this  not yet charecterized sampling regime where the regularity conditions in Lemma \ref{lemma_gard} is satisfied should also satisfy $d_{lim}<1$.   Hence, Theorem \ref{thm:consistency} states in all sampling regimes where GARD can deliver asymptotically correct outlier support recovery, RRT-GARD can also  deliver the same. }
\end{cor}

 \begin{remark} Once the true outlier support $\mathcal{S}_g$ is known, then the $l_2$ error   in the joint estimate  $[{\hat{\boldsymbol{\beta} }}^T  \hat{\bf g}_{\mathcal{S}_g}^T ]^T= {{\bf A}_{g}}^{\dagger}{\bf y}$ satisfies  $\|\boldsymbol{\beta}-\hat{\boldsymbol{\beta}}\|_2\leq \frac{\|{\bf w}\|_2}{\sigma_{min}({\bf X})\sqrt{1-\delta_{k_g}}}$.  Here ${\bf A}_g=[{\bf X},\ {\bf I}^n_{\mathcal{S}_{\bf g}}]$.  Note that the LS estimate in the absence of outlier satisfies $\|\boldsymbol{\beta}-\hat{\boldsymbol{\beta}}\|_2\leq \frac{\|{\bf w}\|_2}{\sigma_{min}({\bf X})}$ which is lower than the joint estimate only by a factor of $\sqrt{1-\delta_{k_g}}$. Hence, the outlier support recovery guarantees given in Theorems \ref{thm:rrt-gard} and \ref{thm:consistency}  automatically translate into a near outlier free LS performance\cite{gard}.  
\end{remark}
\subsection{Choice of $\alpha$ in finite sample sizes}
{ In the previous section, we discussed the choice of $\alpha$ when the sample size $n$  is increasing to $\infty$. In this section, we discuss the choice of $\alpha$ when the sample size $n$ is fixed at a finite value (on the order of ten or hundred).  This regime is arguably the most important in practical applications and the asymptotic results developed earlier might not be directly applicable here. In this regime, we propose to fix the value of $\alpha$ to a fixed value $\alpha=0.1$ motivated by extensive numerical simulations (please see Section \rom{6}). In particular, our numerical simulations indicate that the RRT-GARD with $\alpha=0.1$ provides nearly the same MSE performance as an oracle supplied with the value of $\alpha$ which minimizes $\|\hat{\boldsymbol{\beta}}-\boldsymbol{\beta}\|_2$. Theorem \ref{thm:high_SNR} justifies this choice of $\alpha$ mathematically. 
\begin{theorem}\label{thm:high_SNR}
Suppose that the design matrix ${\bf X}$ and outlier ${\bf g}_{out}$ satisfy $\delta_{k_g}<\sqrt{\dfrac{{\bf g}_{min}}{2\|{\bf g}_{out}\|_2}}$. Let $\mathcal{M}$, $\mathcal{F}$ and $\mathcal{E}$ denote the events missed discovery $\mathcal{M}=\{card(\mathcal{S}_g/\mathcal{S}_{RRT})>0\}$,  false discovery $\mathcal{F}=\{card(\mathcal{S}_{RRT}/\mathcal{S}_g)>0\}$  and support recovery error $\mathcal{E}=\{\mathcal{S}_{RRT}\neq \mathcal{S}_g\}$ associated with outlier support estimate $\mathcal{S}_{RRT}$ returned by RRT-GARD respectively. Then $\mathcal{M}$, $\mathcal{F}$ and $\mathcal{E}$ satisfy the  following  as $\sigma^2\rightarrow 0$.\\
1).  $\underset{\sigma^2\rightarrow 0}{\lim}\mathbb{P}(\mathcal{E})=\underset{\sigma^2\rightarrow 0}{\lim}\mathbb{P}(\mathcal{F})\leq \alpha$. \\
2). $\underset{\sigma^2\rightarrow 0}{\lim}\mathbb{P}(\mathcal{M})=0$. 
\end{theorem}
\begin{proof}
Please see Appendix E. 
\end{proof} 
Theorem \ref{thm:high_SNR} states that when the  matrix and outlier sparisty regimes are favourable for GARD  to effectively identify the outlier support, then the $\alpha$ parameter in RRT-GARD has  an operational intrepretation of being the  upper bound on the probability of outlier support recovery error and  probability of  false discovery when outlier magnitudes are significantly higher than the inlier variance $\sigma^2$.  Further, it is also clear that when $\sigma^2\rightarrow 0$, the support recovery error in  RRT-GARD is entirely due to the loss of efficiency in the joint LS estimate due to the  identification of outlier free observations as outliers. Consequently, RRT-GARD with the choice of $\alpha=0.1$ motivated by numerical simulations also guarantee accurate outlier support identification  with atleast $90\%$ probability when the outliers values are high compared to the inlier values. 
\begin{remark}
Please note that in the implementation of RRT-GARD recommended by this article, $\alpha$ is fixed to a predefined value $\alpha=0.1$ for finite sample sizes. In other words, $\alpha$ is neither estimated from data nor is chosen using crossvalidation. Consequently, RRT-GARD is not replacing the problem of estimating one unknown parameter (read $\sigma^2$) with another estimation problem (read best value of $\alpha$). This ability to operate GARD with a fixed and data independent hyper parameter and still able to achieve a performance similar to $GARD(\sigma^2)$ is the main advantage   of the residual ratio approach utilized in RRT-GARD.
\end{remark} }
\begin{figure*}
\begin{multicols}{3}

    \includegraphics[width=1\linewidth]{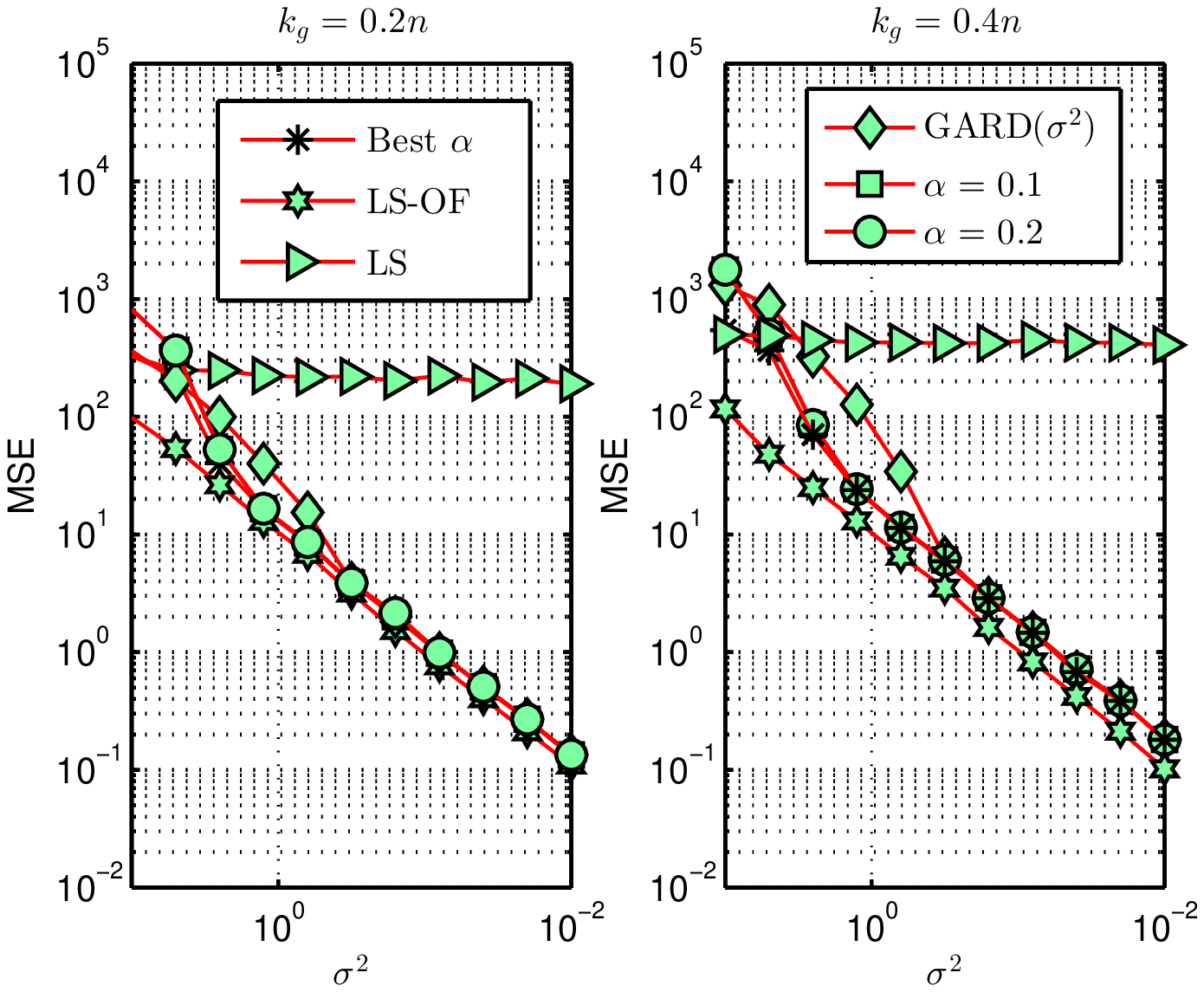} 
    \caption*{a).  Model 1. Varying $\sigma^2$.  }
    
    \includegraphics[width=1\linewidth]{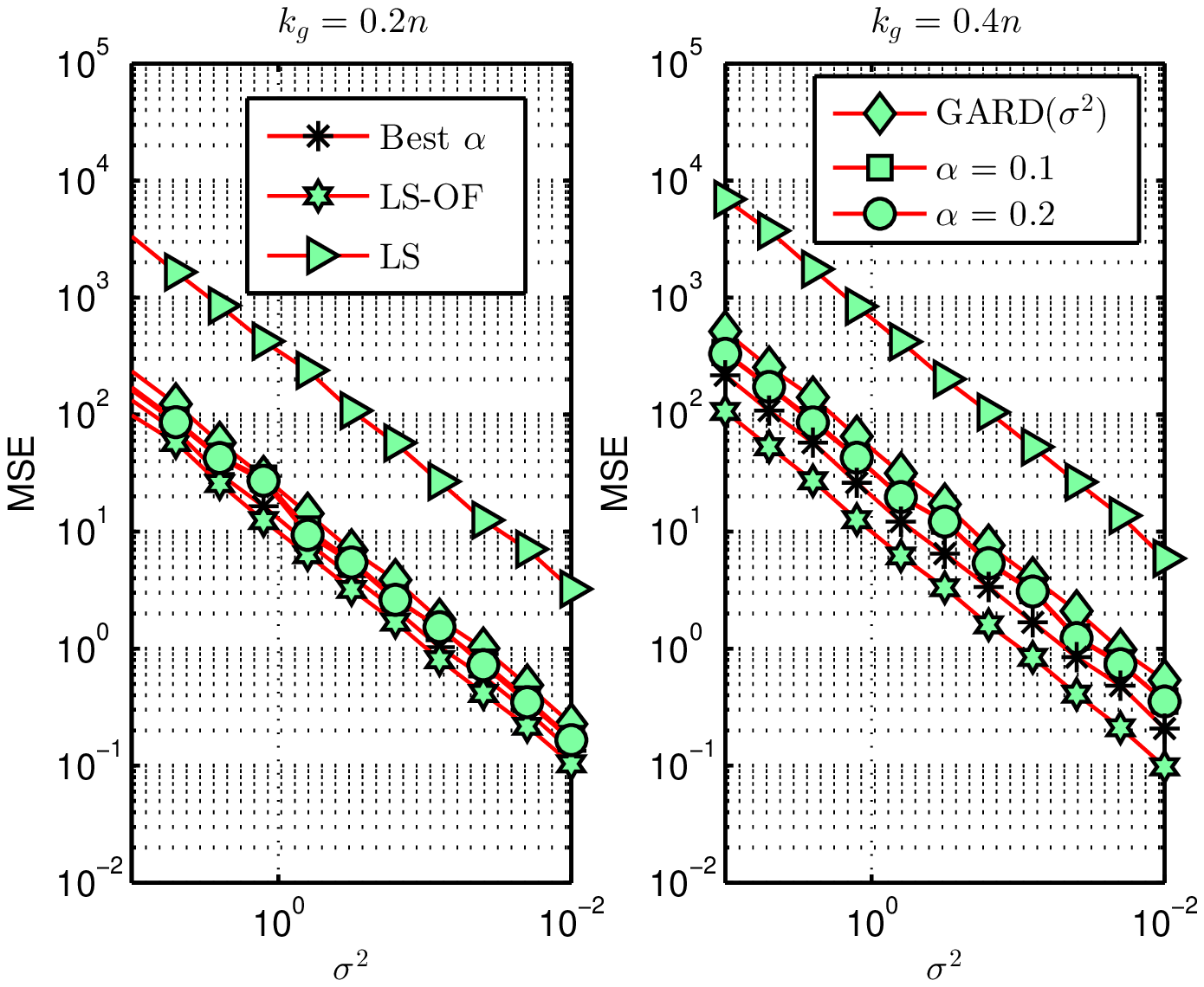} 
    \caption*{b). Model 2. Varying $\sigma^2$. }
    \includegraphics[width=1\linewidth]{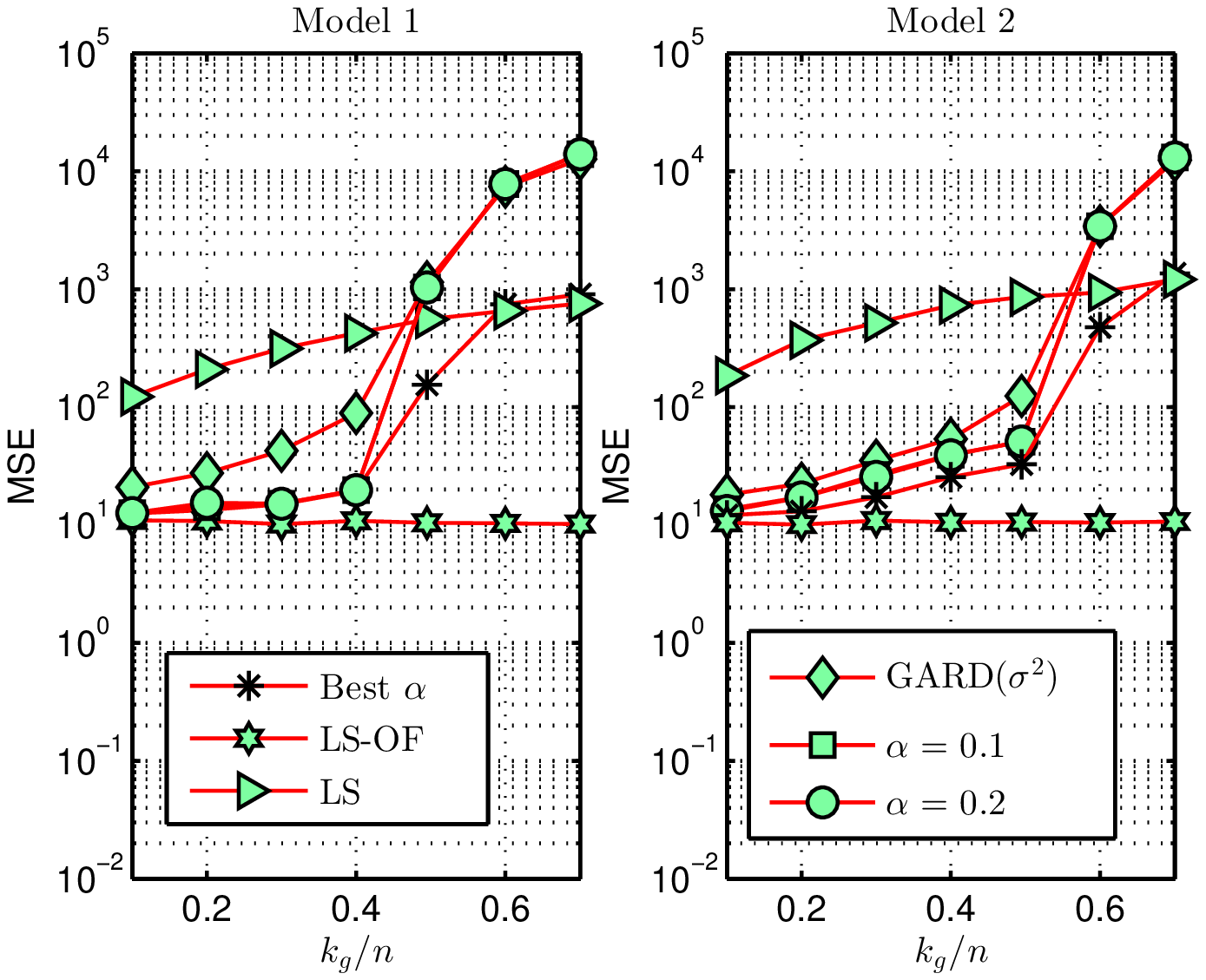}
    \caption*{c). Models 1 and  2. $\sigma^2=1$.  Varying $k_g$.  }
   \end{multicols}
   \caption{Near optimality of $\alpha=0.1$ in RRT-GARD. Number of predictors $p=10$. Legends are distributed among the sub-figures.}
   \label{fig:near_optimality}
\end{figure*}
  
\section{Numerical Simulations}
In this section, we numerically evaluate and compare the  performance  of RRT-GARD  and  popular robust regression techniques in both synthetic and  real life data sets.  
\subsection{Simulation settings for experiments using synthetic data}
 The design matrix ${\bf X}$ is randomly generated according to ${\bf X}_{i,j}\sim \mathcal{N}(0,1)$ and the columns of the resulting matrix are normalised to have unit $l_2$ norm. The number of samples $n$ is fixed at $n=200$.   All entries of $\boldsymbol{\beta}$ are randomly set to $\pm 1$. Inlier noise ${\bf w}$ is distributed ${\bf w} \sim \mathcal{N}({\bf 0}_n,\sigma^2{\bf I}^n)$. Two outlier models are considered.\\
{\bf Model 1:-} ${\bf g}_{out}(j)$ for $j \in \mathcal{S}_g$ are sampled from  $\{10,-10\}$.\\
{\bf Model 2:-} ${\bf g}_{out}(j)$ for $j \in \mathcal{S}_g$ are sampled according to ${\bf g}_{out}(j){\sim}0.5\mathcal{N}(12\sigma, 16\sigma^2)+0.5\mathcal{N}(-12\sigma, 16\sigma^2)$\cite{arosi}.\\
Model 1 have outlier power independent of $\sigma^2$, whereas, Model 2 have outlier power increasing with increasing $\sigma^2$.  Figures \ref{fig:near_optimality}- \ref{fig:noise_estimation} are presented after performing $10^2$  Monte Carlo iterations. In each iteration ${\bf X}$, $\boldsymbol{\beta}$, ${\bf g}_{out}$ and ${\bf w}$ are independently generated. MSE in figures \ref{fig:near_optimality}- \ref{fig:noise_estimation}  represent the averaged value of $\|\boldsymbol{\beta}-\hat{\boldsymbol{\beta}}\|_2^2$. ``LS-OF", ``LS" and ``$\alpha$" in figures \ref{fig:near_optimality}- \ref{fig:noise_estimation} represent the LS performance in outlier free data,  LS performance with outliers and RRT-GARD with parameter  $\alpha$.
\subsection{Choice of $\alpha$ in finite sample sizes}
 Theorem \ref{thm:asymptotic} implies that RRT-GARD is asymptotically tuning free. However, in finite sample sizes, the choice of $\alpha$ will have a significant impact on the performance of RRT-GARD. In this section, we compare the performance of RRT-GARD with $\alpha=0.1$ and $\alpha=0.2$  with that of an oracle aided estimator which compute  RRT-GARD estimate over 100  different values of $\alpha$ between $10$ and $10^{-6}$  and choose the RRT-GARD estimate with lowest $l_2$-error $\|\boldsymbol{\beta}-\hat{\boldsymbol{\beta}}\|_2^2$ (Best $\alpha$). This estimator requires \textit{a priori} knowledge of $\boldsymbol{\beta}$ and  is not practically implementable. However, this estimator gives the best possible performance achievable by RRT-GARD. From the six experiments  presented in Fig. \ref{fig:near_optimality}, it is clear that the performance of RRT-GARD with  $\alpha=0.1$ and $\alpha=0.2$ are only slightly inferior compared to the performance of ``$\text{Best}\ \alpha$" in all situations where ``$\text{Best}\ \alpha$" reports near LS-OF performance.  Also RRT-GARD with  $\alpha=0.1$ and $\alpha=0.2$ perform atleast as good as GARD($\sigma^2$).  This trend was visible in many other experiments not reported here. Also please note that in view of Theorem \ref{thm:high_SNR}, $\alpha=0.1$ gives better outlier support recovery guarantees   than $\alpha=0.2$.  Hence, we recommend setting $\alpha$ in RRT-GARD to  $\alpha=0.1$  when  $n$ is finite.

\begin{figure*}
\begin{multicols}{3}

    \includegraphics[width=1\linewidth]{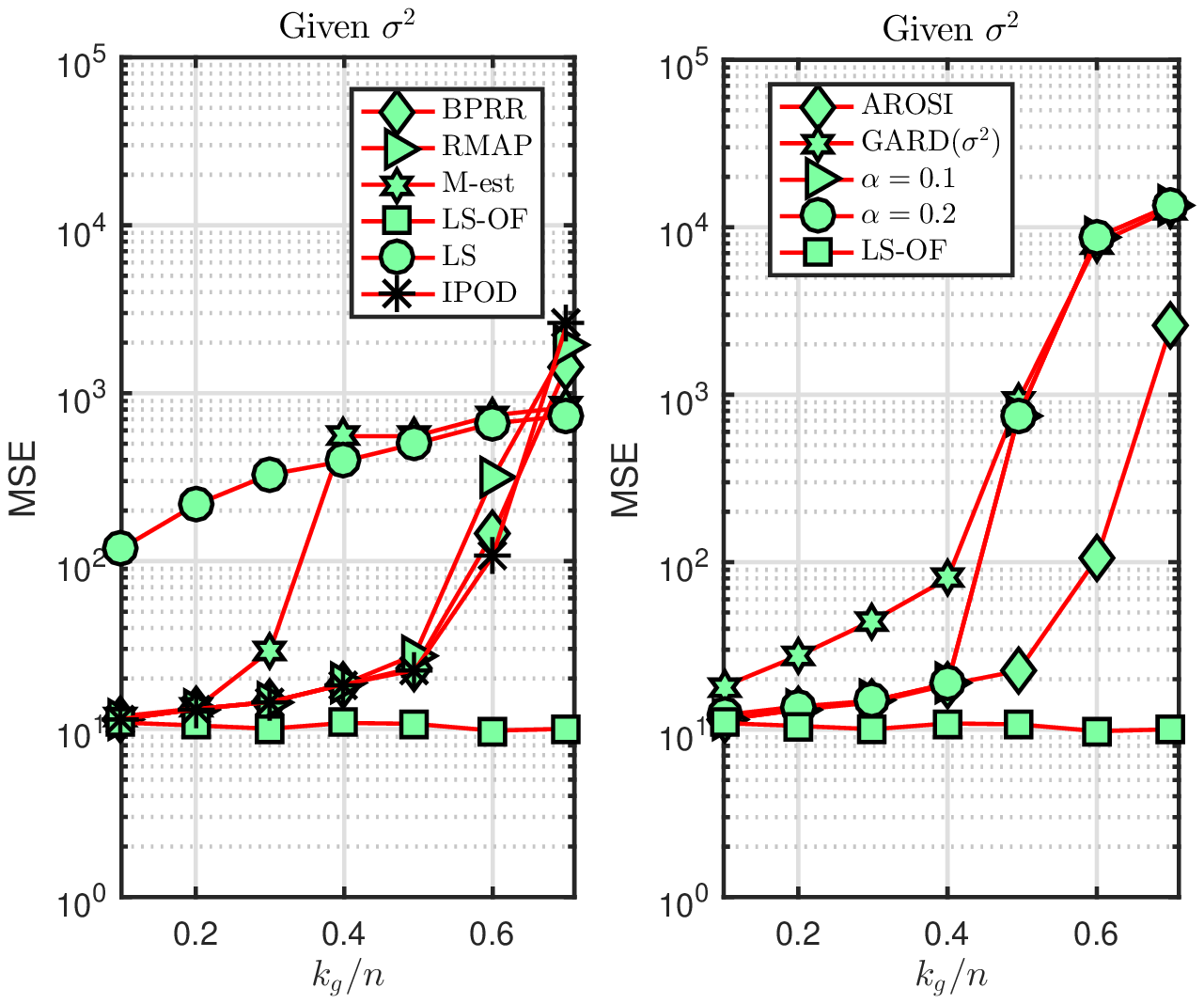} 
    \caption*{a).   Given $\sigma^2$.  }
    
    \includegraphics[width=1\linewidth]{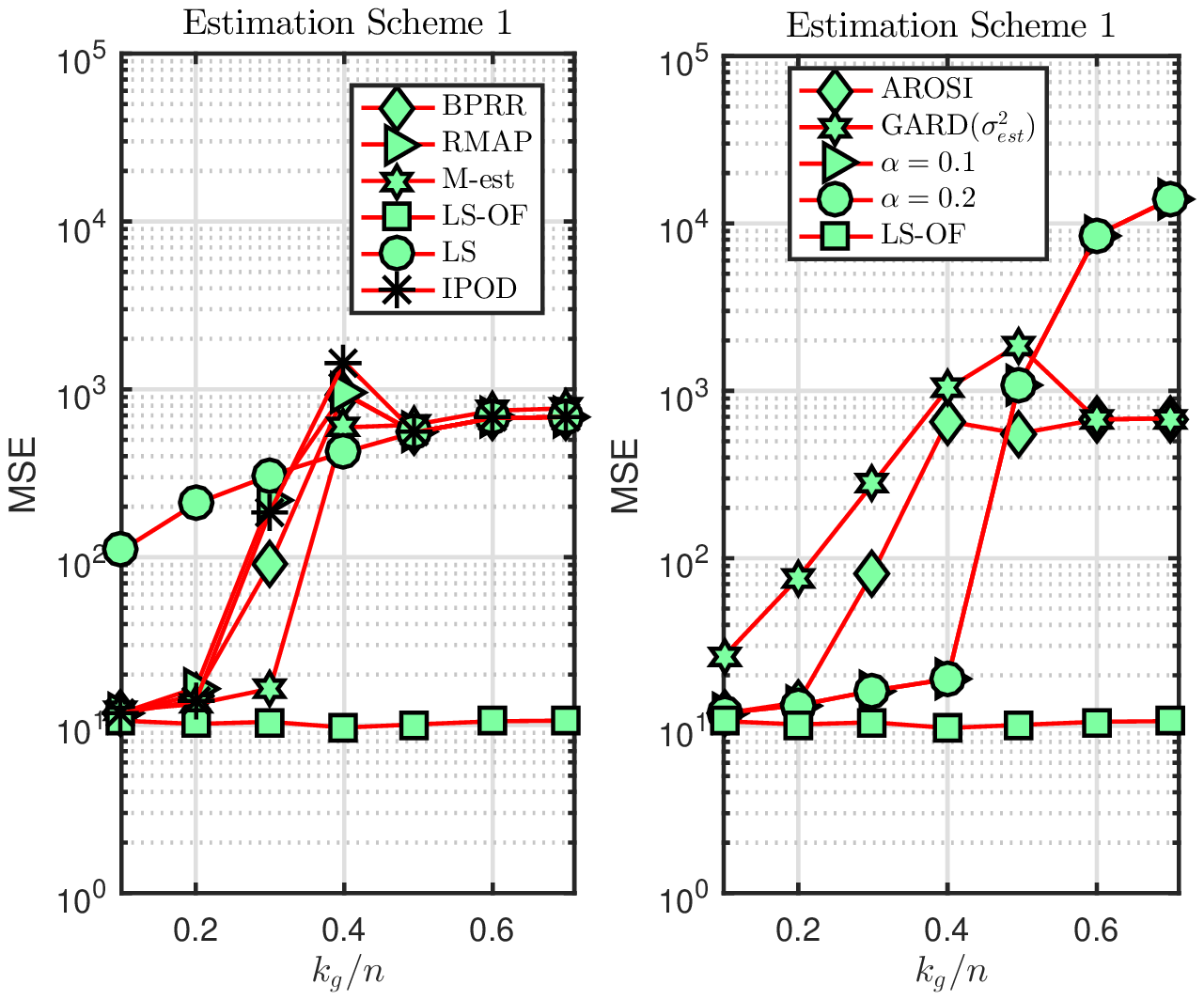} 
    \caption*{b).  $\sigma^2$ estimation scheme 1 (\ref{l1noise}). }
    \includegraphics[width=1\linewidth]{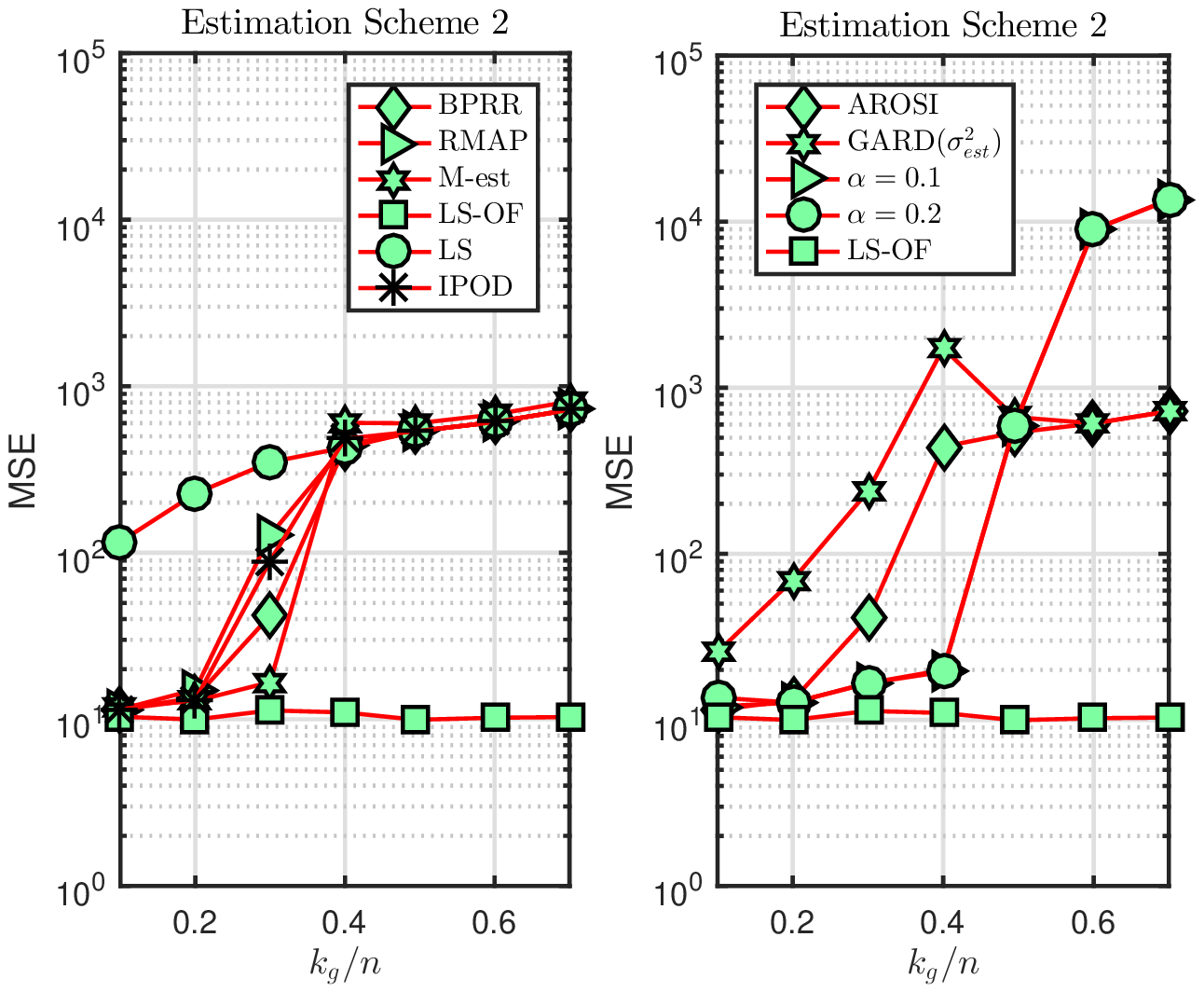}
    \caption*{c).  $\sigma^2$ estimation scheme 2 (\ref{mad}).  }
   \end{multicols}
   \caption{Model 1. Number of predictors $p=10\ll n$  and $\sigma^2=1$.}
   \label{fig:model1}
\end{figure*}

\subsection{Comparison of RRT-GARD with popular algorithms}
The following algorithms are compared with RRT-GARD. ``M-est" represents Hubers' M-estimate  with Bisquare loss function computed using the Matlab function ``robustfit". Other parameters are set according to the default setting in Matlab.  ``BPRR" represents (\ref{BPRR}) with parameter $\lambda_{bprr}=\sqrt{\frac{n-p}{n}}\epsilon^{\sigma}$  \cite{koushik_conf}.  ``RMAP" represents (\ref{RMAP}) with parameter $\lambda_{rmap}=\sigma\sqrt{2\log(n)}/3$ \cite{bdrao_robust}.  ``AROSI" represents (\ref{arosi}) with parameter $\lambda_{arosi}=5\sigma$. IPOD represents the estimation scheme in Algorithm 1 of  \cite{she2011outlier} with hard thresholding penalty and $\lambda$ parameter set to $5\sigma$ as in \cite{arosi}.  As noted in \cite{arosi}, the performances of BPRR, RMAP, AROSI etc. improve tremendously after performing the re-projection step detailed in \cite{arosi}.  For algorithms like RMAP, IPOD, AROSI etc. which directly give a robust estimate $\hat{\boldsymbol{\beta}}$ of $\boldsymbol{\beta}$, the re-projection  step identifies the  outlier support by  thresholding the robust residual ${\bf r}={\bf y}-{\bf X}\hat{\boldsymbol{\beta}}$, i.e., $\hat{\mathcal{S}}_g=\{k:|{\bf r}(k)|>\gamma \sigma\}$. For algorithms like BPRR, BSRR etc. which estimate the outliers  directly, the outlier support is identified by thresholding the outlier estimate $\hat{\bf g}_{out}$, i.e., $\hat{\mathcal{S}}_g=\{k:|\hat{\bf g}_{out}(k)|>\gamma \sigma\}$. Then the nonzero  outliers and regression vector ${\boldsymbol{\beta}}$ are jointly estimated  using  $[{\hat{\boldsymbol{\beta} }}^T  \hat{\bf g}_{out}({\hat{\mathcal{S}_g}})^T ]^T= {[{\bf X},  {\bf I}^n_{\hat{\mathcal{S}_g}}]}^{\dagger}{\bf y}$. The re-projection thresholds are set at $\gamma=3$, $\gamma=3$, $\gamma=3$ and $\gamma=5$ respectively for BPRR, RMAP, IPOD and AROSI. Two schemes to estimate $\sigma^2$ are considered in this article. Scheme 1 implements (\ref{l1noise}) and Scheme 2 implements  (\ref{mad}) using ``M-est" residual respectively. { Since there do not exist any analytical guidelines on  how to set the re-projection thresholds,  we set these parameters such that they maximise the performance  of BPRR, RMAP, IPOD and AROSI when $\sigma^2$ is known. Setting the re-projection thresholds to achieve best performance with estimated $\sigma^2$ would result in different re-projection parameters for different $\sigma^2$ estimation schemes and a highly inflated performance.
 
\begin{figure*}
\begin{multicols}{3}

    \includegraphics[width=1\linewidth]{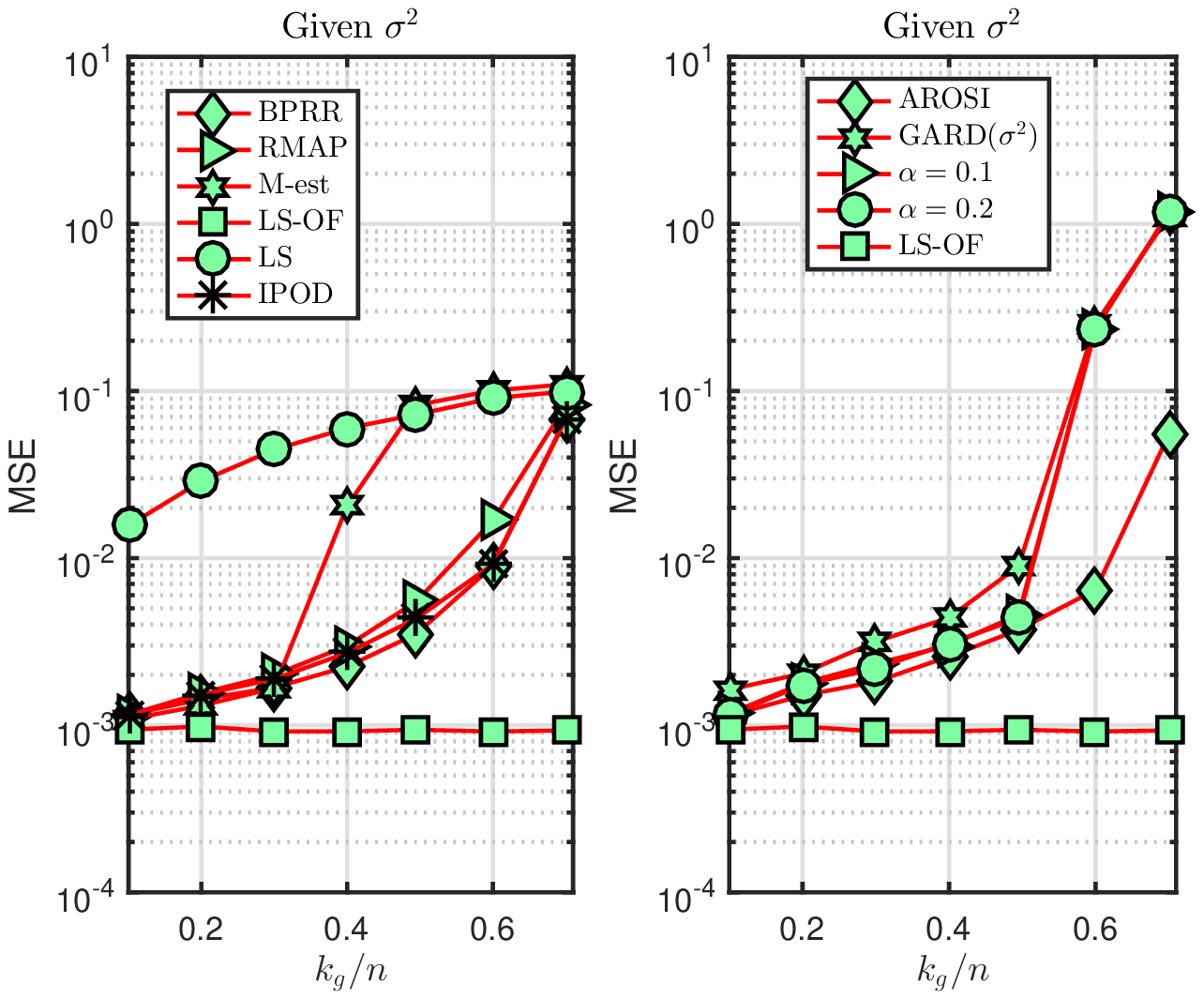} 
    \caption*{a).  Given $\sigma^2$.  }
    
    \includegraphics[width=1\linewidth]{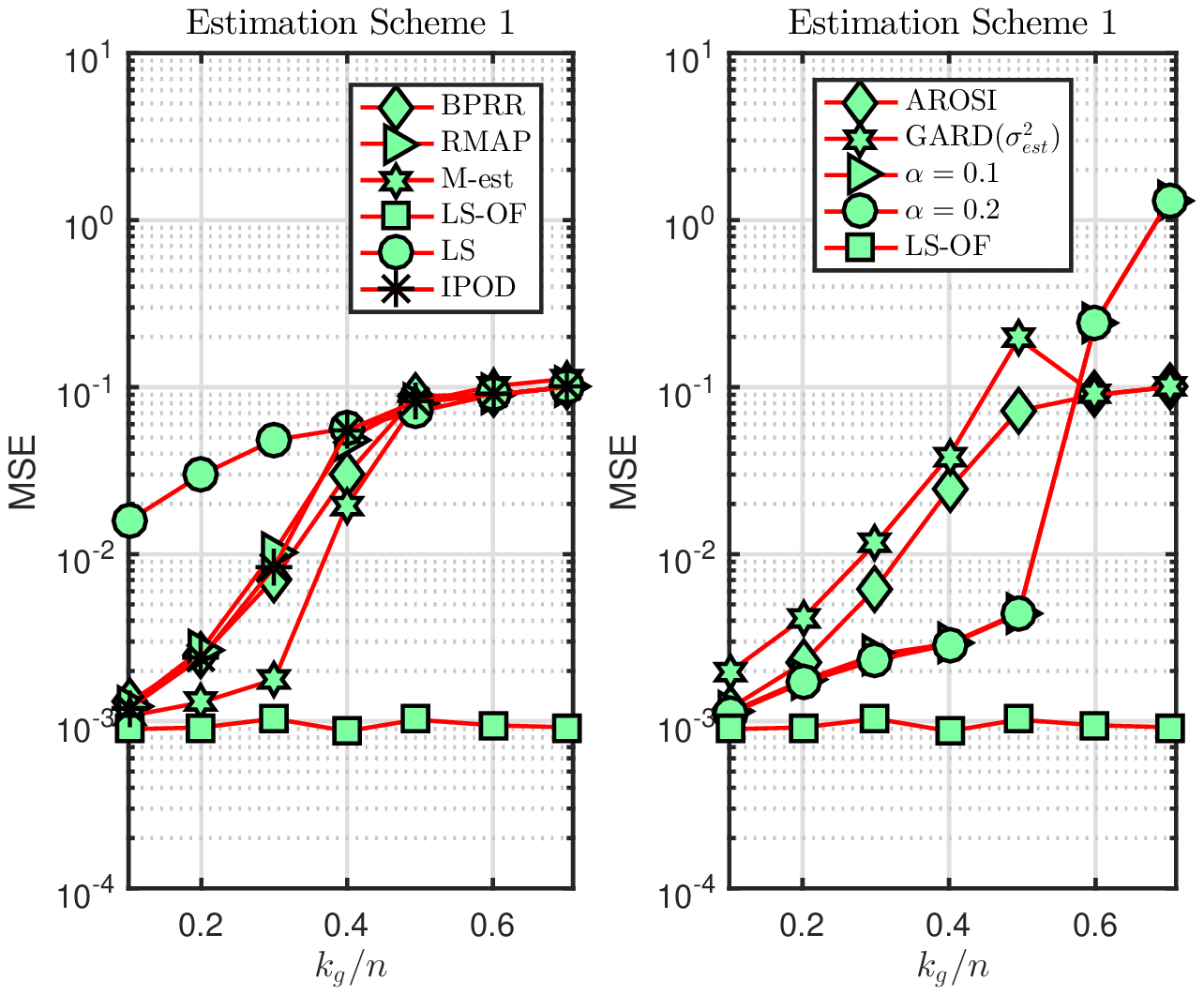} 
    \caption*{b). $\sigma^2$ estimation scheme 1 (\ref{l1noise}). }
    \includegraphics[width=1\linewidth]{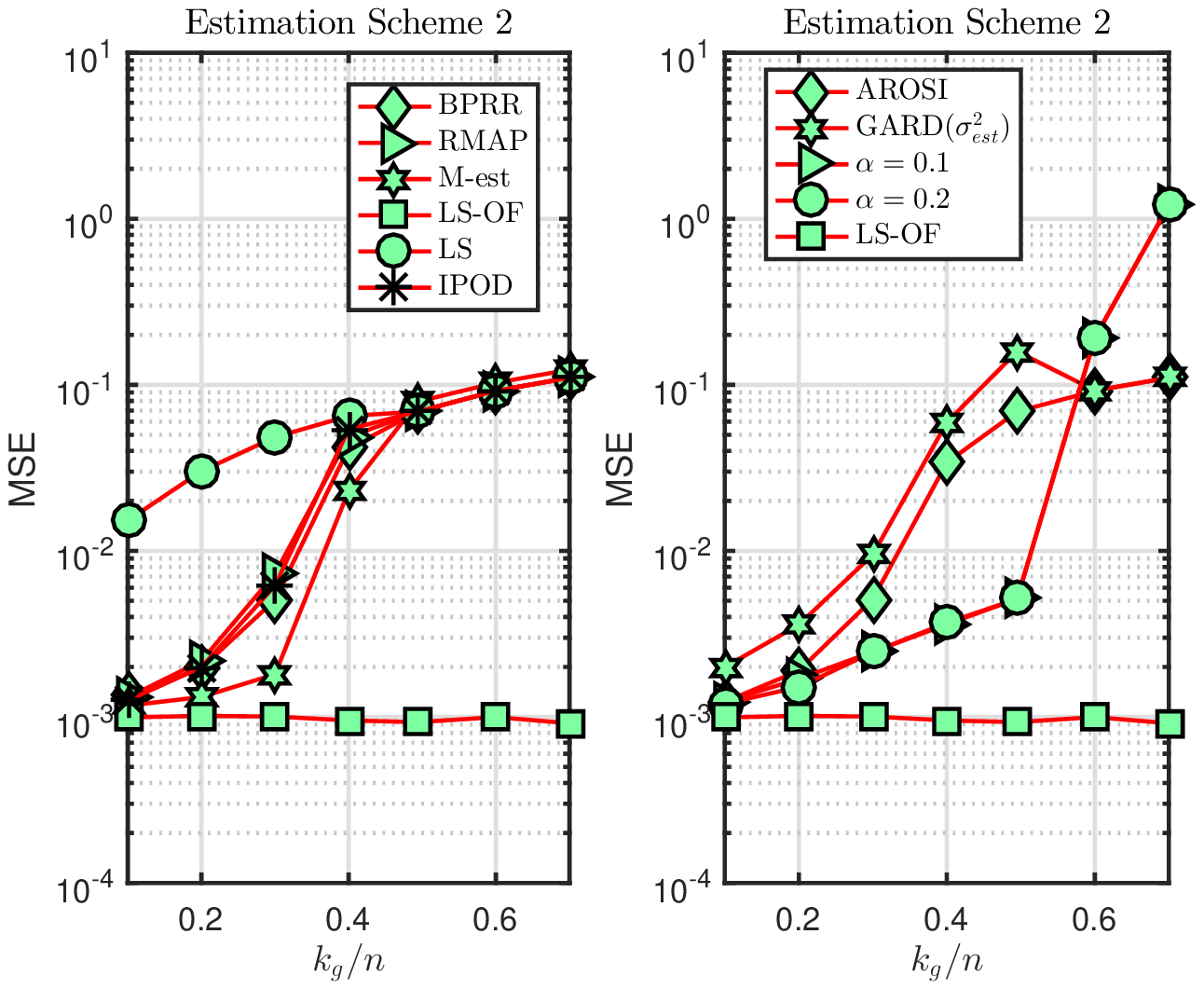}
    \caption*{c). $\sigma^2$ estimation scheme 2 (\ref{mad}).  }
   \end{multicols}
   \caption{Model 2. Number of predictors $p=10\ll n$ and $\sigma=\ median |({\bf X}\boldsymbol{\beta})_j|/16$ \cite{arosi}. }
   \label{fig:model2}
\end{figure*}
 
 \begin{figure*}
\begin{multicols}{3}

    \includegraphics[width=1\linewidth]{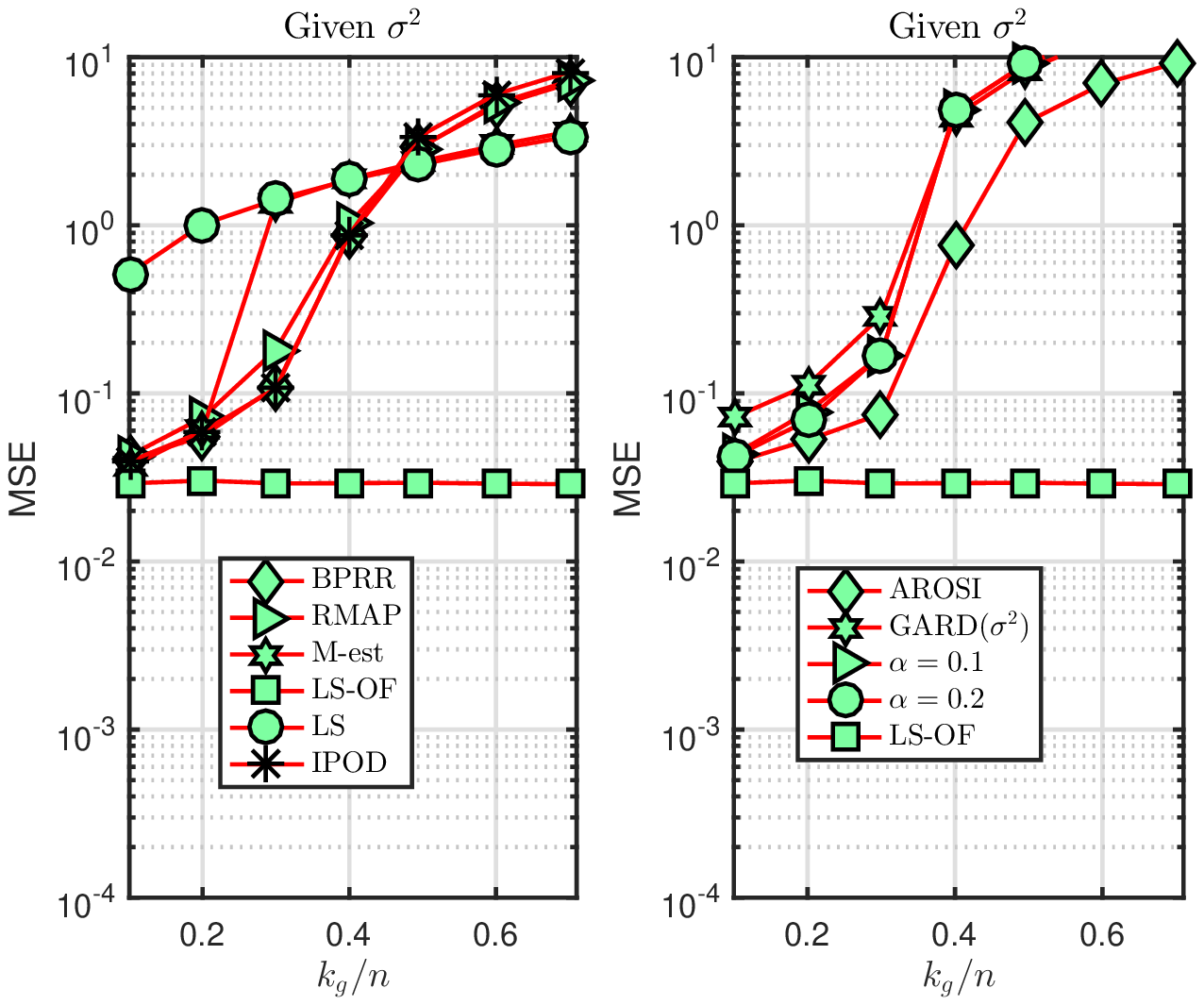} 
    \caption*{a).  Given $\sigma^2$.  }
    
    \includegraphics[width=1\linewidth]{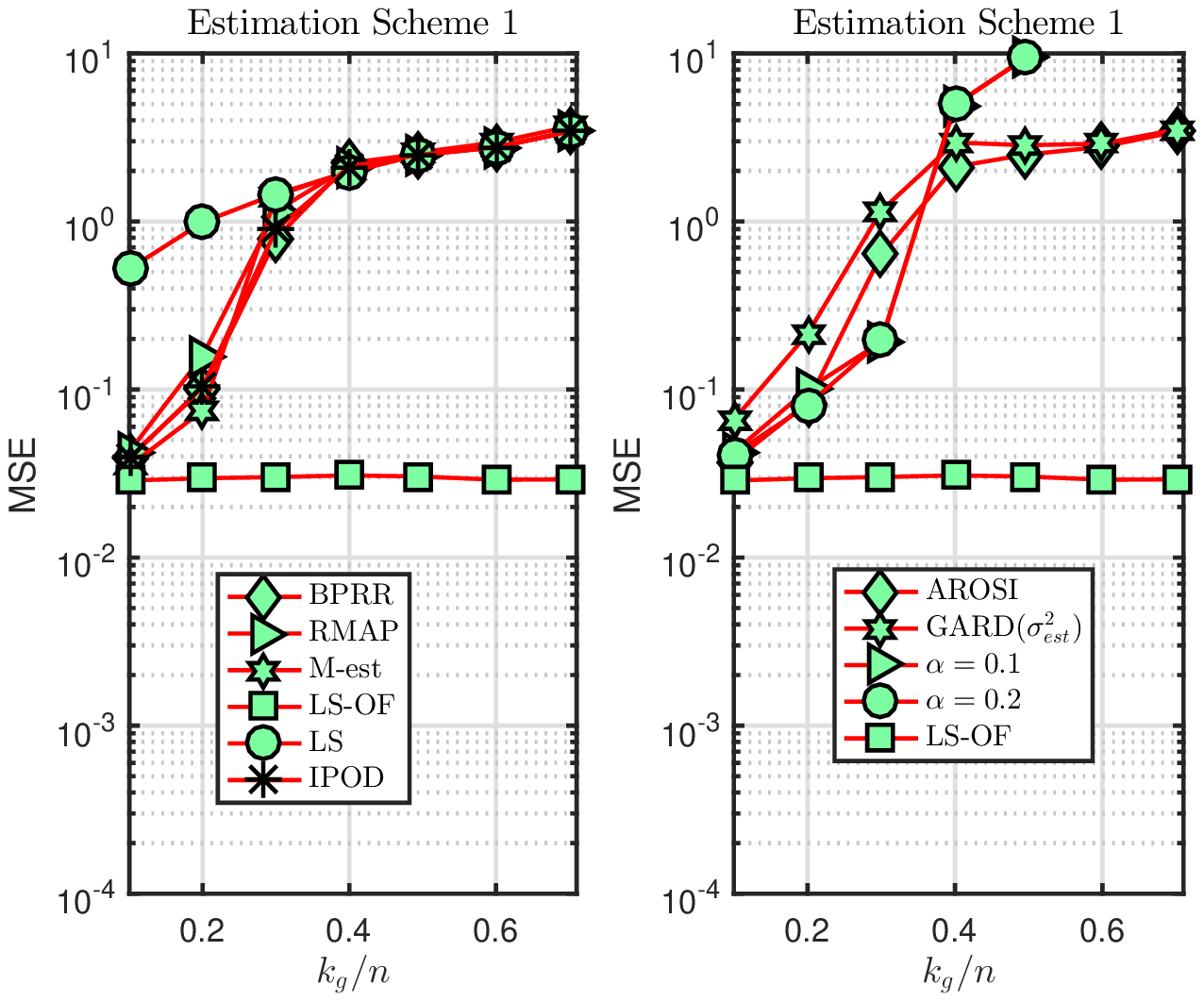} 
    \caption*{b). $\sigma^2$ estimation scheme 1 (\ref{l1noise}). }
    \includegraphics[width=1\linewidth]{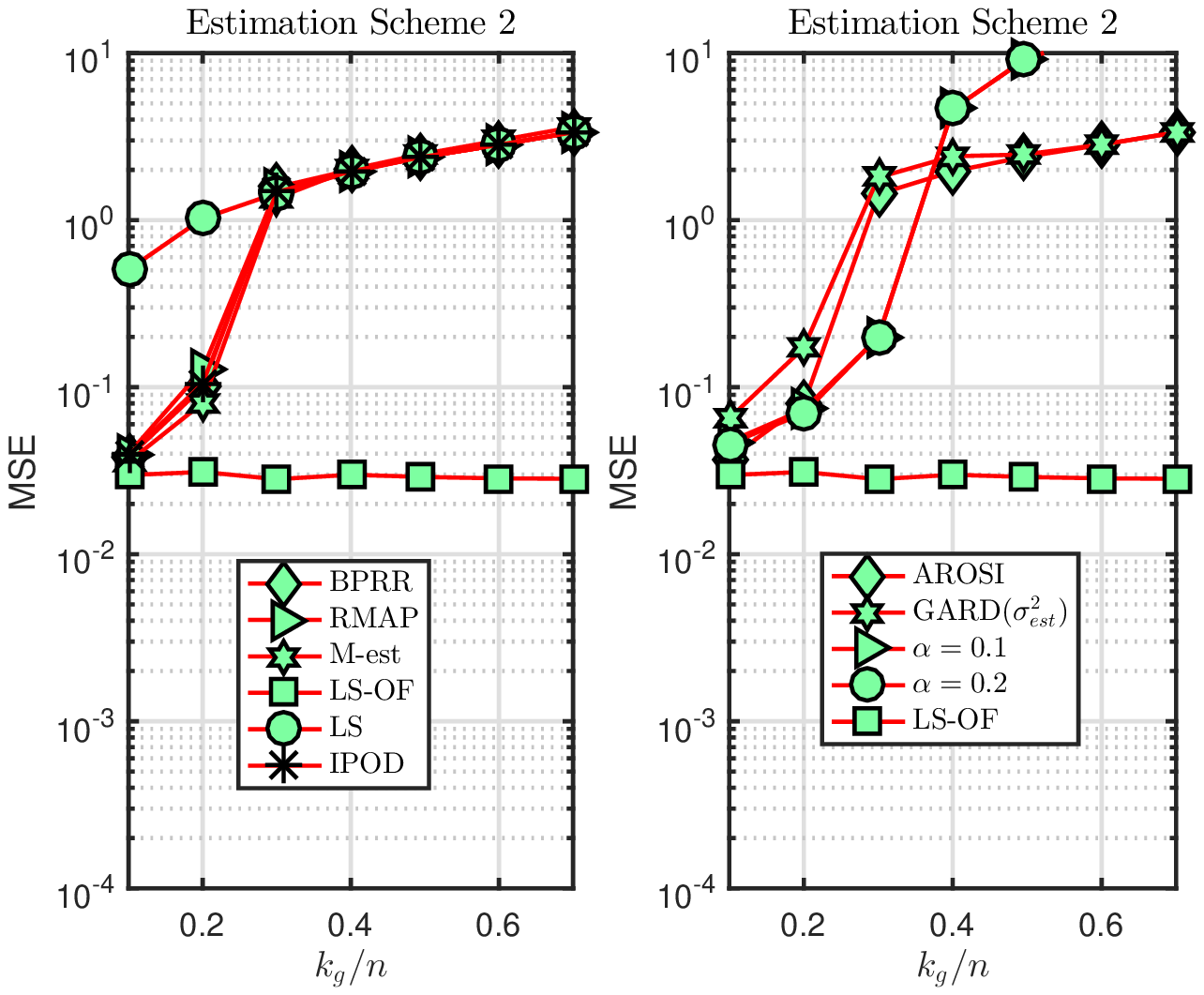}
    \caption*{c). $\sigma^2$ estimation scheme 2 (\ref{mad}).  }
   \end{multicols}
   \caption{Model 2.  Number of predictors increased to $p=50$ and  $\sigma=\ median |({\bf X}\boldsymbol{\beta})_j|/16$ \cite{arosi}. }
   \label{fig:model3}
\end{figure*}
 
{ We first consider the situation where the number of predictors $p=10$ is very small compared to the number of measurements $p$.  As one can see from Fig. \ref{fig:model1} and Fig. \ref{fig:model2}, BPRR, RMAP, IPOD and AROSI perform much better compared to GARD($\sigma^2$) and RRT-GARD when $\sigma^2$ is known. In fact, AROSI outperforms all other algorithms. Similar trends were visible in \cite{arosi}. Further, this good performance of AROSI, BPRR, IPOD and RMAP also validates the choice of tuning parameters used in these algorithms.  However, when the estimated $\sigma^2$ is used to set the parameters, one can see from  Fig. \ref{fig:model1} and Fig. \ref{fig:model2} that the performance of GARD($\sigma^2$), BPRR, RMAP, IPOD and AROSI degrade tremendously. In fact, in all the four experiments conducted with estimated $\sigma^2$, RRT-GARD outperforms M-est,  GARD($\sigma^2$), BPRR, RMAP and AROSI except when $k_g/n$ is very high. However, when $k_g/n$ is very high,  all these algorithms  perform similar to or worse than the LS estimate.
Next we consider the performance of algorithms when the number of predictors $p$ is increased from $p=10$ to $p=50$. Note that the number of outliers that can be identified using any SRIRR algorithm is an increasing function of the "number of free dimensions" $n-p$. Consequently, the BDP of all algorithms in  Fig. 6  are much smaller than the corresponding BDPs in Fig.  \ref{fig:model2}. Here also the performance of AROSI is superior to other algorithms when $\sigma^2$ is known \textit{a priori}. However, when $\sigma^2$ is unknown \textit{a priori}, the performance of RRT-GARD is still superior compared to the other algorithms under consideration.  }
\begin{figure*}
\begin{multicols}{3}

\includegraphics[width=\linewidth]{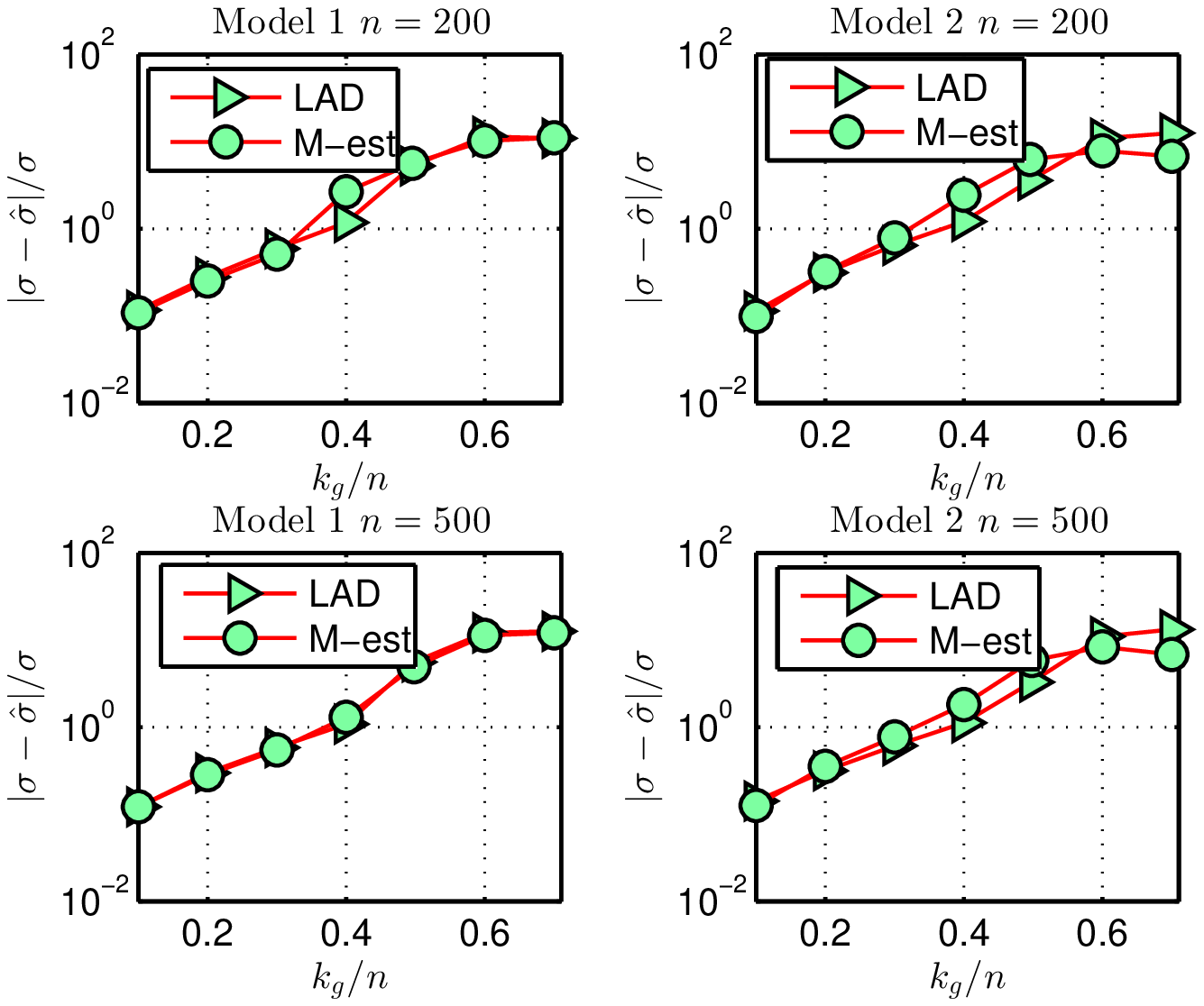}
\caption*{a). Error in the $\sigma^2$ estimate with increasing $k_g/n$.}
\includegraphics[width=\linewidth]{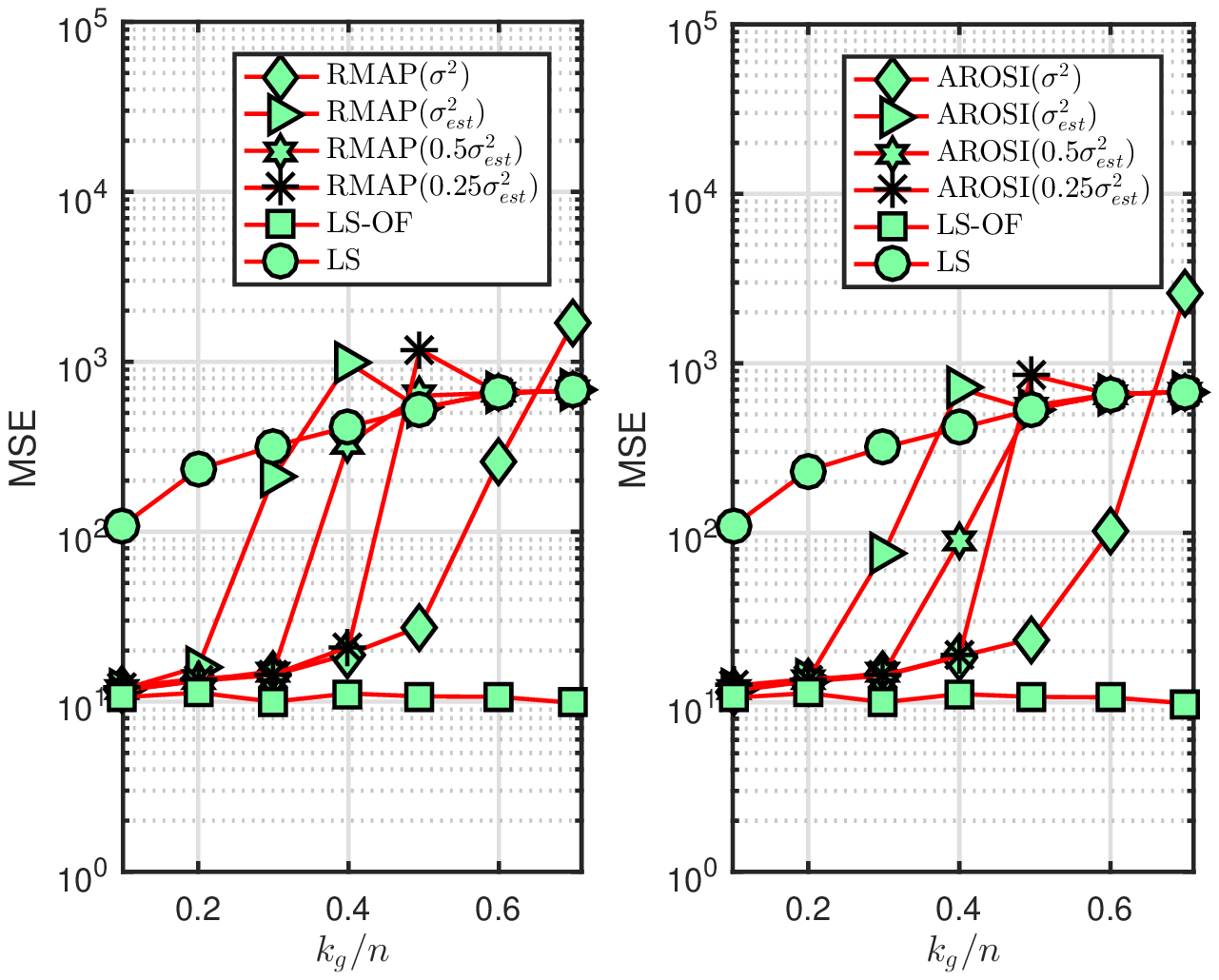}
\caption*{b). Performance of RMAP and AROSI with scaled down $\sigma^2$ estimates}
\includegraphics[width=\linewidth]{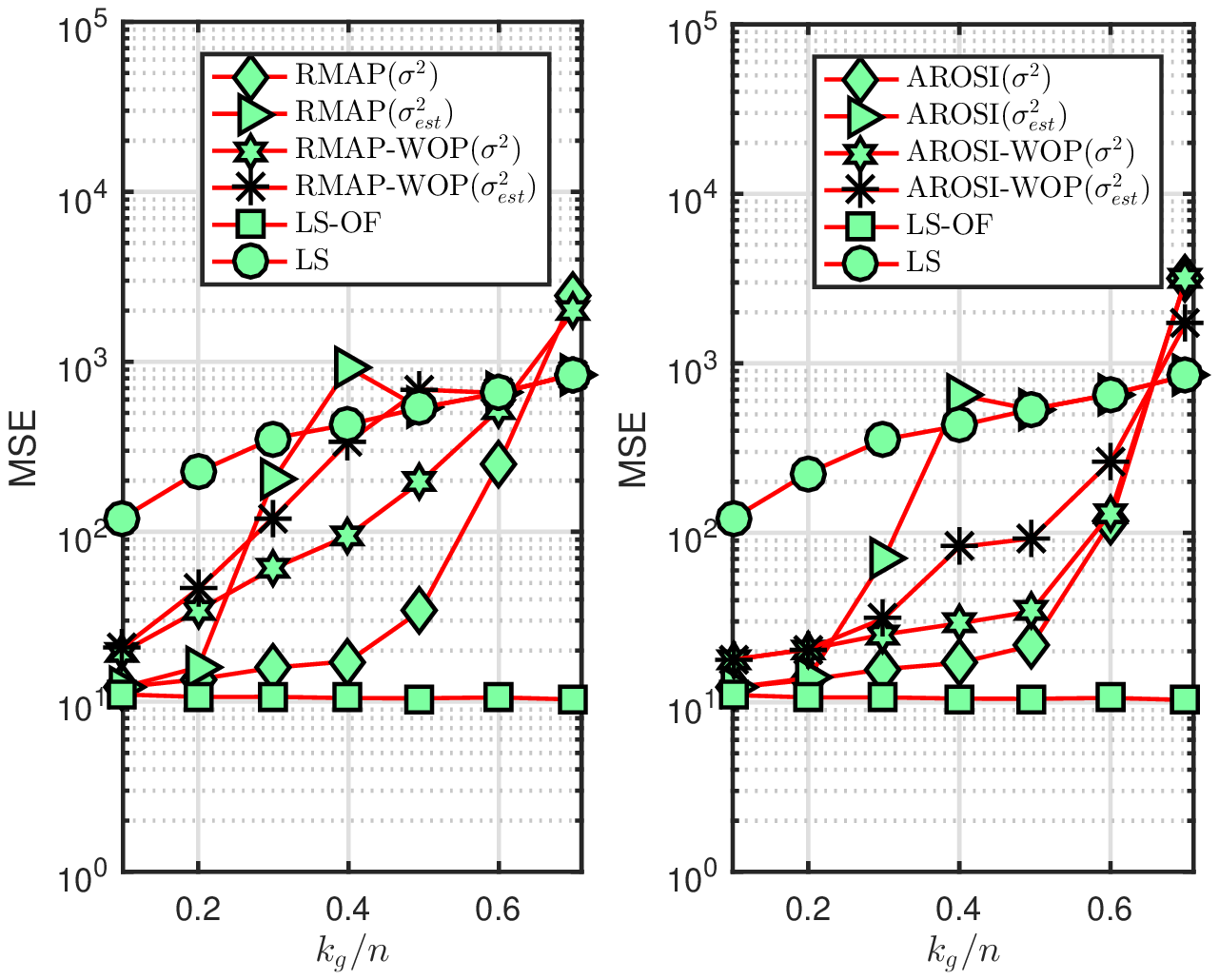}

\caption*{c). Sensitivity of re-projection step with estimated $\sigma$. }

\end{multicols}
\caption{Performance degradation in AROSI, RMAP etc. with estimated $\sigma^2$. For $Alg\in \{RMAP,AROSI \}$,  $Alg(\sigma^2)$ represents the performance when $\sigma^2$ is known \textit{a priori} and $Alg(a\ \sigma^2_{est})$ represents the performance when $\sigma^2$ is estimated using  scheme 1 (\ref{l1noise}) and scaled by a factor $a$. Similarly, Alg-WOP is the performance of Alg without the reprojection step. }
\label{fig:noise_estimation}

\end{figure*}
\subsection{Analysing the performance of RMAP, AROSI etc. with estimated $\sigma^2$}
{ In this section, we consider the individual factors that cumulatively results in the degraded performance of algorithms like AROSI, RMAP etc.  As one can see from Fig. \ref{fig:model1}-Fig. \ref{fig:model3}, the performance of RMAP, AROSI etc. degrade significantly with increasing $k_g$. This is directly in agreement with Fig. \ref{fig:noise_estimation}.a) where it is shown that the error in the noise variance estimate also increases with increasing $k_g/n$ ratio. We have also observed that both the LAD and M-estimation based noise estimates typically overestimate the true  $\sigma$.  Consequently, one can mitigate the effect of error in $\sigma^2$ estimates by scaling these estimates  downwards before using them in RMAP, AROSI etc. The usage of scaled $\sigma^2$ estimates, as  demonstrated in Fig. \ref{fig:noise_estimation}.b) can significantly improve  the performance of RMAP and AROSI. However, the choice of a good  scaling value would be dependent upon the  unknown outlier sparsity regime and the particular noise variance estimation algorithm used.

The noise variance estimate  in AROSI, RMAP etc. are used in two different occassions, \textit{viz}. 1). to set the hyperparameters $\lambda_{arosi}$ and $\lambda_{rmap}$ and 2). to set the reprojection thresholds $\gamma$.   It is important to know which of these  two $\sigma^2$ dependent steps is most sensitive to the error in $\sigma^2$ estimate. From Fig. \ref{fig:noise_estimation}.c), it is clear that the performance of RMAP and AROSI significantly improves after the reprojection step when $\sigma^2$ is known \textit{a priori}. However, the performance of AROSI and RMAP is much better without reprojection when $\sigma^2$ is unknown and $k_g/n$ is higher. It is also important to note that when $k_g/n$ is small, the performance of RMAP and AROSI without reprojection is poorer than the performance with reprojection even when $\sigma^2$  is unknown.  Hence, the choice of whether to have a reprojection step with estimated $\sigma^2$ is itself dependent on the outlier sparsity regime. 
Both these analyses point out to the fact that it is difficult to improve  the performance of AROSI, RMAP etc. with estimated $\sigma^2$ uniformly over all outlier sparsity regimes by tweaking the various hyper parameters involved.  Please note that the performance of AROSI, RMAP etc. without reprojection or scaled down $\sigma^2$ estimate is still poorer than that  of RRT-GARD. }    

\subsection{Outlier detection in real data sets}
\begin{figure}
\begin{multicols}{2}
\caption*{a). Stack loss.}
    \includegraphics[width=\linewidth]{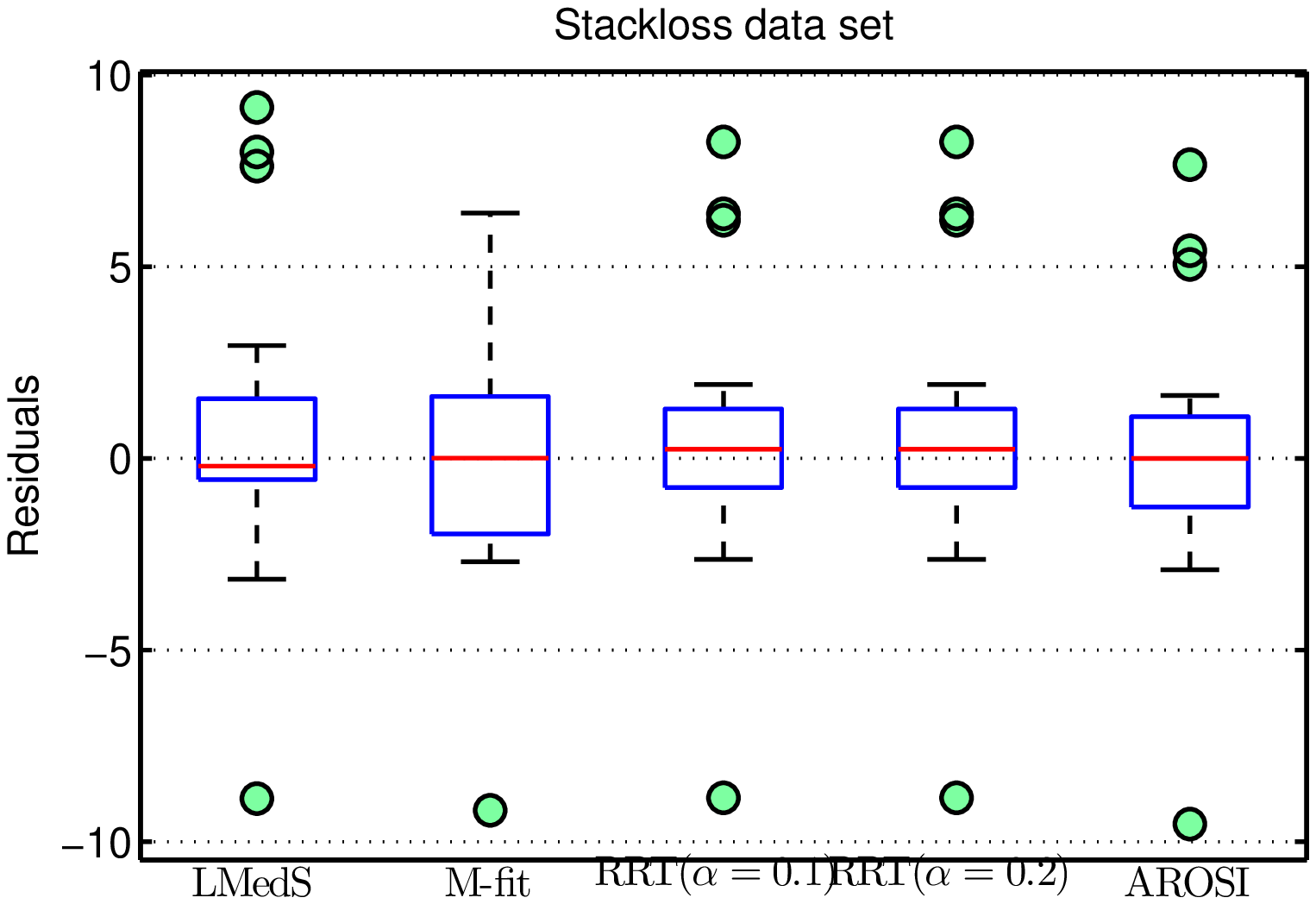}\par 
    \caption*{b). Star.}
    \includegraphics[width=\linewidth]{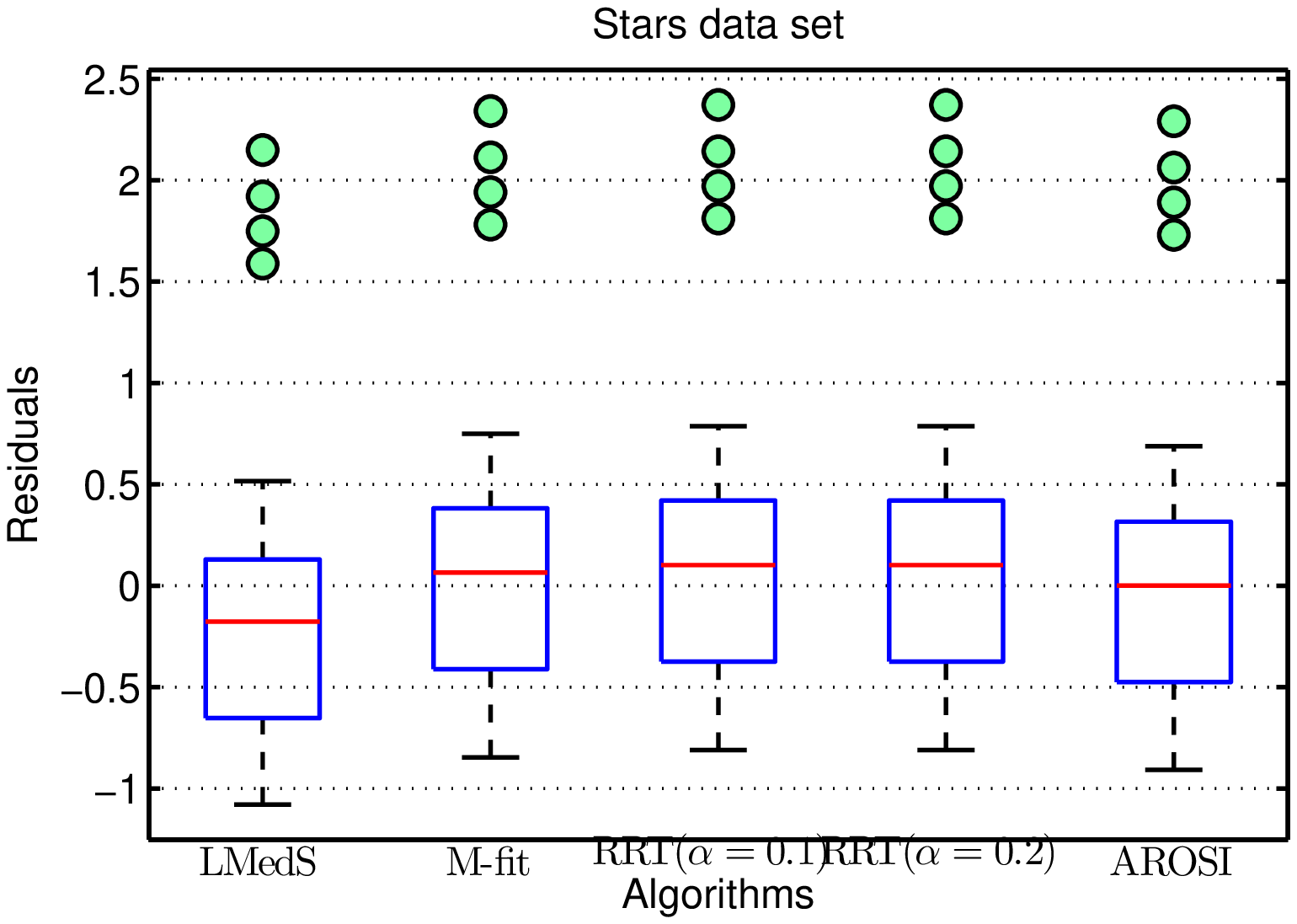}\par 
    \end{multicols}
\begin{multicols}{2}
\caption*{c). Brain-body weight.}
    \includegraphics[width=\linewidth]{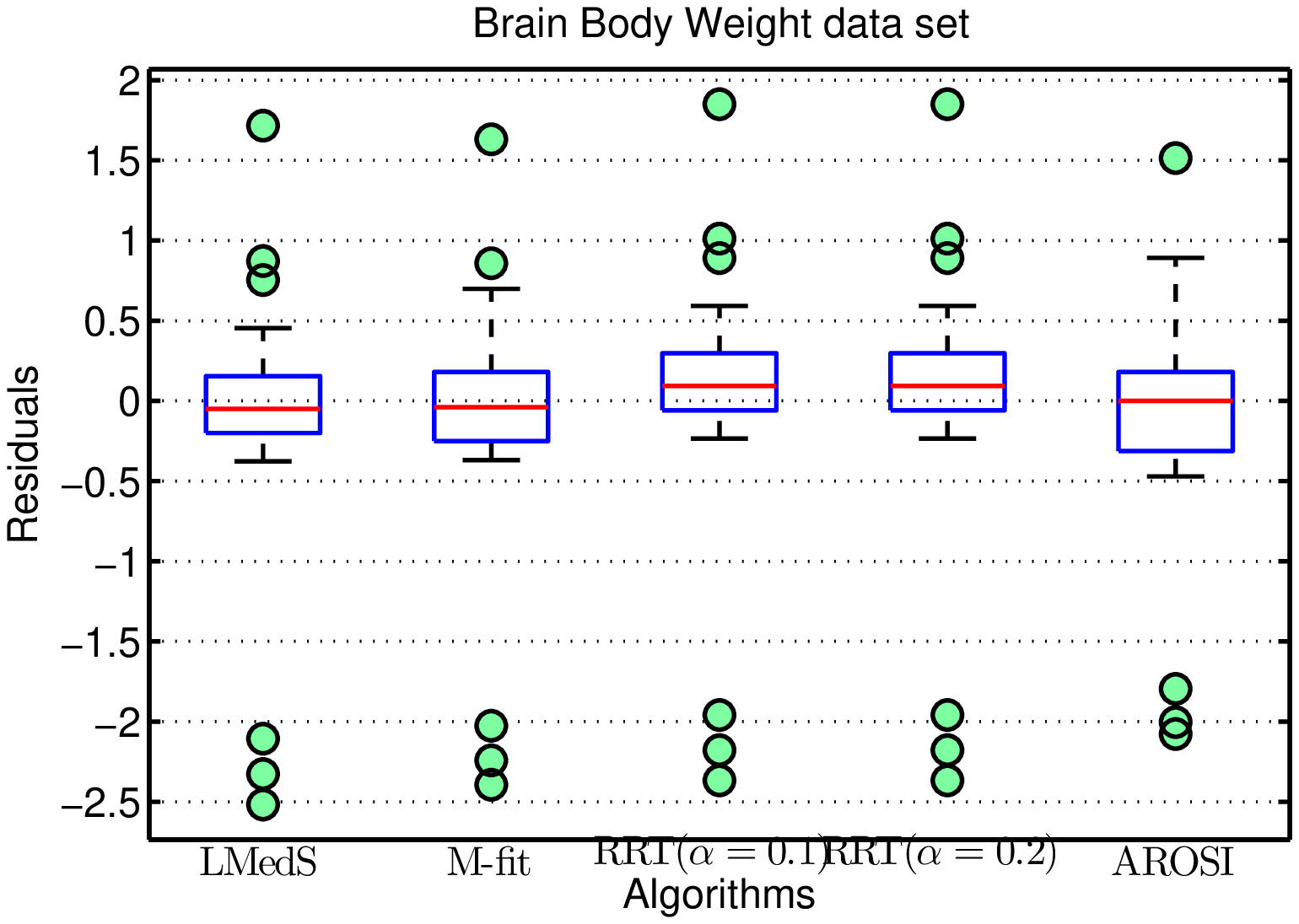}\par
    \caption*{d).  AR2000.}
    \includegraphics[width=\linewidth]{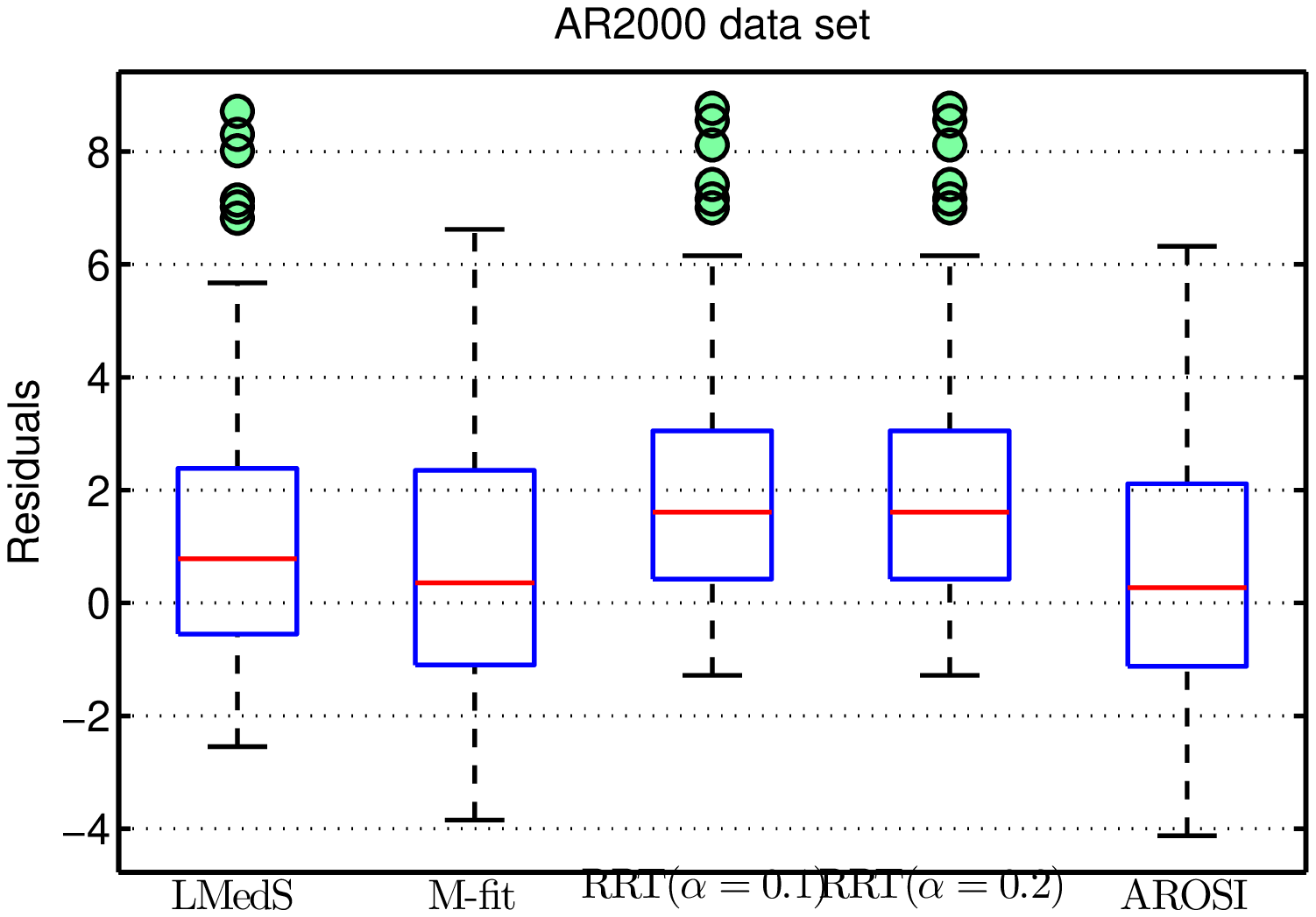}\par
\end{multicols}
\caption{Outlier detection in real data sets using Box plots.}
\label{fig:real_data_sets}
\end{figure}
In this section, we evaluate the performance of the RRT-GARD for outlier detection in four widely studied real life data sets, \textit{viz}., Brownlee's Stack loss  data set, Star data set, Brain and body weight data set (all three discussed in \cite{rousseeuw2005robust}) and the AR2000 dataset studied in \cite{atkinson2012robust}. Algorithms like RRT-GARD, AROSI, M-est etc. are not designed directly to perform outlier detection, rather they are designed to produce good estimates of $\boldsymbol{\beta}$. Hence,  we accomplish outlier detection using RRT-GARD, M-est, AROSI etc.   by analysing the corresponding residual ${\bf r}={\bf y}-{\bf X}\hat{\boldsymbol{\beta}}$  using the popular  Tukeys' box plot\cite{martinez2010exploratory}. Since, there is no ground truth in real data sets, we compare RRT-GARD with the computationally complex LMedS algorithm and  the existing studies on these data sets.  The $\sigma^2$ used in  AROSI is estimated using scheme 1. 

Stack loss data set contains $n=21$ observations and three predictors plus an intercept term. This data set deals with the operation of a plant that convert ammonia to nitric acid. Extensive previous studies\cite{rousseeuw2005robust,bdrao_robust} reported that observations $\{1,3,4,21\}$ are potential outliers.  Box plot in Fig. \ref{fig:real_data_sets} on the residuals computed by RRT-GARD, AROSI and LMedS also agree with the existing results. However, box plot of M-est   can identify only one outlier.    Star data set explore the relationship between the intensity of a star (response) and its surface temperature (predictor) for 47 stars in the star cluster CYG OB1 after taking a log-log transformation\cite{rousseeuw2005robust}. It is well known that 43   of these 47 stars belong to one group, whereas, four stars \textit{viz.} 11, 20, 30 and 34 belong to another group. This can be easily seen from scatter plot\cite{martinez2010exploratory} itself. Box plots for all algorithms identify these four stars as outliers. 

Brain body weight data set explores the interesting hypothesis that body weight (predictor)   is  positively correlated with brain weight (response) using the data available for 27 land animals\cite{rousseeuw2005robust}. Scatter plot after log-log transformation itself reveals three extreme outliers, \textit{viz.} observations 6, 16 and 25 corresponding to three Dinosaurs (big body and small brains). Box plot using LMedS and RRT-GARD residuals identify 1 (Mountain Beaver), 14 (Human) and 17 (Rhesus monkey) also as outliers. These animals have smaller body sizes and disproportionately large brains. However, Box plot using residuals computed by M-est   shows 17 as an inlier, whereas, AROSI shows 14 and 17 as inliers.    AR2000 is an artificial  data set discussed in TABLE A.2 of \cite{atkinson2012robust}. It has $n=60$ observations and $p=3$ predictors. Using extensive graphical analysis, it was shown in \cite{atkinson2012robust} that observations $\{9, 21, 30, 31, 38,47\}$ are outliers. Box plot with LMedS and RRT-GARD  also identify these as outliers, whereas, M-est and AROSI does not identify any outliers at all. To summarize, RRT-GARD matches LMedS and existing results in literature on  all the four datasets considered. This points to the superior performance and practical utility of RRT-GARD over M-est, AROSI etc.  Also please note that RRT-GARD with both $\alpha=0.1$ and $\alpha=0.2$ delivered exactly similar results in real data sets also.     

\section{Conclusions and future directions}
{This article developed a novel  noise statistics oblivious robust regression technique   and derived finite sample and asymptotic guarantees for the same. Numerical simulations indicate that RRT-GARD can deliver a very high quality performance compared to many state of the art algorithms. Note that GARD($\sigma^2$) itself is inferior in performance to BPRR, RMAP, AROSI etc. when $\sigma^2$ is known \textit{a priori} and RRT-GARD is designed to perform similar to GARD($\sigma^2$). Hence, developing similar inlier statistics oblivious frameworks with finite sample guarantees for  BPRR, RMAP, AROSI etc. may produce robust regression algorithms with much better performances than RRT-GARD itself.  This would be a topic of future research.  Another interesting topic of future research is to charecterize the optimum regularization and reprojection parameters for algorithms like AROSI, RMAP etc. when estimated noise statistics are used. }
\section*{Appendix A: Proof of Theorem \ref{lemma_gard}. }
Define ${\bf y}^*={\bf X}\boldsymbol{\beta}+{\bf g}_{out}$, i.e., ${\bf y}^*$ is  ${\bf y}$ without inlier  noise ${\bf w}$. Since, $\mathcal{S}_{g}=supp({\bf g}_{out})$, ${\bf y}^*=[{\bf X} \ {\bf I}_{\mathcal{S}_{g}}^n] [\boldsymbol{\beta}^T,  {\bf g}_{out}({\mathcal{S}_{ g}})^T]^T$. In other words, ${\bf y}^*\in span({\bf A}_g)$, where ${\bf A}_g=[{\bf X} \ {\bf I}_{\mathcal{S}_{g}}^n]$.    Lemma \ref{Projection} follows directly from this observation and the properties of projection matrices. 
\begin{lemma}\label{Projection} ${\bf y}^*\in span({\bf A}_g)$ implies that  $({\bf I}^n-{\bf P}_{{\bf A}^{k}}){\bf y}^*\neq {\bf 0}_n$ if $\mathcal{S}_{g} \not \subseteq {\mathcal{S}}^k_{GARD}  $ and $({\bf I}^n-{\bf P}_{{\bf A}^{k}}){\bf y}^*= {\bf 0}_n$ if $\mathcal{S}_{g}  \subseteq {\mathcal{S}}^k_{GARD}$. Likewise,  ${\bf X}\boldsymbol{\beta}\in span({\bf X})\subseteq span({\bf A}^k)$ implies that $({\bf I}^n-{\bf P}_{{\bf A}^{k}}){\bf X}\boldsymbol{\beta}={\bf 0}_n$, $\forall k\geq 0$.  
\end{lemma} 
 The definition of $k_{min}$ along with the  monotonicity of support ${\mathcal{S}}^k_{GARD}$ in Lemma \ref{lemma:gard} implies that ${\bf y}^*\notin span({\bf A}^{k})$ for $k<k_{min}$ and ${\bf y}^*\in span({\bf A}^{k})$ for $k\geq k_{min}$.  It then follows from Lemma \ref{Projection} that  ${\bf r}^{k}_{GARD}=({\bf I}^n-{\bf P}_{{\bf A}^{k}}){\bf y}=({\bf I}^n-{\bf P}_{{\bf A}^{k}}){\bf g}_{out}+({\bf I}^n-{\bf P}_{{\bf A}^{k}}){\bf w}$ for $k<k_{min}$, whereas, ${\bf r}^{k}_{GARD}=({\bf I}^n-{\bf P}_{{\bf A}^{k}}){\bf w}$ for $k\geq k_{min}$. Also by Lemma \ref{lemma_gard}, we know that $\|{\bf w}\|_2\leq \epsilon_{GARD}$ implies that $k_{min}=k_g$ and ${\bf A}^{k_{min}}={\bf A}_g$.  Then following the previous analysis, $RR(k_{min})=\dfrac{\|({\bf I}^n-{\bf P}_{{\bf A}_g}){\bf w} \|_2}{\|({\bf I}^n-{\bf P}_{{\bf A}^{k_{g}-1}})({\bf g}_{out}+{\bf w})\|_2} $ \text{for} $\|{\bf w}\|_2 \leq \epsilon_{GARD}$. From the proof of Theorem 4 in \cite{gard}, we have $\|({\bf I}^n-{\bf P}_{{\bf A}^{k_{g}-1}})({\bf g}_{out}+{\bf w})\|_2 \geq {\bf g}_{min}-\delta_{k_g}^2\|{\bf g}_{out}\|_2-(\sqrt{\dfrac{{3}}{2}}+1)\|{\bf w}\|_2$  once $\|{\bf w}\|_2\leq \epsilon_{GARD}$. When $\|{\bf w}\|_2\geq \epsilon_{GARD}$, $k_{min}$ may not be equal to $k_g$. However, it will satisfy  $RR(k_{min})\leq 1$.   Hence,
 \begin{equation}
 \begin{array}{ll}
RR(k_{min})\leq \left(\dfrac{\|({\bf I}^n-{\bf P}_{{\bf A}_g}){\bf w} \|_2}{{\bf g}_{min}-\delta_{k_g}^2\|{\bf g}_{out}\|_2-(\sqrt{\dfrac{{3}}{2}}+1)\|{\bf w}\|_2}\right)\\
 \times \mathcal{I}_{\{\|{\bf w}\|_2 \leq \epsilon_{GARD}\}}
 +\mathcal{I}_{\{\|{\bf w}\|_2 > \epsilon_{GARD}\}},
\end{array}
\end{equation}
 where $\mathcal{I}_{\{x\}}$ is the indicator function satisfying $\mathcal{I}_{\{x\}}=1$ for $x>0$ and $\mathcal{I}_{\{x\}}=0$ for $x\leq 0$. Note that   $\|{\bf w}\|_2\overset{P}{\rightarrow } 0$ as $\sigma^2\rightarrow 0$ implies that  $\dfrac{\|({\bf I}^n-{\bf P}_{{\bf A}_g}){\bf w} \|_2}{{\bf g}_{min}-\delta_{k_g}^2\|{\bf g}_{out}\|_2-(\sqrt{\dfrac{{3}}{2}}+1)\|{\bf w}\|_2}\overset{P}{\rightarrow} 0$, $\mathcal{I}_{\{\|{\bf w}\|_2 > \epsilon_{GARD}\}}\overset{P}{\rightarrow} 0$ and $\mathcal{I}_{\{\|{\bf w}\|_2 \leq \epsilon_{GARD}\}}\overset{P}{\rightarrow} 1$. This together with $RR(k)\geq 0$ for all $k$ implies that $RR(k_{min})\overset{P}{\rightarrow } 0$ as $\sigma^2\rightarrow 0$. Similarly, $\|{\bf w}\|_2\overset{P}{\rightarrow } 0$ as $\sigma^2\rightarrow 0$ also implies that  $\underset{\sigma^2\rightarrow 0}{\lim}\mathbb{P}(k_{min}=k_g)\geq \underset{\sigma^2\rightarrow 0}{\lim}\mathbb{P}(\|{\bf w}\|_2\leq \epsilon_{GARD})= 1$. 

\section*{Appendix B: Proof of Theorem \ref{thm:Beta}.}
The proof of Theorem \ref{thm:Beta} is based on the distributions associated with projection matrices.  We first discuss some preliminary distributional results and the proof of Theorem \ref{thm:Beta} is given in the next subsection.  
\subsection{Projection matrices and distributions.} 
Assume temporarily that the support of ${\bf g}_{out}$ is given by $\mathcal{S}_g^{temp}=\{1,2,\dotsc,k_g\}$. Further, consider an algorithm $Alg$ that produces support estimates $\mathcal{S}^k_{Alg}=\{1,2\dotsc,k\}$, i.e., the support estimate sequence is deterministic. For this support sequence, $k_{min}=k_g$ deterministically.  Define ${\bf A}_{Alg}^k=[{\bf X}, \ {\bf I}^n_{\mathcal{S}_{Alg}^k}]$. Then using Lemma \ref{Projection} , ${\bf r}^k_{Alg}=({\bf I}^n-{\bf P}_{{\bf A}_{Alg}^k}){\bf y}=({\bf I}^n-{\bf P}_{{\bf A}_{Alg}^k}){\bf g}_{out}+({\bf I}^n-{\bf P}_{{\bf A}_{Alg}^k}){\bf w}$ for $k<k_g$ and ${\bf r}^k_{Alg}=({\bf I}^n-{\bf P}_{{\bf A}_{Alg}^k}){\bf w}$ for $k\geq k_g$. Using  standard distributional results discussed in\cite{tsp} for deterministic projection matrices give the following for $k>k_g$ and $\sigma^2>0$. 
\begin{equation}\label{beta3}
RR(k)^2=\dfrac{\|{\bf r}^k_{Alg}\|_2^2}{\|{\bf r}^{k-1}_{Alg}\|_2^2}=\frac{\|({\bf I}^n-{\bf P}_{{\bf A}^{k}_{Alg}}){\bf w}\|_2^2}{\|({\bf I}^n-{\bf P}_{{\bf A}^{k-1}_{Alg}}){\bf w}\|_2^2} \sim \mathbb{B}(\frac{n-p-k}{2},\frac{1}{2}).
\end{equation}
Define $\Gamma_{Alg}^{\alpha}(k)=\sqrt{F^{-1}_{\frac{n-p-k}{2},\frac{1}{2}}\left(\dfrac{\alpha}{k_{max}}\right)}$. Then it follows from the union bound and the definition of $\Gamma_{Alg}^{\alpha}(k)$ that
\begin{equation}\label{naive_bound}
\begin{array}{ll}
\mathbb{P}(RR(k)>\Gamma_{Alg}^{\alpha}(k),\forall k\geq k_{min}=k_g)\\=1-\mathbb{P}\left(\exists k\geq k_{g}, RR(k)^2<\left(\Gamma_{Alg}^{\alpha}(k)\right)^2\right)\\
\geq 1-\sum\limits_{k>k_g}F_{\frac{n-p-k}{2},\frac{1}{2}}\left(F^{-1}_{\frac{n-p-k}{2},\frac{1}{2}}\left(\frac{\alpha}{k_{max}}\right)\right)\geq 1-\alpha,
\end{array}
\end{equation}
$\forall\ \sigma^2>0$. The support sequence produced by GARD is different from the hypothetical  algorithm Alg in at least two ways. a) The support sequence $\mathcal{S}_{GARD}^k$ and projection matrix sequence ${\bf P}_{{\bf A}^k}$ in GARD are not deterministic and is data dependent. b) $k_{min}$ is not a deterministic quantity, but a R.V taking value in $\{k_g,\dotsc,k_{max},\infty\}$. a) and b) imply that the distributional results  (\ref{beta3}) and  (\ref{naive_bound}) derived for deterministic support and projection matrix  sequences  are not applicable to GARD support sequence estimate $\{\mathcal{S}_{GARD}^k\}_{k=1}^{k_{max}}$.   
\subsection{Analysis of GARD residual ratios}
 The proof of Theorem \ref{thm:Beta} proceeds by conditioning on the R.V $k_{min}$ and  by lower bounding  $RR(k)$ for $k>k_{min}$  using  R.Vs with known distribution. 

{\bf Case 1:-}  {\bf Conditioning on $k_g\leq k_{min}=j<k_{max}$}.
 Since $\mathcal{S}_g\subseteq \mathcal{S}_{GARD}^k$ for $k\geq k_{min}$,  it follows from  the proof of Theorem \ref{lemma_gard} and Lemma \ref{Projection}  that ${\bf r}^k_{GARD}=({\bf I}^n-{\bf P}_{{\bf A}^k}){\bf w}$ for $k\geq k_{min}=j$ which in turn implies that
\begin{equation}
RR(k)=\dfrac{\|({\bf I}^n-{\bf P}_{{\bf A}^k}){\bf w}\|_2}{\|({\bf I}^n-{\bf P}_{{\bf A}^{k-1}}){\bf w}\|_2}
\end{equation}
for $k>k_{min}=j$.  
 Consider the step $k-1$ of the GARD where $k>j$. Current support estimate ${\mathcal{S}}^{k-1}_{GARD}$ is itself a R.V.   Let $\mathcal{L}_{k-1}\subseteq \{[n]/{\mathcal{S}}^{k-1}_{GARD}\}$ represents the set of  all possible indices $l$ at stage $k-1$ such that ${\bf A}^{k-1,l}=[{\bf X}\ {\bf I}_{{\mathcal{S}}^{k-1}_{GARD}\cup l}^n]=[{\bf A}^{k-1}\ {\bf I}^n_l]$ is full rank. Clearly,  $card(\mathcal{L}_{k-1})\leq n-card(\mathcal{S}^{k-1}_{GARD})=n-k+1$. Likewise, let $\mathcal{K}^{k-1}$ represents the set of all possibilities for the set $\mathcal{S}^{k-1}_{GARD}$ that would also satisfy the constraint $k> k_{min}=j$, i.e., $\mathcal{K}^{k-1}$ is the set of all ordered sets of size $k-1$ such that the $j^{th}$ entry should belongs to $\mathcal{S}_g$ and the  $k_g-1$ entries out of the first $j-1$   entries  should belong to $\mathcal{S}_g$. 
 
 Conditional on both the R.Vs $k_{min}=j$ and $\mathcal{S}^{k-1}_{GARD}=s^{k-1}_{gard}\in \mathcal{K}^{k-1}$, the projection matrix ${\bf P}_{{\bf A}^{k-1}}$ is a deterministic matrix  and so are ${\bf P}_{{\bf A}^{k-1,l}}$  for each $l\in \mathcal{L}_{k-1}$. Consequently, conditional on  $k_{min}=j$ and $\mathcal{S}^{k-1}_{GARD}=s^{k-1}_{gard}$, it follow from the discussions in Part A of Appendix B for deterministic projection matrices  that  the conditional R.V
\begin{align*}
Z_k^{l}|\{{\mathcal{S}}^{k-1}_{GARD}=s^{k-1}_{gard},k_{min}=j\}=\frac{\|({\bf I}^n-{\bf P}_{{\bf A}^{k-1,l}}){\bf w}\|_2^2}{\|({\bf I}^n-{\bf P}_{{\bf A}^{k-1}}){\bf w}\|_2^2} 
\end{align*} 
$ \text{for} \ l \ \in \mathcal{L}_{k-1}$ has distribution 
\begin{align*}
Z_k^{l}|\{{\mathcal{S}}^{k-1}_{GARD}=s^{k-1}_{gard},k_{min}=j\} \sim \mathbb{B}\left(\frac{n-p-k}{2},\frac{1}{2}\right), 
\end{align*}
$\forall l \in \mathcal{L}_{k-1}$. Since the index selected in the $k-1^{th}$ iteration belongs to  $\mathcal{L}_{k-1}$, it follows that conditioned on $\{{\mathcal{S}}^{k-1}_{GARD}=s^{k-1}_{gard},k_{min}=j\}$,
\begin{equation}\label{cond_Beta}
\underset{l\in \mathcal{L}_{k-1}}{\min}\sqrt{Z_k^l|\{{\mathcal{S}}^{k-1}_{GARD}=s^{k-1}_{gard},k_{min}=j\} }\leq RR(k). 
\end{equation}
By  the distributional result (\ref{cond_Beta}), $\Gamma_{RRT}^{\alpha}(k)=\sqrt{F_{\frac{n-p-k}{2},0.5}^{-1}\left(\frac{\alpha}{k_{max}(n-k+1)}\right)}$ satisfies \squeezeup
\begin{equation}\label{abbbbcbc}
\begin{array}{ll}
\mathbb{P}({Z_k^l}<\left(\Gamma_{RRT}^{\alpha}(k)\right)^2|\{{\mathcal{S}}^{k-1}_{GARD}=s^{k-1}_{gard},k_{min}=j\})\\=F_{\frac{n-p-k}{2},0.5}\left(F_{\frac{n-p-k}{2},0.5}^{-1}\left(\frac{\alpha}{k_{max}(n-k+1)}\right)\right)=\dfrac{\alpha}{k_{max}(n-k+1)}
\end{array}
\end{equation}
Using union bound and $card(\mathcal{L}_{k-1})\leq n-k+1$ in (\ref{abbbbcbc}) gives
\begin{equation}\label{firstbound}
\begin{array}{ll}
\mathbb{P}(RR(k)<\Gamma_{RRT}^{\alpha}(k)|\{{\mathcal{S}}^{k-1}_{GARD}=s_{gard}^{k-1},k_{min}=j\})\\
\leq \mathbb{P}(\underset{l\in \mathcal{L}_{k-1}}{\min}\sqrt{Z_k^l|}<\Gamma_{RRT}^{\alpha}(k)|\{{\mathcal{S}}^{k-1}_{GARD}=s_{gard}^{k-1},k_{min}=j\}) \\
 {\leq} \sum\limits_{l \in \mathcal{L}_{k-1}}\mathbb{P}({Z_k^l}<\Gamma_{RRT}^{\alpha}(k)^2|\{{\mathcal{S}}^{k-1}_{GARD}=s_{gard}^{k-1},k_{min}=j\})\\
{\leq} \dfrac{\alpha}{k_{max}}.
\end{array}
\end{equation}
 Eliminating the random set $\mathcal{S}^{k-1}_{GARD}=s_{gard}^{k-1}$ from (\ref{firstbound}) using the law of total probability gives the following $\forall k>k_{min}=j$
\begin{equation}\label{secondbound}
\begin{array}{ll}
\mathbb{P}(RR(k)<\Gamma_{RRT}^{\alpha}(k)|k_{min}=j)\\
=\sum\limits_{s_{gard}^{k-1} \in \mathcal{K}^{k-1}} \mathbb{P}(RR(k)<\Gamma_{RRT}^{\alpha}(k)|\{\mathcal{S}^{k-1}_{GARD}=s^{k-1}_{gard},k_{min}=j\}) \\
\ \ \ \ \ \ \ \ \  \times \mathbb{P}(\mathcal{S}^{k-1}_{GARD}=s^{k-1}_{gard}|k_{min}=j) \\
\leq \sum\limits_{s_{gard}^{k-1} \in \mathcal{K}^{k-1}}\dfrac{\alpha}{k_{max}} \mathbb{P}(\mathcal{S}^{k-1}_{Alg}=s^{k-1}_{gard}|k_{min}=j)
=\dfrac{\alpha}{k_{max}}.
\end{array}
\end{equation}
Now applying union bound  and (\ref{secondbound}) gives
\begin{equation}\label{thirdbound}
\begin{array}{ll}
\mathbb{P}(RR(k)>\Gamma_{RRT}^{\alpha}(k),\forall k>k_{min}|k_{min}=j)\\
\geq 1-\sum\limits_{k=j+1}^{k_{max}}\mathbb{P}(RR(k)<\Gamma_{RRT}^{\alpha}(k)|k_{min}=j)\\
\geq 1-\alpha \dfrac{k_{max}-j}{k_{max}} \geq 1-\alpha.
\end{array}
\end{equation}
{\bf Case 2:-}  {\bf Conditioning on $ k_{min}=\infty$ and $k_{min}=k_{max}$}. In both these cases, the set $\{k_g< k\leq k_{max}:k>k_{min}\}$ is empty. Applying the usual convention of assigning the minimum value of empty sets to $\infty$, one has for $j \in \{k_{max},\infty\}$
\begin{equation}\label{fourthbound}
\begin{array}{ll}
\mathbb{P}(RR(k)>\Gamma_{RRT}^{\alpha}(k),\forall k>k_{min}|k_{min}=j)\\
\geq \mathbb{P}(\underset{k>j}{\min}RR(k)>\underset{k>j}{\max}\Gamma_{RRT}^{\alpha}(k)|k_{min}=j)\\
=1 \geq 1-\alpha.
\end{array}
\end{equation}
Again applying law of total probability to remove the conditioning on $k_{min}$ along with bounds (\ref{thirdbound}) and (\ref{fourthbound}) gives
\begin{equation}\label{finalbound}
\begin{array}{ll}
\mathbb{P}(RR(k)>\Gamma_{RRT}^{\alpha}(k),\forall k>k_{min})\\=\sum\limits_{j }\mathbb{P}(RR(k)>\Gamma_{RRT}^{\alpha}(k),\forall k>k_{min}|k_{min}=j)\mathbb{P}(k_{min}=j)\\
\geq \sum\limits_{j }(1-\alpha)\mathbb{P}(k_{min}=j)=1-\alpha,\ \forall \sigma^2>0.
\end{array}
\end{equation}
This proves the statement in Theorem \ref{thm:Beta}.
\section*{Appendix C: Proof of Theorem \ref{thm:rrt-gard}}
 RRT-GARD support estimate $\mathcal{S}_{RRT}={\mathcal{S}}^{{k}_{RRT}}_{GARD}$, where ${k}_{RRT}=\max\{k:RR(k)<\Gamma^{\alpha}_{RRT}(k)\}$ equals outlier support $\mathcal{S}_{ g}$ iff the following three events  occurs simultaneously. \\
$\mathcal{A}_1:$ First $k_g$ iterations in GARD are correct, i.e., $k_{min}=k_g$.\\
$\mathcal{A}_2:$ $RR(k)>\Gamma^{\alpha}_{RRT}(k)$ for all $k>k_{min}$. \\ $\mathcal{A}_3:$ $RR(k_g)<\Gamma^{\alpha}_{RRT}(k_g)$\\
 Hence, the probability  of correct outlier support recovery, i.e., $\mathbb{P}({\mathcal{S}}^{{k}_{RRT}}_{GARD}=\mathcal{S}_g)=\mathbb{P}(\mathcal{A}_1\cap \mathcal{A}_2\cap \mathcal{A}_3)$. 
 
By Lemma \ref{lemma_gard}, event $\mathcal{A}_1$ is true once $\|{\bf w}\|_2\leq \epsilon_{GARD}$. By Theorem \ref{thm:Beta}, $\mathcal{A}_2$ is true with  probability $\mathbb{P}(\mathcal{A}_2)\geq 1-\alpha,\forall \sigma^2>0$. Next, consider the event $\mathcal{A}_3$ assuming that $\mathcal{A}_1$ is true, i.e., $\|{\bf w}\|_2\leq \epsilon_{GARD}$.  From the proof of Theorem 4 in \cite{gard}, ${\bf r}^{k}_{GARD}$  for $k<k_g$ and $\|{\bf w}\|_2\leq \epsilon_{GARD}$ satisfies
 \begin{equation}\label{lb}
 \|{\bf r}^{k}_{GARD}\|_2\geq {\bf g}_{min}-\delta_{k_g}^2\|{\bf g}_{out}\|_2-(\sqrt{\dfrac{{3}}{2}}+1)\|{\bf w}\|_2.
 \end{equation}
   By Lemma \ref{lemma_gard}, $\mathcal{S}_{GARD}^{k_g}=\mathcal{S}_g$  if $\|{\bf w}\|_2<\epsilon_{GARD}$. This  implies that $\|{\bf r}^{k_g}_{GARD}\|_2=\|({\bf I}^n-{\bf P}_{{\bf A}^{k_g}}){\bf y}\|_2=\|({\bf I}^n-{\bf P}_{{\bf A}^{k_g}}){\bf w}\|_2\leq  \|{\bf w}\|_2$. Hence, if $\|{\bf w}\|_2\leq \epsilon_{GARD}$, then  $RR(k_g)$ satisfies
  \begin{equation}\label{ub}
  RR(k_g)\leq \dfrac{\|{\bf w}\|_2}{{\bf g}_{min}-\delta_{k_g}^2\|{\bf g}_{out}\|_2-(\sqrt{\dfrac{{3}}{2}}+1)\|{\bf w}\|_2}.
  \end{equation}
$\mathcal{A}_3$ is true once the upper bound on $RR(k_g)$ in (\ref{ub}) is lower than $\Gamma^{\alpha}_{RRT}(k_g)$ which in turn is true whenever $\|{\bf w}\|_2< \min(\epsilon_{GARD},\epsilon_{RRT})$. Hence, $\epsilon^{\sigma}\leq \min(\epsilon_{GARD},\epsilon_{RRT})$ implies that $\mathbb{P}(\mathcal{A}_1\cap \mathcal{A}_3)\geq 1-1/n$. This along with $\mathbb{P}(\mathcal{A}_2)\geq 1-\alpha,\forall \sigma^2>0$ implies that $\mathbb{P}(\mathcal{A}_1\cap \mathcal{A}_2\cap \mathcal{A}_3)\geq 1-1/n-\alpha$,  whenever $\epsilon^{\sigma}<\min(\epsilon_{GARD},\epsilon_{RRT})$. Hence proved.

 \section*{Appendix D: Proof of Theorem \ref{thm:asymptotic}}
 Recall that $\Gamma_{RRT}^{\alpha}(k_g)=\sqrt{\Delta_n}$, where $\Delta_n=F_{\frac{n-p-k_g}{2},0.5}^{-1}\left(x_n\right)$ and $x_n=\dfrac{\alpha}{(n-p-1)(n-k_g+1)}$. Irrespective of whether $\alpha$ is a constant or $\alpha\rightarrow 0$ with increasing $n$, the condition $\underset{n \rightarrow \infty}{\lim}\dfrac{p+k_g}{n}<1$ implies that  $\underset{n \rightarrow \infty}{\lim}x_n=0$.  Expanding  $F^{-1}_{a,b}(z)$  at $z=0$ gives \cite{kallummil18a}
\begin{equation}\label{beta_exp}
\begin{array}{ll}
F^{-1}_{a,b}(z)=\rho(n,1)+\dfrac{b-1}{a+1}\rho(n,2) \\
+\dfrac{(b-1)(a^2+3ab-a+5b-4)}{2(a+1)^2(a+2)}\rho(n,3)
+O(z^{(4/a)})
\end{array}
\end{equation}
for all $a>0$.   We associate $a=\frac{n-p-k_g}{2}$, $b=1/2$ , $z=x_n$ and $\rho(n,l)=(az\mathbb{B}(a,b))^{(l/a)}=\left(\frac{\left(\frac{n-p-k_g}{2}\right)\alpha\mathbb{B}(\frac{n-p-k_g}{2},0.5)}{{(n-p+1)}(n-k_g+1)}\right)^{\frac{2l}{n-p-k_g}}$ for $l\geq 1$. 
Then $\log(\rho(n,l))$ gives
\begin{equation}\label{log_rho}
\begin{array}{ll}
\log(\rho(n,l))=\frac{2l}{n-p-k_g}\log\left(\frac{(\frac{n-p-k_g}{2})}{{n-p+1}}\right)-\frac{2l}{n-p-k_g}\log(n-k_g+1) \\
+\frac{2l}{n-p-k_g}\log\left(\mathbb{B}(\frac{n-p-k_g}{2},0.5)\right)+\frac{2l}{n-p-k_g}\log(\alpha)
\end{array}
\end{equation}
In the limits $n\rightarrow \infty$ and $0\leq \underset{n\rightarrow \infty}{\lim} \dfrac{p+k_g}{n}<1$, the first and second terms in the R.H.S of (\ref{log_rho}) converge to zero. Using the asymptotic expansion\cite{kallummil18a}
 $\mathbb{B}(a,b)=G(b)a^{-b}\left(1-\frac{b(b-1)}{2a}(1+O(\frac{1}{a}))\right)$ as\footnote{$G(b)=\int_{t=0}^{t=\infty}e^{-t}t^{b-1}$ is the Gamma function.} $a \rightarrow \infty$ in the second term of (\ref{log_rho}) gives
\begin{equation}
\underset{n \rightarrow \infty}{\lim}\frac{2l}{n-p-k_g}\log\left(\mathbb{B}(\frac{n-p-k_g}{2},0.5)\right)=0.
\end{equation}
Hence, only the behaviour of $\frac{2l}{n-p-k_g}\log(\alpha)$ need to be considered. Now we consider the three cases depending on the behaviour of $\alpha$.

{\bf Case 1:-} When $\underset{n \rightarrow \infty}{\lim}\log(\alpha)/n=0$ 
one has $\underset{n \rightarrow \infty}{\lim}\log(\rho(n,l))=0$ which in turn implies that $\underset{n \rightarrow \infty}{\lim}\rho(n,l)=1$ for every $l$. 

{\bf Case 2:-} When $-\infty<\alpha_{lim}=\underset{n \rightarrow \infty}{\lim}\log(\alpha)/n<0$ and $\underset{n \rightarrow \infty}{\lim}\dfrac{p+k_g}{n}=d_{lim}<1$,   
one has $-\infty<\underset{n \rightarrow \infty}{\lim}\log(\rho(n,l))=(2l\alpha_{lim})/(1-d_{lim})<0$. This in turn implies that $0<\underset{n \rightarrow \infty}{\lim}\rho(n,l)=e^{\frac{2l\alpha_{lim}}{1-d_{lim}}}<1$ for every $l$. 

{\bf Case 3:-} When $\underset{n \rightarrow \infty}{\lim}\log(\alpha)/n=-\infty$,  
one has $\underset{n \rightarrow \infty}{\lim}\log(\rho(n,l))=-\infty$ which in turn implies that $\underset{n \rightarrow \infty}{\lim}\rho(n,l)=0$ for every $l$.  

Note that the coefficient of $\rho(n,l)$ in (\ref{beta_exp}) for $l>1$ is asymptotically $1/a$. Hence, these coefficients decay to zero in the limits $n\rightarrow \infty$ and $0\leq \underset{n\rightarrow \infty}{\lim} \dfrac{p+k_g}{n}<1$. Consequently, only the $\rho(n,1)$ is non zero as $n \rightarrow \infty$. This implies that $\underset{n \rightarrow \infty}{\lim}\Delta_n=1$ for Case 1,  $0<\underset{n \rightarrow \infty}{\lim}\Delta_n=e^{\frac{2\alpha_{lim}}{1-d_{lim}}}<1$ for Case 2 and $\underset{n \rightarrow \infty}{\lim}\Delta_n=0$ for Case 3. This proves Theorem \ref{thm:asymptotic}.

\section*{Appendix E: Proof of Theorem \ref{thm:high_SNR}}
 { Following the description of RRT in TABLE \ref{tab:rrt-gard}, the missed discovery event
$\mathcal{M}=\{card(\mathcal{S}_g/\mathcal{S}_{RRT})>0\}$ occurs if any of these events occurs.\\
a)$\mathcal{M}_1= \{k_{min}=\infty\}$: then any support in the support sequence produced by GARD suffers from missed discovery. \\
b)$\mathcal{M}_2= \{k_{min}\leq k_{max}$ but $k_{RRT}<k_{min}$\}: then the RRT support estimate misses atleast one entry in $\mathcal{S}_g$. \\
Since these two events are disjoint, it follows that $\P(\mathcal{M})=\P(\mathcal{M}_1)+\P(\mathcal{M}_2)$.  By Lemma \ref{lemma_gard}, it is true that $k_{min}=k_g\leq k_{max}$ whenever $\|{\bf w}\|_2\leq \epsilon_{GARD}$. Note that 
\begin{equation}
\mathbb{P}(\mathcal{M}_1^C)\geq \mathbb{P}(k_{min}=k_g)\geq \mathbb{P}(\|{\bf w}\|_2\leq \epsilon_{GARD}).
\end{equation}
Since ${\bf w}\sim \mathcal{N}({\bf 0}_n,\sigma^2{\bf I}^n)$, we have $\|{\bf w}\|_2\overset{P}{\rightarrow}0$ as $\sigma^2\rightarrow 0$. This implies that $\underset{\sigma^2\rightarrow 0}{\lim}\mathbb{P}(\|{\bf w}\|_2<\epsilon_{GARD})=1$ and $\underset{\sigma^2\rightarrow 0}{\lim}\mathbb{P}(\mathcal{M}_1^C)=1$. This implies  that $\underset{\sigma^2\rightarrow 0}{\lim}\mathbb{P} (\mathcal{M}_1)=0$.

 Next we consider the   event $\mathcal{M}_2$. Using the law of total probability, we have 
 \begin{equation}\label{supp:a1}
 \begin{array}{ll}
 \mathbb{P}(\{k_{min}\leq k_{max} \& k_{RRT}< k_{min}\})=\mathbb{P}(k_{min}\leq k_{max})\\
\ \ \ \ \ \ \ \ \ \  -\mathbb{P}(\{k_{min}\leq k_{max} \& k_{RRT}\geq k_{min}\})
 \end{array}
 \end{equation}
 Following Lemma \ref{lemma:gard},  we have $\mathbb{P}(k_{min}\leq k_{max})\geq \mathbb{P}(k_{min}= k_{g})\geq \mathbb{P}(\|{\bf w}\|_2\leq \epsilon_{GARD})$. This implies that $\underset{\sigma^2\rightarrow 0}{\lim}\mathbb{P}(k_{min}\leq k_{max})=1$. Following the proof of Theorem \ref{thm:rrt-gard}, we know that both $k_{min}=k_g$ and $RR(k_g)<\Gamma_{RRT}^{\alpha}(k_g)$ hold true once $\|{\bf w}\|_2\leq\min(\epsilon_{GARD},\epsilon_{RRT})$.  Hence, 
 \begin{equation}
 \begin{array}{ll}
 \mathbb{P}(\{k_{min}\leq k_{max} \& k_{RRT}\geq k_{min}\})\\
 \ \ \ \ \ \ \ \ \ \ \  \geq \mathbb{P}(\|{\bf w}\|_2\leq\min(\epsilon_{GARD},\epsilon_{RRT})).
 \end{array}
 \end{equation}
 This  in turn implies that $\underset{\sigma^2\rightarrow 0}{\lim}\mathbb{P}(\{k_{min}\leq k_{max} \& k_{RRT}\geq k_{min}\})=1$. Applying these two limits  in (\ref{supp:a1}) give $\underset{\sigma^2\rightarrow 0}{\lim}\mathbb{P}(\mathcal{M}_2)=0$.  Since $\underset{\sigma^2\rightarrow 0}{\lim}\P(\mathcal{M}_1)=0$  and $\underset{\sigma^2\rightarrow 0}{\lim}\P(\mathcal{M}_2)=0$, it follows that $\underset{\sigma^2\rightarrow 0}{\lim}\P(\mathcal{M})=0$.

 Following the proof of Theorem \ref{thm:rrt-gard}, one can see that the event $\mathcal{E}^C=\{{\mathcal{S}_{RRT}}=\mathcal{S}\}$ occurs once three events  $\mathcal{A}_1$,  $\mathcal{A}_2$ and  $\mathcal{A}_3$ occurs simultaneously, i.e.,  $\mathbb{P}(\mathcal{E}^C)= \mathbb{P}(\mathcal{A}_1\cap \mathcal{A}_2\cap \mathcal{A}_3)$.  Of these three events, $\mathcal{A}_1\cap \mathcal{A}_2$ occur once $\|{\bf w}\|_2\leq \min(\epsilon_{GARD},\epsilon_{RRT})$.  This implies that
 \begin{equation}
 \underset{\sigma^2\rightarrow 0}{\lim}\mathbb{P}(\mathcal{A}_1\cap \mathcal{A}_2)\geq \underset{\sigma^2\rightarrow 0}{\lim}\mathbb{P}(\|{\bf w}\|_2\leq \min(\epsilon_{GARD},\epsilon_{RRT}))=1.
 \end{equation}
At the same time, by Theorem \ref{thm:Beta},  $\mathbb{P}(\mathcal{A}_3)\geq 1-\alpha, \forall\sigma^2>0$. Hence, it follows that 
\begin{equation}
\underset{\sigma^2\rightarrow 0}{\lim}\mathbb{P}(\mathcal{E}^C)=\underset{\sigma^2\rightarrow 0}{\lim}\mathbb{P}(\mathcal{A}_1\cap \mathcal{A}_2\cap \mathcal{A}_3)\geq 1-\alpha.
\end{equation}
This in turn implies that $\underset{\sigma^2\rightarrow 0}{\lim}\mathbb{P}(\mathcal{E})\leq \alpha$. Since $\P(\mathcal{E})=\P(\mathcal{M})+\P(\mathcal{F})$ and $\underset{\sigma^2\rightarrow 0}{\lim}\P(\mathcal{M})=0$, it follows that $\underset{\sigma^2\rightarrow 0}{\lim}\P(\mathcal{F})\leq \alpha$. Hence proved.}

\bibliography{compressive.bib}
\bibliographystyle{IEEEtran}

\end{document}